\documentclass{article}
\usepackage{microtype}
\usepackage{graphicx}
\usepackage{subfigure}
\usepackage{booktabs} %
\usepackage{multirow}
\usepackage{hyperref}

\usepackage{url}
\usepackage[section]{placeins}
\usepackage{algorithmic}
\usepackage{amsfonts}       %
\usepackage{nicefrac}       %
\usepackage{xcolor}         %
\usepackage{wrapfig}
\usepackage{listings}
\usepackage{pifont}
\usepackage{dsfont}

\newcommand{\cmark}{\ding{51}}%
\newcommand{\xmark}{\ding{55}}

\usepackage[accepted]{icml2023}

\usepackage{amsmath}
\usepackage{amssymb}
\usepackage{mathtools}
\usepackage{amsthm}

\usepackage[capitalize,noabbrev]{cleveref}

\theoremstyle{plain}
\newtheorem{theorem}{Theorem}[section]

\theoremstyle{definition}

\theoremstyle{remark}

\newenvironment{proof-sketch}{\noindent{\textit{Sketch of proof.}}}

\usepackage[textsize=tiny]{todonotes}

\usepackage{amsmath,amsfonts,bm}

\newcommand{\x}{\vx}

\newcommand{\epsilonv}{\bm{\epsilon}}

\def\eqref#1{Eq.~(\ref{#1})}

\def\1{\bm{1}}
\DeclarePairedDelimiterX{\infdivx}[2]{(}{)}{%
	#1\;\delimsize\|\;#2%
}

\newcommand{\kl}{D_{\mathrm{KL}}\infdivx}

\newcommand{\vect}[1]{\bm{#1}}
\newcommand{\dm}{\mathrm{d}}

\def\va{{\bm{a}}}

\def\vs{{\bm{s}}}

\def\vx{{\bm{x}}}

\def\mI{{\bm{I}}}

\DeclareMathAlphabet{\mathsfit}{\encodingdefault}{\sfdefault}{m}{sl}
\SetMathAlphabet{\mathsfit}{bold}{\encodingdefault}{\sfdefault}{bx}{n}

\def\gE{{\mathcal{E}}}

\def\gN{{\mathcal{N}}}

\def\gU{{\mathcal{U}}}

\def\gX{{\mathcal{X}}}

\newcommand{\Nc}{\mathcal N}

\newcommand{\Xc}{\mathcal X}

\newcommand{\Iv}{\vect I}

\newcommand{\E}{\mathbb{E}}

\newcommand{\R}{\mathbb{R}}

\newcommand{\KL}{D_{\mathrm{KL}}}

\newcommand\Tstrut{\rule{0pt}{2.6ex}}         %
\newcommand\Bstrut{\rule[-0.9ex]{0pt}{0pt}}   %

\icmltitlerunning{Contrastive Energy Prediction for Exact Energy-Guided Diffusion Sampling in Offline Reinforcement Learning}

\begin{document}

\twocolumn[
\icmltitle{Contrastive Energy Prediction for Exact Energy-Guided \\ 
Diffusion Sampling in Offline Reinforcement Learning}

\icmlsetsymbol{equal}{*}

\begin{icmlauthorlist}

\icmlauthor{Cheng Lu}{equal,thu}
\icmlauthor{Huayu Chen}{equal,thu}
\icmlauthor{Jianfei Chen}{thu}
\icmlauthor{Hang Su}{thu}
\icmlauthor{Chongxuan Li}{ruc}
\icmlauthor{Jun Zhu}{thu,pazhou}

\end{icmlauthorlist}

\icmlaffiliation{thu}{Dept. of Comp. Sci. \& Tech., Institute for AI, BNRist Center, Tsinghua-Bosch Joint ML Center, THBI Lab, Tsinghua University}
\icmlaffiliation{ruc}{Gaoling School of AI, Renmin University of China; Beijing Key Lab of Big Data Management and Analysis Methods, Beijing, China}
\icmlaffiliation{pazhou}{Pazhou Lab (Huangpu), Guangzhou, China}

\icmlcorrespondingauthor{Jun Zhu}{dcszj@tsinghua.edu.cn}

\icmlkeywords{Diffusion Models, Guided Sampling, Offline Reinforcement Learning}

\vskip 0.3in
]

\printAffiliationsAndNotice{\icmlEqualContribution} %

\begin{abstract}
Guided sampling is a vital approach for applying diffusion models in real-world tasks that embeds human-defined guidance during the sampling procedure. This paper considers a general setting where the guidance is defined by an (unnormalized) energy function. The main challenge for this setting is that the intermediate guidance during the diffusion sampling procedure, which is jointly defined by the sampling distribution and the energy function, is unknown and is hard to estimate. To address this challenge, we propose an exact formulation of the intermediate guidance as well as a novel training objective named contrastive energy prediction (CEP) to learn the exact guidance. Our method is guaranteed to converge to the exact guidance under unlimited model capacity and data samples, while previous methods can not. We demonstrate the effectiveness of our method by applying it to offline reinforcement learning (RL). Extensive experiments on D4RL benchmarks demonstrate that our method outperforms existing state-of-the-art algorithms. We also provide some examples of applying CEP for image synthesis to demonstrate the scalability of CEP on high-dimensional data. Code is available at \url{https://github.com/thu-ml/CEP-energy-guided-diffusion}.

\end{abstract}

\section{Introduction}

Diffusion models~\citep{sohl2015deep,ho2020denoising,song2020score,karras2022elucidating} have demonstrated incredible success.
A key for applying diffusion models in real-world tasks is to embed human controllability in the sampling procedure. A common paradigm for introducing human preference in diffusion models is \textit{guided sampling}, which includes \textit{classifier guidance}~\citep{diffusion_beat_gan}, \textit{classifier-free guidance}~\citep{ho2021classifier} and other guidance methods~\citep{nichol2021glide,ho2022video,zhao2022egsde}. By leveraging guided sampling, diffusion models can realize amazing text-to-image generation~\citep{saharia2022photorealistic}, video generation~\citep{ho2022video,ho2022imagen,yang2022diffusion,zhou2022magicvideo}, controllable text generation~\citep{li2022diffusion}, inverse molecular design~\citep{bao2022equivariant} and 
reinforcement learning~\citep{diffuser, sfbc, dd}.

Both classifier and classifier-free guidance deal with a conditional sampling problem where there exists paired data with additional conditioning variables during the training procedure~\citep{diffusion_beat_gan,graikos2022diffusion,rombach2022high}.
However, sometimes it is difficult to describe human preference through a conditioning variable and we can only embed our preference through a scalar function.
Examples of such a function include a reward function or pretrained Q-function in reinforcement learning~\citep{diffuser, sfbc}, human feedback in dialogue systems~\citep{ziegler2019fine}, cosine similarity between sample features and designated features in image synthesis~\citep{kwon2022diffusion}, or $L_2$-distance between the sampled frame and the previous frame in video synthesis~\citep{ho2022video}. In such cases, we aim to leverage human preference to manipulate the desired distribution and draw samples by diffusion sampling with additional guidance, while it is hard to directly use classifier or classifier-free guidance since no actual condition is provided.

We consider a general form that subsumes all the above cases. Let $q(\x)$ be an unknown data distribution in $\gX\subseteq \R^d$. We aim to sample from the following distribution:
\begin{equation}
\label{Eq:target_distribution_intro}
    p(\x) \propto q(\x)e^{-\beta\gE(\x)},
\end{equation}
where $\gE(\cdot)$ is an \textit{energy function} from $\Xc\in \R^d$ to $\R$ and we can compute $\gE(\x)$ for each datum. $\beta \geq 0$ is the inverse temperature for controlling the energy strength. 
The high-density region of $p(\x)$ is approximately the intersection of both the high-density regions of $q(\x)$ and $e^{-\beta\gE(\x)}$. As a result, we can insert controllability by embedding the desired properties into the energy function $\gE(\cdot)$. 
The choice of the energy function $\gE(\cdot)$ is highly flexible: we only need to ensure that the integral of $q(\x)e^{-\beta\gE(\x)}$ over $\x\in\gX$ is finite.
We can also introduce additional conditioning variables $c$ by an energy function $\gE(\cdot, c)$.
In particular, let $\beta=1$ and $\gE(\x,c) = -\log q(c|\x)$, the target distribution $p(\x)$ becomes a conditional distribution $q(\x|c)$, which recovers the classic conditional sampling problem as a special case.

Sampling from $p(\x)$ is difficult in general as $p(\x)$ is unnormalized.
Existing attempts~\citep{diffuser, zhao2022egsde,ho2022video,chung2022diffusion,kawar2022denoising} leverage a pretrained diffusion model $q_g(\x)\approx q(\x)$ and apply diffusion sampling with an additional guidance term related to $\gE(\cdot)$ called \textit{energy guidance}.
However, all previously proposed energy guidance is either manually or arbitrarily defined across the diffusion process, and it is unstudied whether the final samples follow the desired distribution $p(\x)$. In fact, we show that previous energy-guided samplers are all inexact in the sense that they cannot guarantee convergence to $p(\x)$.
To the best of our knowledge, how to use the pretrained diffusion model $q_g(\x)$ to draw samples from the exact $p(\x)$ remains largely open.%

In this work, we analyze and derive an exact formulation of the desired guidance for diffusion sampling from $p(\x)$ in \eqref{Eq:target_distribution_intro}. In contrast with previous work, we show that the exact guidance in the diffused data space during sampling is completely determined by the energy function $\gE(\cdot)$ in the original data space. The exact energy guidance is in an intractable form which cannot be computed directly, so we propose a novel training method named \textit{contrastive energy prediction (CEP)} to estimate the guidance using samples from $q(\x)$ as the training data. 
CEP trains the guidance model by comparing the energy $\gE(\cdot)$ within a set of noise-perturbed data samples and using their soft energy labels as supervising signals. 
We theoretically prove that the gradient of the optimal learned model is exactly the desired energy guidance, and thus the final samples are guaranteed to follow $p(\x)$. In a special formulation of $\gE(\cdot)$ which corresponds to the classic conditional sampling case, we additionally show that CEP could be understood as an alternative contrastive approach to the classifier guidance method.

To verify the effectiveness and scalability of CEP, we take two important applications of \eqref{Eq:target_distribution_intro}: offline reinforcement learning (RL) and image synthesis.
For offline RL, we formulate the classic constrained policy optimization problem~\citep{rwr, awr} as Q-guided policy optimization, and evaluate our method in mainstream D4RL \citep{d4rl} benchmarks. Extensive experiments demonstrate that our method outperforms existing state-of-the-art algorithms in most tasks, especially in hard tasks such as AntMaze. For image synthesis, we evaluate conditional sample quality by CEP against classic classifier guidance \citep{diffusion_beat_gan} both quantitatively and qualitatively on ImageNet and find two methods almost equally well-performing. We also provide an example of energy-guided image synthesis to affect the color appearance of sampled images and validate the flexibility of CEP.

\section{Background}
We first present preliminary knowledge of diffusion models as well as offline RL that serves as an important motivation and application of sampling from distribution (\ref{Eq:target_distribution_intro}). 
\subsection{Diffusion (Probabilistic) Models}
Diffusion (probabilistic) models~\citep{sohl2015deep,ho2020denoising,song2020score} are powerful generative models. Given a dataset $\{\x_0^{(i)}\}_{i=1}^N$ with $N$ samples of $D$-dimensional random variable $\x_0$ from an unknown data distribution $q_0(\x_0)$, diffusion models gradually add Gaussian noise from $\x_0$ at time $0$ to $\x_T$ at time $T > 0$. The transition distribution $q_{t0}(\x_t|\x_0)$ satisfies
\begin{equation}
\label{Eq:forward_diffusion}
    q_{t0}(\x_t|\x_0) = \gN(\x_t | \alpha_t\x_0, \sigma_t^2\mI),
\end{equation}
where $\alpha_t,\sigma_t > 0$. Denote $q_t(\x_t)$ as the marginal distribution of $\x_t$ at time $t$. The transition distribution at time $T$ satisfies $q_{T}(\x_T|\x_0)\approx q_T(\x_T) \approx \gN(\x_T | 0, \tilde\sigma^2\mI)$ for some $\tilde\sigma > 0$ and is independent of $\x_0$. Thus, starting from $\x_T \sim \gN(\x_T | 0, \tilde\sigma^2\mI)$, diffusion models aim to recover the original data $\x_0$ by solving a reverse process from $T$ to $0$. The reverse process can alternatively be the \textit{diffusion ODE}~\citep{song2020score}:
\begin{equation}
\label{Eq:diffusion_ode}
    \frac{\dm \x_t}{\dm t}=f(t)\x_t-\frac{1}{2}g^2(t)\nabla_{\x_t}\log q_t(\x_t),
\end{equation}
where $f(t)=\frac{\dm\log \alpha_t}{\dm t},g^2(t)=\frac{\dm \sigma_t^2}{\dm t}-2\frac{\dm\log \alpha_t}{\dm t}\sigma_t^2$~\citep{kingma2021variational} and the only unknown term is the \textit{score function} $\nabla_{\x_t}\log q_t(\cdot)$ of the distribution $q_t$ at each time $t$. Thus, diffusion models train a neural network $\epsilonv_\theta(\x_t,t)$ parameterized by $\theta$ to estimate the scaled score function: $-\sigma_t\nabla_{\x_t}\log q_t(\x_t)$, and the training objective is~\citep{ho2020denoising,song2020score}
\begin{equation}
\begin{aligned}
\label{Eq:diffusion_loss}
    \phantom{{}={}}&\min_\theta \E_{t,\x_t}\left[
        \omega(t)\|\epsilonv_\theta(\x_t,t) + \sigma_t\nabla_{\x_t}\log q_t(\x_t) \|_2^2
    \right]\\
    \Leftrightarrow &\min_\theta \E_{t,\x_0,\epsilonv}\left[
        \omega(t)\|\epsilonv_\theta(\x_t,t) - \epsilonv\|_2^2
    \right]
\end{aligned}
\end{equation}
where $\x_0\sim q_0(\x_0)$, $\epsilonv\sim\gN(\bm{0},\mI)$, $t\sim\gU([0,T])$, $\x_t=\alpha_t \x_0 + \sigma_t \epsilonv$, and $\omega(t)$ is a weighting function and usually set to $\omega(t)\equiv 1$~\citep{ho2020denoising}.
Thus, after training a diffusion model, we usually have $\epsilonv_\theta(\x_t,t)\approx -\sigma_t\nabla_{\x_t}\log q_t(\x_t)$. By replacing the score function with $\epsilonv_\theta$, we can fastly sample $\x_0$ by solving the corresponding diffusion ODEs with some dedicated solvers~\citep{song2020denoising,lu2022dpm}.

\subsection{Constrained Policy Optimization in Offline Reinforcement Learning}

Consider a Markov Decision Process (MDP), described by the tuple $\langle\mathcal{S},\mathcal{A}, P,r,\gamma\rangle$. $\mathcal{S}$ is the state space and $\mathcal{A}$ is the action space. $P(\vs'|\vs,\va)$ and $r(\vs, \va)$ are respectively the transition and reward functions. $\gamma \in (0,1]$ is a constant discount factor. In offline RL, a static dataset $\mathcal{D}^\mu$ containing some interacting history $\{\vs, \va, r, \vs'\}$ between a behavior agent $\mu(\va|\vs)$ and the environment is given. The goal is to maximize the expected accumulated rewards of a model policy $\pi_\theta(\va|\vs)$ in the above MDP by solely utilizing the knowledge learned from the dataset.

Previous work~\citep{rwr, awr} formulate offline RL as constrained policy optimization:
\begin{equation}
\label{Eq:rl_main}
    \max_{\pi} \mathbb{E}_{\vs \sim \mathcal{D}^\mu}\big[ \mathbb{E}_{\va \sim \pi(\cdot|\vs)} Q_\psi(\vs, \va)
     - \frac{1}{\beta} \KL \left(\pi(\cdot |\vs) || \mu(\cdot |\vs) \right)\big],
\end{equation}
where $Q_\psi$ is an action evaluation model which indicates the quality of decision $(\vs, \va)$ 
by estimating the Q-function $Q^\pi(\vs, \va) := \mathbb{E}_{\vs_1=\vs, \va_1=\va; \pi} [\sum_{n=1}^\infty \gamma^n r(\vs_n, \va_n)]$ of the current policy $\pi$.  
$\beta$ is an inverse temperature coefficient. The first term in \eqref{Eq:rl_main} intends to perform policy optimization, while the second term stands for policy constraint. 

It is shown that the optimal policy $\pi^*$ in (\ref{Eq:rl_main}) satisfies:
\begin{equation}
\label{Eq:pi_optimal}
    \pi^*(\va|\vs) \propto \ \mu(\va|\vs) \ e^{\beta Q_\psi(\vs, \va)},
\end{equation}
which falls into the general family of distributions (\ref{Eq:target_distribution_intro}). Therefore, sampling from the optimal policy $\pi^*$ can be implemented by energy-guided sampling with a pretrained diffusion behavior model $\mu_g(\va|\vs)\approx \mu(\va|\vs)$, which motivates us to propose an exact energy-guided sampling method.

\section{Exact Energy-Guided Sampling}
To perform energy-guided sampling for \eqref{Eq:target_distribution_intro}, the guidance during the sampling procedure needs to guarantee that final samples follow the desired distribution $p(\x)$.
In this section, we propose an exact formulation of such energy guidance and propose a novel training objective to estimate the guidance. All the proofs can be found in Appendix~\ref{appendix:proof}.

\subsection{Exact Formulation of Intermediate Energy Guidance}
\label{Sec:problem}

\begin{figure}[t]
    \begin{minipage}{0.14\linewidth}
        \centering 
        \scriptsize{$q_t(\x_t)$}
    \end{minipage}
    \begin{minipage}{0.85\linewidth}
		\centering
		\includegraphics[width=\linewidth]{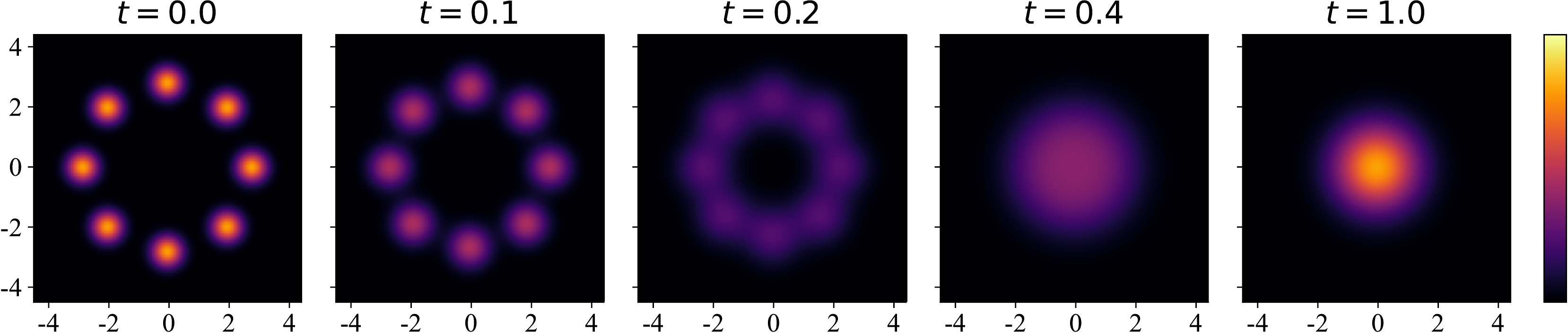}
	\end{minipage} \\
	
	\begin{minipage}{0.14\linewidth}
        \centering 
        \scriptsize{$e^{-\gE_t(\x_t)}$}
    \end{minipage}
    \begin{minipage}{0.85\linewidth}
		\centering
		\includegraphics[width=\linewidth]{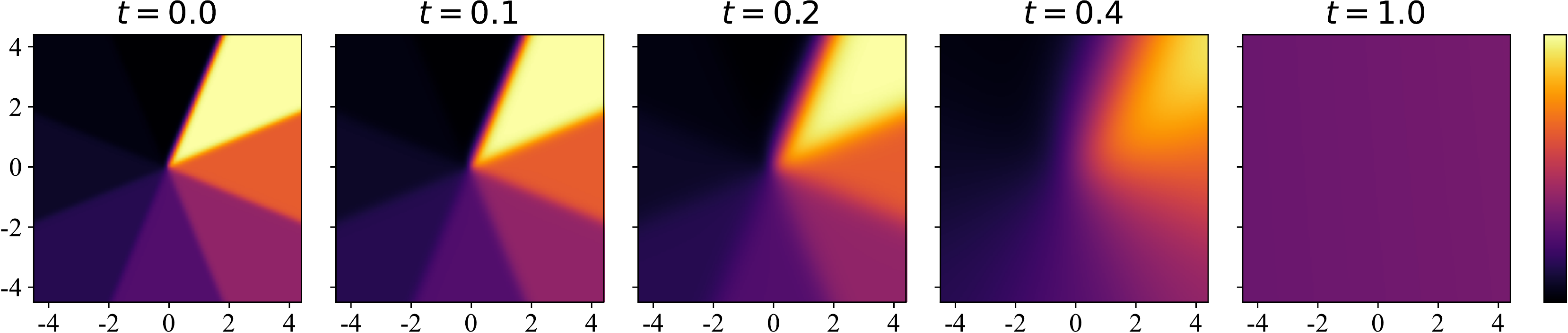}
	\end{minipage} \\
	
	\begin{minipage}{0.14\linewidth}
        \centering 
        \scriptsize{$p_t(\x_t)$}
    \end{minipage}
    \begin{minipage}{0.85\linewidth}
		\centering
		\includegraphics[width=\linewidth]{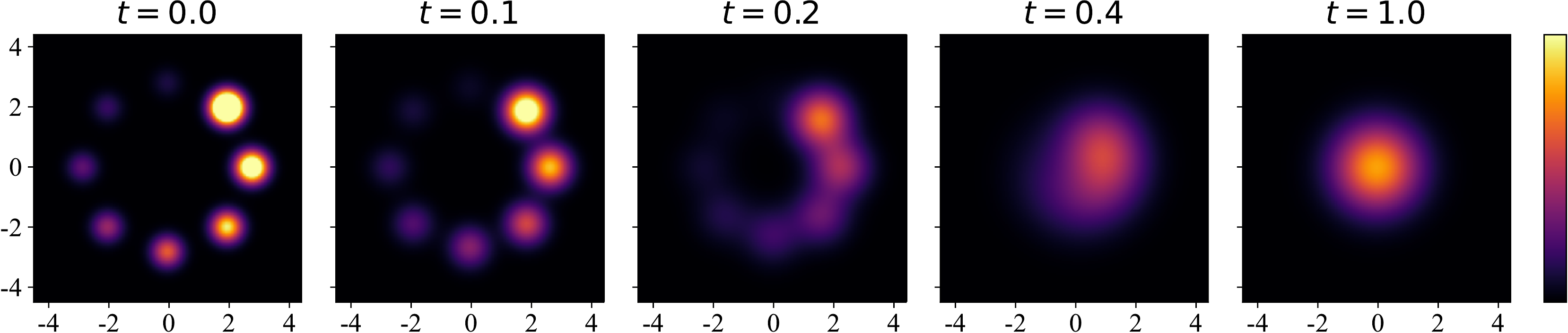}
	\end{minipage}
	\vspace{-0.05in}
	\caption{\label{fig:intermediate_energy}\small{A 2-D mixtures-of-Gaussians example of the density functions (unnormalized) for $q_t(\x_t)$, $e^{-\gE_t(\x_t)}$ and $p_t(\x_t)$ during the diffusion process, where $p_t(\x_t)\propto q_t(\x_t)e^{-\gE_t(\x_t)}$.}}
	\vspace{-0.15in}
\end{figure}

Below we formally analyze how to sample from \eqref{Eq:target_distribution_intro} by diffusion models.

Rewrite $p_0 \coloneqq p$ and $q_0\coloneqq q$. The target distribution is
\begin{equation}
\label{Eq:target_distribution}
    p_0(\x_0) \propto q_0(\x_0)e^{-\beta\gE(\x_0)}.
\end{equation}
Let $q_t(\x_t)$ be the marginal distribution of the forward diffusion process at time $t$ defined in \eqref{Eq:forward_diffusion} starting from $q_0(\x_0)$. 
Suppose we have a pretrained diffusion model $q^{\theta}_0(\x_0)\approx q_0(\x_0)$ by learning a noise prediction model $\epsilonv_\theta(\x_t,t)\approx -\sigma_t\nabla_{\x_t}\log q_t(\x_t)$ parameterized by $\theta$ at each time $t\in[0,T]$. By \eqref{Eq:diffusion_ode}, drawing samples from $p_0$ requires the corresponding score functions $\nabla_{\x_t}\log p_t(\x_t)$ at each intermediate time $t$ of the diffusion process starting from $p_0$ instead of $q_0$. 
Our first key result is on revealing the relationship between the corresponding score functions during the diffusion processes starting from $q_0$ and $p_0$, as summarized below:

\begin{theorem}[Intermediate Energy Guidance]
\label{thrm:intermediate_energy_guidance}
Suppose $q_0$ and $p_0$ are defined as in \eqref{Eq:target_distribution}. For $t\in(0, T]$, let
\begin{equation}
\label{Eq:forward_p}
    p_{t0}(\x_t|\x_0) \coloneqq q_{t0}(\x_t|\x_0) = \gN(\x_t|\alpha_t\x_0,\sigma_t^2\mI).
\end{equation}
Denote $q_t(\x_t):=\int q_{t0}(\x_t|\x_0) q_0(\x_0)\dm\x_0$ and $p_t(\x_t):=\int p_{t0}(\x_t|\x_0) p_0(\x_0)\dm\x_0$ as the marginal distributions at time $t$, and define
\begin{equation}
\label{Eq:intermediate_energy}
    \gE_t(\x_t)\! \coloneqq\! \left\{
    \begin{array}{ll}
        \beta\gE(\x_0), & t=0, \\
        -\log \E_{q_{0t}(\x_0|\x_t)}\left[
        e^{-\beta\gE(\x_0)}
        \right], & t > 0.
    \end{array} \right.
\end{equation}
Then $q_t$ and $p_t$ satisfy
\begin{equation}
    \label{Eq:pt_qt_Et}
    p_t(\x_t) \propto q_t(\x_t)e^{-\gE_t(\x_t)},
\end{equation}
and their score functions satisfy
\begin{equation}
\label{Eq:energy_guidance}
    \nabla_{\x_t}\log p_t(\x_t) = \underbrace{\nabla_{\x_t}\log q_t(\x_t)}_{\approx -\epsilonv_\theta(\x_t,t)/\sigma_t}
    - \underbrace{\nabla_{\x_t}\gE_t(\x_t)}_{\substack{\text{energy guidance}\\\text{\color{red}{\textbf{(intractable)}}}}}.
\end{equation}
\vspace{-0.2in}
\end{theorem}

\cref{thrm:infoNCE_energy} reveals a previously unnoticed exact form of the intermediate distributions $p_t$: though $p_t$ is defined as an intractable marginal distribution for all $t>0$, they could still be written \emph{in the same form} as  Eq.~(\ref{Eq:target_distribution}), proportional to the product of the (diffused) data distribution $q_t$ and an exponential energy term $e^{-\gE_t(\x_t)}$.
Since such energy is defined during the diffusion process, we name $\gE_t(\cdot)$ as \textit{intermediate energy}.
According to \eqref{Eq:intermediate_energy}, the intermediate energy is completely determined by the energy function $\gE(\x_0)$ at time $0$.
An illustration is given in \cref{fig:intermediate_energy}, from which we draw several observations: (1) $p_t(\x_t)$ prefers areas with both high data density $q_t(\x_t)$ and high energy density $e^{-\gE_t(\x_t)}$; (2) Through the forward diffusion process, both $p_t$ and $q_t$ gradually become standard Gaussian, which guarantees that we can reverse the diffusion process starting from the same noise distribution $p_T \approx q_T \approx \gN(0, \tilde\sigma^2\mI)$; (3) The energy function $\gE_0(\x_0)$ is also ``diffused'' into intermediate energy functions $\gE_t(\x_t)$ as $t$ increases. In particular, $\gE_T(\x_T)$ is almost equal to a constant function and thus $\nabla_{\x_T}\gE_T(\x_T) \approx 0$.

The result in \eqref{Eq:energy_guidance} directly defines a principled method to perform guided sampling from $p_0(\x_0)$. 
Namely, we can sample with \eqref{Eq:diffusion_ode} as along as we know both $\nabla_{\x_t}\log q_t(\x_t)$ and $\nabla_{\x_t}\gE_t(\x_t)$.
The former score is already given by the pretrained diffusion model $\epsilonv_\theta$. 
The remaining problem is to estimate the latter score $\nabla_{\x_t}\gE_t(\x_t)$, which we name as \emph{intermediate energy guidance}.
An unbiased estimation of $\nabla_{\x_t}\gE_t(\x_t)$ is generally non-trivial due to the log-expectation formulation and a potentially complex form of $\gE(\x_0)$ in \eqref{Eq:intermediate_energy}. To the best of our knowledge, it is still an open problem for estimating the exact intermediate energy guidance. 
We present a first attempt by developing a novel learning-based method to learn the energy guidance by comparing energy of samples from $q_t$, as detailed below.

\subsection{Learning Energy Guidance by Contrastive Energy Prediction}
\label{Sec:method}

Let $K>1$ be a positive integer. Let $\x_0^{(1)},\dots,\x_0^{(K)}$ be $K$ i.i.d. samples from $q_0(\x_0)$ and $\epsilonv^{(1)},\dots,\epsilonv^{(K)}$ be $K$ i.i.d. Gaussian samples following $p(\epsilonv)=\gN(\epsilonv|\bm{0}, \mI)$. Let $t \sim \mathcal{U}(0,T)$ be a randomly sampled time step. For each $i=1,\dots,K$, let $\x_t^{(i)}\coloneqq \alpha_t\x_0^{(i)} + \sigma_t\epsilonv^{(i)}$, where $\alpha_t,\sigma_t$ are defined in \eqref{Eq:forward_p}. Assume that the intermediate energy $\gE_t(\cdot)$ is approximated by a network $f_\phi(\cdot,t):\R^d\rightarrow \R$ parameterized by $\phi$. We propose to solve the following problem to learn $f_\phi$, whose solution is characterized in \cref{thrm:infoNCE_energy}:
\begin{equation}
\label{Eq:infoNCE_energy}
\begin{split}
    &\min_\phi \E_{p(t)}\E_{q_0(\x_0^{(1:K)})}\E_{p(\epsilonv^{(1:K)})}\Bigg[ \\
    &- \sum_{i=1}^K
    \underbrace{%
        \vphantom{\frac{\exp(-f_\phi(\x_t^{(i)},t))}{\sum_{j=1}^K \exp(-f_\phi(\x_t^{(j)},t))}} 
         e^{-\beta\gE(\x_0^{(i)})}
    }_{\text{soft energy label}}
    \log \underbrace{\frac{e^{-f_\phi(\x_t^{(i)},t)}}{\sum_{j=1}^K e^{-f_\phi(\x_t^{(j)},t)}}}_{\text{predicted label}}
\Bigg].
 \end{split}
\end{equation}
\begin{theorem}
\label{thrm:infoNCE_energy}

Given unlimited model capacity and data samples, For all $K>1$ and $t\in[0,T]$, the optimal $f_{\phi^*}$ in problem (\ref{Eq:infoNCE_energy}) satisfies $\nabla_{\x_t}f_{\phi^*}(\x_t,t) = \nabla_{\x_t}\gE_t(\x_t)$.
\end{theorem}

According to \cref{thrm:infoNCE_energy}, we indeed can train the energy guidance model $f_\phi$ by solving problem (\ref{Eq:infoNCE_energy}) and use the gradients of $f_\phi$ for guided sampling by estimating the energy guidance in \eqref{Eq:energy_guidance}.
Here we give an intuitive explanation of \cref{thrm:infoNCE_energy}. To estimate the energy guidance $\nabla_{\x_t}\gE_t(\cdot)$, we only need to ensure $f_\phi$ to be a relative proportional value of $\gE_t(\cdot)$, so it is enough to relatively compare $f_\phi$ within $K$ samples instead of directly train $f_\phi$ with the absolute values. Built upon such an observation, we leverage a cross-entropy loss in problem (\ref{Eq:infoNCE_energy}), where the energy $\gE(\x_0^{(i)})$ of $K$ clean samples $\x_0^{(i)}$ are soft supervising labels and the softmax of energy predictions $f_\phi(\x_t^{(i)},t)$ of $K$ noisy samples $\x_t^{(i)}$ are predicted labels.
Due to the contrastive manner of this objective, we name our proposed method in \eqref{Eq:infoNCE_energy} as \textit{Contrastive Energy Prediction (CEP)}.

Although the optimal solution in problem (\ref{Eq:infoNCE_energy}) is exact, sometimes we may suffer from numerical issues during training because the exponential term $e^{-\beta\gE(\x_0^{(i)})}$ is unnormalized. For example, suppose $\gE(\x_0)$ is a complex function that might contain ``spikes'' at some data point, the exponential term will greatly amplify such instability during training. To address this issue, we further use a self-normalized energy label by normalizing the energy function across the $K$ samples and define the optimization problem as:
\begin{equation}
\label{Eq:infoNCE_energy_normalized}
\begin{split}
    &\min_\phi \E_{p(t)}\E_{q_0(\x_0^{(1:K)})}\E_{p(\epsilonv^{(1:K)})}\Bigg[ \\
    &- \sum_{i=1}^K \!
    \underbrace{\frac{e^{-\beta\gE(\x_0^{(i)})}}{\sum_{j=1}^K e^{-\beta\gE(\x_0^{(j)})}}}_{\text{self-normalized energy label}}
    \log \underbrace{\frac{e^{-f_\phi(\x_t^{(i)},t)}}{\sum_{j=1}^K e^{-f_\phi(\x_t^{(j)},t)}}}_{\text{predicted label}}
\Bigg].
\end{split}
\end{equation}
Such a self-normalized objective can ensure the soft energy label is within $[0,1]$ and the sum of each item is exactly $1$. Although this objective may be biased against the original one in problem (\ref{Eq:infoNCE_energy}), we empirically find that it can greatly improve the numerical stability and help to achieve a good converged result. Moreover, we can reduce bias in the objective by increasing $K$. For $K\rightarrow \infty$, the objective in (\ref{Eq:infoNCE_energy_normalized}) is equivalent to that in (\ref{Eq:infoNCE_energy}) because $\sum_{\x_0}  e^{-\beta\gE(\x_0)} = \E_{q_0(\x_0)}[e^{-\beta\gE(\x_0)}]$ is the normalizing constant of $p_0$. Therefore, a larger $K$ is preferred in practice given enough computation and memory budget.

\section{Comparison with Previous Methods for Guided Sampling}
Below we compare CEP with previous methods for guided sampling.
We show that all previous energy-guided samplers are inexact; and for a special case of the energy function which corresponds to the conditional sampling problem, CEP is a contrastive alternative to classifier guidance.

\begin{figure}[ht]
\vspace{-0.1in}
\centering
\begin{minipage}{0.12\linewidth}
\rotatebox{90}{\scriptsize{Ground Truth}}
\end{minipage}
\begin{minipage}{0.86\linewidth}
\includegraphics[width=\columnwidth]{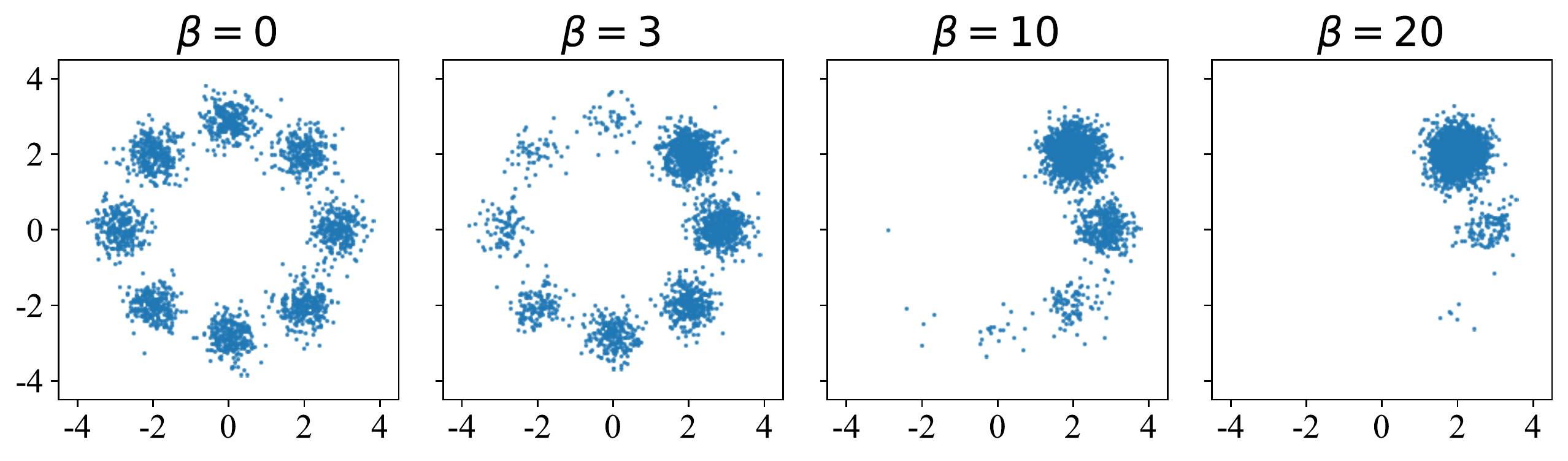}
\end{minipage}

\begin{minipage}{0.06\linewidth}
\rotatebox{90}{\scriptsize{CEP}}
\end{minipage}
\begin{minipage}{0.06\linewidth}
\rotatebox{90}{\scriptsize{(\textbf{ours})}}
\end{minipage}
\begin{minipage}{0.86\linewidth}
\includegraphics[width=\columnwidth]{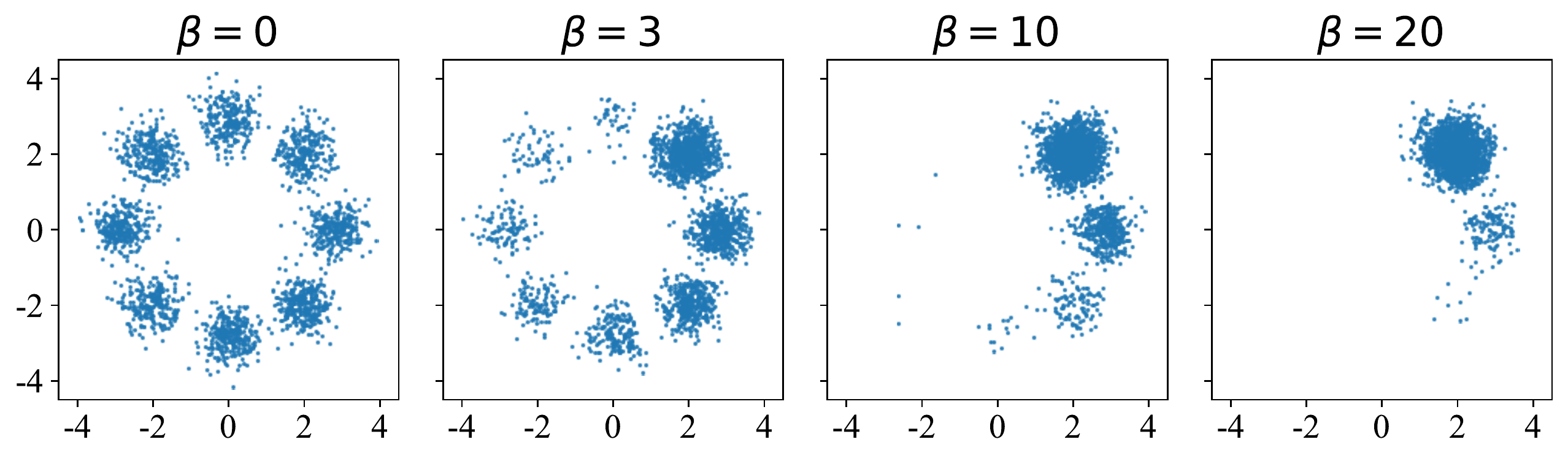}
\end{minipage}

\begin{minipage}{0.06\linewidth}
\rotatebox{90}{\scriptsize{MSE}}
\end{minipage}
\begin{minipage}{0.06\linewidth}
\rotatebox{90}{\scriptsize{(\citeauthor{diffuser})}}
\end{minipage}
\begin{minipage}{0.86\linewidth}
\includegraphics[width=\columnwidth]{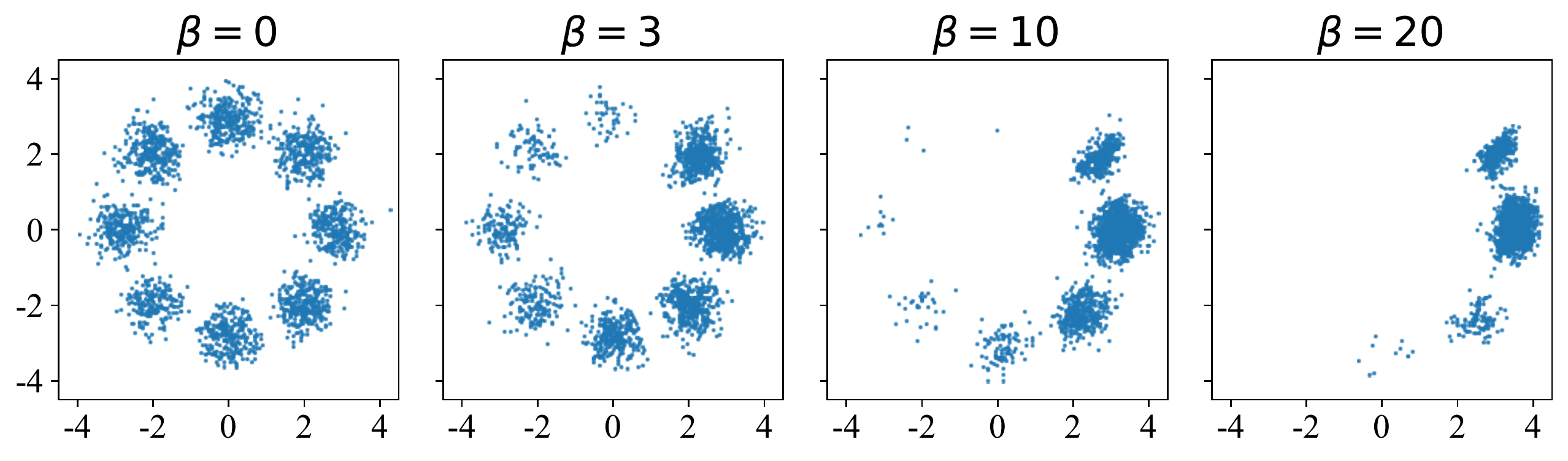}
\end{minipage}

\begin{minipage}{0.06\linewidth}
\rotatebox{90}{\scriptsize{DPS}}
\end{minipage}
\begin{minipage}{0.06\linewidth}
\rotatebox{90}{\scriptsize{(\citeauthor{chung2022diffusion})}}
\end{minipage}
\begin{minipage}{0.86\linewidth}
\includegraphics[width=\columnwidth]{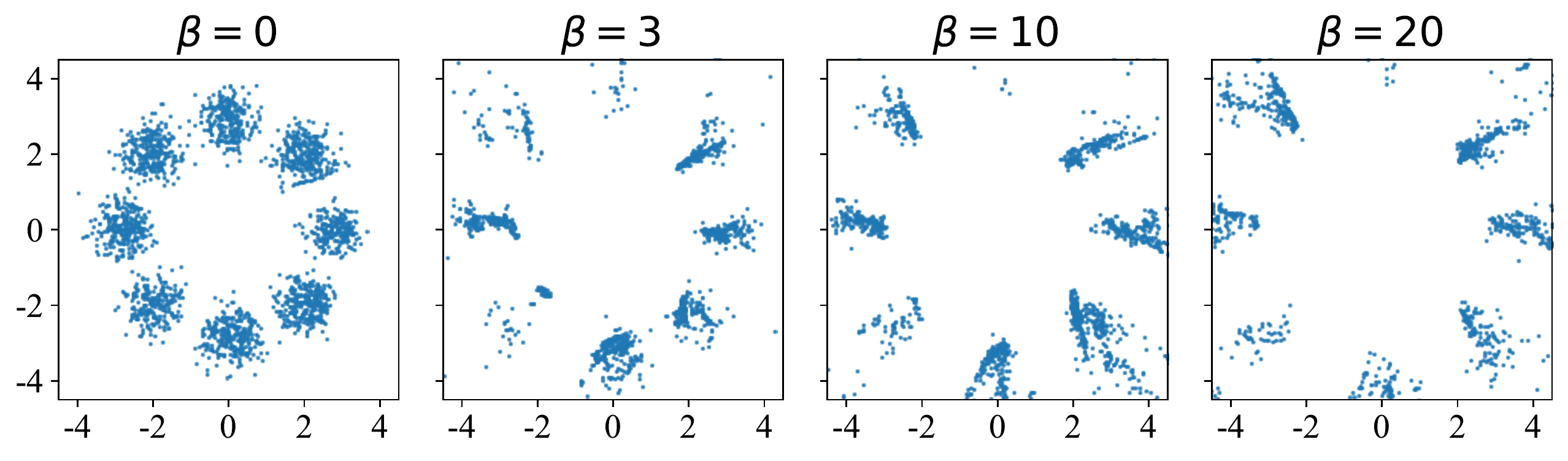}
\end{minipage}

\caption{
\small{A 2-D example for comparing different energy-guided sampling algorithms, varying different inverse temperature $\beta$.}
}
\label{fig:toy_main}
\end{figure}

\begin{table}[t]
    \vspace{-.1in}
    \centering
    \caption{\small{Comparison between energy-guided sampling algorithms.}}
    \vskip 0.1in
    \begin{small}
    \resizebox{0.48\textwidth}{!}{%
    \begin{tabular}{llc}
    \toprule
    Method & Optimal Solution of Energy &
    Exact Guidance \\
    \midrule
    \vspace{0.03in}
    CEP (ours) & $-\log\E_{q_{0t}(\x_0|\x_t)}\left[e^{-\gE_0(\x_0)}\right]$ & \color{green}{\ding{51}} \\
    \vspace{0.03in}
    MSE & $\E_{q_{0t}(\x_0|\x_t)}[\gE_0(\x_0)]$ & \color{red}{\ding{55}} \\
    DPS & $\gE_0\left(\E_{q_{0t}(\x_0|\x_t)}[\x_0]\right)$ & \color{red}{\ding{55}} \\
    \bottomrule
    \end{tabular}%
    }
    \end{small}
    \label{tab:comparison}
    \vspace{-0.15in}
\end{table}

\subsection{Previous Energy-Guided Samplers are Inexact}
\label{sec:compare_energy}
In this section, we show that previous energy-guided samplers for \eqref{Eq:target_distribution} are all inexact and do not guarantee convergence to $p_0$. 
Without loss of generality, we focus on a fixed time $t\in(0,T]$. 
We summarize the relationship in \cref{tab:comparison}.

\paragraph{MSE for Predicting Energy.}
Many existing energy guidance methods \citep{diffuser,bao2022equivariant} use a mean-square-error (MSE) objective to train an energy model $f_\phi(\x_t, t)$ and use its gradient for energy guidance. The training objective is:
\begin{equation}
\label{Eq:MSE_qt_objective}
    \min_{\phi} \E_{q_{0t}(\x_0, \x_t)} \left[
    \|f_\phi(\x_t, t) - \gE_0(\x_0)\|_2^2
\right]
\end{equation}
Given unlimited model capacity, the optimal $f_\phi$ satisfies:
\begin{equation*}
\begin{split}
    f_\phi^\text{MSE}(\x_t,t) = \E_{q_{0t}(\x_0|\x_t)}\left[
        \gE_0(\x_0)
    \right]
\end{split}
\end{equation*}
However, by \eqref{Eq:intermediate_energy}, the true energy function satisfies
\begin{equation*}
\begin{aligned}
    \gE_t(\x_t) &= -\log \E_{q_{0t}(\x_0|\x_t)}\left[
        e^{-\gE_0(\x_0)} \right] \\
        &\geq \E_{q_{0t}(\x_0|\x_t)}\left[
        \gE_0(\x_0)\right] = f_\phi^\text{MSE}(\x_t,t),
\end{aligned}
\end{equation*}
and the equality only holds when $t=0$. Therefore, the MSE energy function $f_\phi^{\text{MSE}}$ is inexact for all $t>0$. Moreover, we show in Appendix~\ref{appendix:relationship} that the gradient of $f_\phi^{\text{MSE}}$ is also inexact against the true guidance $\nabla_{\x_t}\gE(\x_t)$.

\paragraph{Diffusion Posterior Sampling.}
There also exist some training-free algorithms for energy-guided sampling, such as reconstruction guidance~\citep{ho2022video} and diffusion posterior sampling~\citep{chung2022diffusion}. The basic idea in these methods is to reuse the pretrained diffusion model in the data prediction formulation~\citep{kingma2021variational}:
\begin{equation*}
     \E_{q_{0t}(\x_0|\x_t)}[\x_0] \approx \hat\x_\theta(\x_t,t) \coloneqq \frac{\x_t - \sigma_t\epsilonv_\theta(\x_t,t)}{\alpha_t},
\end{equation*}
and then define the intermediate energy function by:
\begin{equation*}
    f_\theta^{\text{DPS}}(\x_t,t) \coloneqq \gE_0(\hat\x_\theta(\x_t,t)) \approx \gE_0(\E_{q_{0t}(\x_0|\x_t)}[\x_0]).
\end{equation*}
However, we show in Appendix~\ref{appendix:relationship} that $\nabla_{\x_t}f_\theta^{\text{DPS}}\neq \nabla_{\x_t}\gE_t$ and thus it is also inexact.

\paragraph{2-D Example.}
We further compare the different methods for energy-guided sampling on a 2-D example, as shown in Fig.~\ref{fig:toy_main}, and provide more 2-D results in Appendix~\ref{appendix:toy_more}. Experiments show that our method outperforms all referred methods, especially when the inverse temperature $\beta$ is large.

\subsection{Relationship with Contrastive Learning and Classifier Guidance}
\label{sec:compare_classifier_guidance}
In this section, we consider a special case of our method in which the energy function $\gE_0(\x_0)$ is defined as negative log-likelihood $-\log q_0(c|\x_0)$ for a given conditioning variable $c$ with $\beta=1$. In such case, the desired distribution is:
\begin{equation*}
    p_0(\x_0) \propto q_0(\x_0)q(c|\x_0) \propto q(\x_0|c).
\end{equation*}
Different from the problem we consider in \eqref{Eq:target_distribution} that $p_0(\x_0)$ is hard to draw samples from, here we assume that we can draw samples from $p_0(\x_0)=q_0(\x_0|c)$. Following such an assumption, we prove in Appendix~\ref{appendix:relationship_infonce} that our proposed CEP in \eqref{Eq:infoNCE_energy} is equivalent to
\begin{equation}
\label{Eq:contarstive_condition_loss}
\begin{split}
    \E_{t,\epsilonv^{(1:K)}}\E_{\prod_{i=1}^K q_0(\x_0^{(i)},c^{(i)})}\Bigg[\!\!
    -\sum_{i=1}^K \log \frac{e^{-f_\phi(\x_t^{(i)},c^{(i)},t)}}{\sum_{j=1}^K e^{-f_\phi(\x_t^{(j)},c^{(i)},t)}}
\Bigg],
\end{split}
\end{equation}
where $(\x_0^{(i)}, c^{(i)})$ are $K$ paired data samples from $q_0(\x_0,c)$.
Note that the inner expectation has the same form as the InfoNCE objective~\citep{infoNCE} and is widely used in contrastive learning for multi-modal data, such as CLIP~\citep{radford2021learning} (where $f_\phi$ represents cosine similarity).
Furthermore, \citet{nichol2021glide} uses the above objective and trains a CLIP at each $t$ for text-image pairs and uses its gradient as guidance for text-to-image sampling by diffusion models.
Therefore, such guidance can be considered as a special case of CEP in \eqref{Eq:infoNCE_energy}, under the assumption that we can draw samples from $p_0(\x_0)=q_0(\x_0|c)$.

Moreover, if $c$ is a discrete variable with a total of $M$ possible values (classes). An alternative guided sampling method for sampling from $q_0(\x_0|c)$ is classifier guidance~\citep{diffusion_beat_gan}, which optimize the following objective:
\begin{equation}
\label{Eq:classifier_guidance_loss}
    \E_{t,\epsilonv^{(1:K)}}\E_{\prod_{i=1}^K q_0(\x_0^{(i)},c^{(i)})}\Bigg[\!\!
    -\sum_{i=1}^K \log \frac{e^{-f_\phi(\x_t^{(i)},c^{(i)},t)}}{\sum_{j=1}^M e^{-f_\phi(\x_t^{(i)},c^{(j)},t)}}
\Bigg],
\end{equation}

The most notable difference between \eqref{Eq:contarstive_condition_loss} and \eqref{Eq:classifier_guidance_loss} is the normalizing axes in the training objective's denominator. Classifier guidance aims to \textit{classify conditions} for a given data $\x_t^{(i)}$, so the objective could be understood as a classification loss which is normalized across all $c^{(j)}$. CEP is trying to \textit{compare within data} for a specified condition $c^{(i)}$, so the objective could be understood as a contrastive loss which is normalized across all $\x_t^{(j)}$.

Theoretically, both classifier guidance and CEP could guarantee exact guidance given unlimited data and model capacity. Experimentally, we show in Sec.~\ref{sec:image_exp} that the two methods have quite similar performance.
However, traditional classifier guidance cannot be applied to energy-guided sampling because there does not exist a set of conditions $c$ across which we could normalize, whereas our proposed method can.
We thus conclude that CEP could be considered as a contrastive alternative to classifier guidance for conditional sampling, but is in a more general form that could transfer to the energy-guided sampling problem.

\section{Q-Guided Policy Optimization for Offline Reinforcement Learning}
\begin{table*}[t]
\centering
\small
\resizebox{1.0\textwidth}{!}{%
\begin{tabular}{llccccccccc}
\toprule
\multicolumn{1}{c}{\bf Dataset} & \multicolumn{1}{c}{\bf Environment} & \multicolumn{1}{c}{\bf CQL}& \multicolumn{1}{c}{\bf BCQ}  & \multicolumn{1}{c}{\bf IQL} & \multicolumn{1}{c}{\bf SfBC} & \multicolumn{1}{c}{\bf DD} & \multicolumn{1}{c}{\bf Diffuser} & \multicolumn{1}{c}{\bf D-QL}& \multicolumn{1}{c}{\bf D-QL@1} & \multicolumn{1}{c}{\bf QGPO (ours)} \\
\midrule
Medium-Expert & HalfCheetah    & $62.4$               &  $64.7    $ &  $86.7   $  & $\bf{92.6}$   & $90.6$     & $79.8$       &  $\bf{96.1}$& $ \bf{94.8}$ & $\bf{93.5\pm0.3}$               \\
Medium-Expert & Hopper         & $98.7$               &  $100.9$    &  $    91.5$ & $\bf{108.6}$  &$\bf{111.8}$& $\bf{107.2}$ &$\bf{110.7}$ & $100.6$      & $\bf{108.0\pm2.5}$           \\
Medium-Expert & Walker2d       & $\bf{111.0}$         & $57.5$      & $\bf{109.6}$& $\bf{109.8}$  &$\bf{108.8}$& $\bf{108.4}$ & $\bf{109.7}$& $\bf{108.9}$ & $\bf{110.7\pm0.6}$           \\
\midrule
Medium        & HalfCheetah    &  $44.4$              &  $40.7$     &  $47.4     $& $45.9$        & $49.1$     & $44.2$       &  $50.6$     & $47.8$       & $\bf{54.1\pm0.4}$           \\
Medium        & Hopper         &  $58.0$              &  $54.5$     &  $66.3$     & $57.1$        & $79.3$     & $58.5$       &$ 82.4      $& $64.1$       & $\bf{98.0\pm2.6}$ \\
Medium        & Walker2        &  $79.2$              &  $53.1$     &  $78.3$     & $77.9$        & $\bf{82.5}$ & $79.7$       &$\bf{85.1}$  & $82.0$       & $\bf{86.0\pm0.7}$ \\
\midrule
Medium-Replay & HalfCheetah    &  $\bf{46.2}$         &  $38.2$     &  $44.2$     &   $37.1$      &$39.3$& $42.2$       &$\bf{47.5}$  & $44.0$       & $\bf{47.6\pm1.4}$ \\
Medium-Replay & Hopper         &  $48.6$              &  $33.1$     &  $94.7$     &   $86.2$      & $\bf{100.0}$     & $\bf{96.8}$  & $\bf{100.7}$& $63.1$       & $\bf{96.9\pm2.6}$ \\
Medium-Replay & Walker2d       &  $26.7$              &  $15.0$     &  $73.9$     &   $65.1$      & $75.0$     & $61.2$       &  $\bf{94.3}$& $75.4$       & $84.4\pm4.1$ \\
\midrule
\multicolumn{2}{c}{\bf Average (Locomotion)}&$63.9$   &  $51.9$     & $76.9$      &   $75.6$      &$81.8$      &      $75.3$       & $\bf{86.3}$ & $75.6$       & $\bf{86.6}$ \\
\midrule
Default       & AntMaze-umaze  &  $74.0$              &  $78.9$     & $87.5$      & $\bf{92.0}$   & -          & -            & $68.6$      & $69.4$       & $\bf{96.4\pm1.4}$ \\
Diverse       & AntMaze-umaze  &  $\bf{84.0}$         &  $55.0  $   & $62.2$      & $\bf{85.3}$   & -          & -            &$53.0$       & $56.4$       & $74.4\pm9.7$ \\
\midrule  
Play          & AntMaze-medium &  $61.2$              &  $0.0$      & $71.2$      & $\bf{81.3}$   & -          & -            & $0.0$       & $1.0$        & $\bf{83.6\pm4.4}$ \\
Diverse       & AntMaze-medium &  $53.7$              &  $0.0$      & $70.0$      & $\bf{82.0}$   & -          & -            & $18.4$      & $14.8$       & $\bf{83.8\pm3.5}$ \\
\midrule
Play          & AntMaze-large  &  $15.8$              &  $6.7$      & $39.6$      & $59.3$        & -          & -            & $10.6$      & $15.8$       & $\bf{66.6\pm9.8}$ \\
Diverse       & AntMaze-large  &  $14.9$              &  $2.2$      & $47.5$      & $45.5$        & -          & -            & $4.2$       &  $1.6$       & $\bf{64.8\pm5.5}$ \\
\midrule
\multicolumn{2}{c}{\bf Average (AntMaze)}&  $50.6$    &  $23.8$     & $63.0$      &   $74.2$      & -            &   -   &  $25.8$    & $26.5$       & $\bf{78.3}$ \\
\specialrule{.05em}{.4ex}{.1ex}
\specialrule{.05em}{.1ex}{.65ex}

\multicolumn{2}{c}{\bf{\# Action candidates}}&$1$   &  $100$     & $1$  & $32$  &   $1$     & $1$  &   $50$  & $1$       & $1$ \\
\multicolumn{2}{c}{\bf{\# Diffusion steps}}&-   &  -     & -  & $15$  &   $100$     & $100$  &   $5$   & $5$       & $15$ \\
\bottomrule

\end{tabular}
}
\caption{
Evaluation numbers of D4RL benchmarks (normalized as suggested by \citet{d4rl}). We report mean and standard deviation of algorithm performance across 5 random seeds at the end of training. Numbers within 5 percent of the maximum in every individual task are highlighted. We rerun the experiments of Diffusion-QL to ensure a consistent evaluation metric. See \cref{Sec:rl_detail} for details. }
\label{tbl:rl_results}
\vspace{-0.2in}
\end{table*}

In this section, we showcase how our method can be applied in offline RL, including problem formulation in \cref{sec:rl_formulation}, algorithm method in \cref{sec:rl_method} and \cref{softmax_ql}, and experimental results in \cref{Sec:rl_result}. A pseudocode is provided in \cref{Sec:pseudocode}.
\subsection{Problem Formulation}
\label{sec:rl_formulation}
Recall that from \eqref{Eq:pi_optimal}, our desired policy $\pi^*$ follows $\pi^*(\va|\vs) \propto \ \mu(\va|\vs) \ e^{\beta Q_\psi(\vs, \va)}$, where the behavior policy $\mu(\va|\vs)$ is a diffusion model.
In order to sample actions from $\pi^*$ by diffusion sampling, we denote $\mu_0 : = \mu$, $\pi_0 : = \pi$, $\va_0:=\va$ at time $t=0$. Then we construct a forward diffusion process to simultaneously diffuse $\mu_0$ and $\pi_0$ into the same noise distribution, where $\pi_{t0}(\va_t|\va_0, \vs) \coloneqq \mu_{t0}(\va_t|\va_0, \vs) = \gN(\va_t|\alpha_t\va_0,\sigma_t^2\mI)$. 

According to \cref{thrm:intermediate_energy_guidance}, by replacing the distribution $q$ with $\mu$, $p$ with $\pi$, and the energy function $\gE$ with $-Q$ following conventions in offline RL literature, we have the marginal distributions $\mu_t$ and $\pi_t$ of the noise-perturbed action $\va_t$ satisfy:
\begin{equation}
    \label{Eq:pi_optimal_t_space}
    \pi_t(\va_t|\vs) \propto \mu_t(\va_t|\vs) e^{\gE_t(\vs, \va_t)}.
\end{equation}
$\gE_t(\vs, \va_t)$ is an intermediate energy function determined by the learned action evaluation model $Q_\psi(\vs, \va_0)$. Specifically $\gE_t(\vs, \va_t) = \log \E_{\mu_{0t}(\va_0|\va_t, \vs)}\left[e^{\beta Q_\psi(\vs, \va_0)}\right]$ and $\gE_0(\vs, \va_0)=\beta Q_\psi(\vs, \va_0)$.

We now consider how to estimate the score function of $\pi_t(\va|\vs)$ such that we can sample actions from $\pi_0$ following \eqref{Eq:diffusion_ode}. By \eqref{Eq:pi_optimal_t_space}, we have:
\begin{equation}
\label{Eq:rl_energy_guidance}
    \nabla_{\va_t}\log \pi_t(\va_t|\vs) = \underbrace{\nabla_{\va_t}\log \mu_t(\va_t|\vs)}_{\approx -\epsilonv_\theta(\va_t|\vs,t)/\sigma_t}
    + \nabla_{\va_t}\underbrace{\gE_t(\vs, \va_t)}_{\approx f_\phi(\vs, \va_t, t)}.
\end{equation}

To this end, we have formulated the classic constrained policy optimization problem (\ref{Eq:rl_main}) as energy-guided sampling, with $\nabla_{\va_t}\gE_t(\vs, \va_t)$ being the desired guidance. Because such guidance is determined by the Q function, we name this approach as \textit{Q-guided policy optimization (QGPO)}. 
QGPO requires training a total of three neural networks in order to estimate the targeted score function $\nabla_{\va_t}\log \pi_t(\va_t|\vs)$: (1) a state-conditioned diffusion model $\epsilonv_\theta(\va_t| \vs, t)$ to model the behavior policy $\mu(\va|\vs)$, for which we completely follow \citet{sfbc}; (2) an action evaluation model $Q_\psi(\vs, \va)$ to define the intermediate energy function $\gE_t$ when $t=0$ (\cref{softmax_ql}); and (3) an energy model $f_\phi(\vs, \va_t, t)$ to estimate $\gE_t(\vs, \va_t, t)$ and guide the diffusion sampling process when $t>0$ (\cref{sec:rl_method}).

\subsection{In-Support Contrastive Energy Prediction}
\label{sec:rl_method}
Suppose we already have an action evaluation model $Q_\psi(\vs, \va)$ to estimate the Q-function $Q^{\pi}(\vs, \va)$. According to \cref{thrm:infoNCE_energy}, $f_\phi(\vs, \va_t, t)$ can be trained via our proposed CEP. Rewriting \eqref{Eq:infoNCE_energy} to condition all distributions on state $\vs$, the problem for learning $f_\phi$ becomes:
\begin{equation}
\label{Eq:rl_contrastive}
\begin{split}
    &\min_\phi \E_{p(t)}\E_{\mu(\vs)}\E_{\prod_{i=1}^K \mu(\va^{(i)}|\vs)p(\epsilonv^{(i)})}\Bigg[\\
    &- \sum_{i=1}^K
    \frac{e^{\beta Q_\psi(\vs, \va^{(i)})}}{\sum_{j=1}^K e^{\beta Q_\psi(\vs, \va^{(j)})}}
    \log \frac{e^{f_\phi(\vs, \va_t^{(i)}, t)}}{\sum_{j=1}^K e^{f_\phi(\vs, \va_t^{(j)}, t)}}
\Bigg],
\end{split}
\end{equation}
where $t \sim \mathcal{U}(0,T)$, $\va_t = \alpha_t\va+\sigma_t\bm{\epsilon}$ and $\bm{\epsilon}\sim \mathcal{N}(0,\bm{I})$.

One difficulty in solving the above problem is that we have no access to the true distribution $\mu(\va|\vs)$ for a specified $\vs$. Although we can sample data from the joint distribution $\mu(\vs, \va)$ or the marginal distribution $\mu(\vs)$ given the offline dataset $\mathcal{D}^\mu$, such data samples cannot be directly used to estimate the objective in problem (\ref{Eq:rl_contrastive}). This is because we require $K>1$ independent action samples from $\mu(\va|\vs)$ for a single $\vs$ for contrastive learning, whereas we only have one such action in $\mathcal{D}^\mu$ given that $\vs$ is a continuous variable.

To address this issue, we propose to pre-generate a support action set $\mathcal{D}^{\mu_\theta}$ using the already learned behavior model $\mu_\theta(\va_t|\vs,t)$. Concretely, for each state $\vs$ in the behavior dataset $\mathcal{D}^{\mu}$, we sample $K$ support actions $\{\hat \va^{(i)}\}_K$ from $\mu_\theta(\cdot|\vs)$ and store these actions in pair with the state $\vs$ in $\mathcal{D}^{\mu_\theta}$. We then estimate the objective in (\ref{Eq:rl_contrastive}) with $\mathcal{D}^{\mu_\theta}$:
\begin{equation}
\label{Eq:rl_loss}
    \min_\phi \E_{t,\vs,\epsilonv} - \sum_{i=1}^K
    \frac{e^{\beta Q_\psi(\vs, \hat \va^{(i)})}}{\sum_{j=1}^K e^{\beta Q_\psi(\vs, \hat \va^{(j)})}}
    \log \frac{e^{f_\phi(\vs, \hat \va_t^{(i)}, t)}}{\sum_{j=1}^K e^{f_\phi(\vs, \hat \va_t^{(j)}, t)}}
\end{equation}
where $\hat \va^{(i)}$, $\hat \va^{(j)}$ are support actions for $\mathcal{D}^{\mu_\theta}(\vs)$. Since problem (\ref{Eq:rl_loss}) is optimized in a support action set instead of the true dataset, we refer to it as in-support CEP.

\subsection{In-support Softmax Q-Learning}
\label{softmax_ql}
We now discuss in detail how the action evaluation model $Q_\psi \approx Q^{\pi}$ could be trained. Ideally, we can use a typical Bellman-style bootstrapping method to calculate the mean square error (MSE) training target of $Q_\psi$ \cite{dql, arq}:
\begin{equation}
    \label{Eq:one_step_bellman}
    \mathcal{T}^\pi Q_\psi(\vs, \va) = r(\vs, \va) + \gamma\mathbb{E}_{\vs' \sim P(\cdot|\vs,\va), \va' \sim \pi(\cdot|\vs')} Q_\psi(\vs', \va').
\end{equation}
However, calculating $\mathcal{T}^\pi Q_\psi(\vs, \va)$ could in practice be time-consuming, because it requires sampling from a diffusion model $\pi$ during training. We thus leverage the generated support action set $\mathcal{D}^{\mu_\theta}$ to avoid repeated sampling from a diffusion model. Specifically, we estimate $\mathcal{T}^\pi Q_\psi(\vs, \va)$ via importance sampling:
\begin{equation}
    \label{Eq:one_step_bellman_softmax}
    \mathcal{T}^\pi Q_\psi(\vs, \va) \approx r(\vs, \va) + \gamma \frac{\sum_{\hat\va'}e^{\beta_Q Q_\psi(\vs',\hat\va')}Q_\psi(\vs',\hat\va')}{\sum_{\hat\va'} e^{\beta_Q Q_\psi(\vs',\hat\va')}}.
\end{equation}

\begin{table*}[t]
    \begin{center}
    \begin{small}
    \resizebox{0.75\textwidth}{!}{%
    \begin{tabular}{ccccccccc}
    \toprule
    Conditional &  Resolution & Diffusion Steps & FID   & sFID    & Precision & Recall  \\
    \midrule
    \cmark      & 128$\times$128 & 250  & 3.17 / 2.97  &  5.17 / 5.09   & 0.78 / 0.78 & 0.59 / 0.59 \\
    \cmark      & 128$\times$128 & 25  &  6.15 / 5.98  &  6.97 / 7.04  &  0.79 / 0.78 & 0.51 / 0.51 \\
    \hline
    \cmark      & 256$\times$256 &  250 &  4.74 / 4.59  &  5.23 / 5.25  & 0.82 / 0.82      & 0.52 / 0.52 \Tstrut \\
    \cmark      & 256$\times$256 &  25 & 5.58 / 5.44 &	5.25 / 5.32  &  0.82 / 0.81      & 0.48 / 0.49 \Bstrut \\
    \hline
    \xmark      & 256$\times$256 &  250 & 32.53 / 33.03 & 7.23 / 6.99  & 0.56 / 0.56      &  0.65 / 0.65 \Tstrut \\
    \bottomrule
    \end{tabular}%
    }
    \end{small}
    \end{center}
    \vspace{-0.1in}
    \caption{Effect of CEP guidance (left) on image sample quality compared with classifier guidance (right).}
    \vspace{-0.15in}
    \label{tab:guide2}
\end{table*}

\subsection{Results}
\label{Sec:rl_result}
We compare the performance of QGPO with several related works in multiple D4RL \citep{d4rl} tasks in \cref{tbl:rl_results}. Among them, \texttt{MuJoCo locomotion} tasks are popular benchmarks in offline RL and mainly aim to drive different robots moving forward as fast as possible. The dataset might contain a mixture of expert-level and medium-level policies' decision data (Medium-Expert), decision data generated by a single medium-level policy (Medium), and diverse decision data generated by a large set of medium-level policies (Medium-Replay). \texttt{Antmaze} tasks are typically considered to be hard tasks for RL-based methods. They aim to navigate an ant robot in several prespecified mazes (Umaze, Medium, Large). The learned policy directly outputs an eight-dimensional motor torque to control the motor motion of the ant robot at each degree of freedom. As a result, Antmaze tasks require policies to perform both low-level motion control and high-level navigation.

From \cref{tbl:rl_results}, we can see that in most tasks, our method outperforms referenced baselines, especially in difficult tasks such as Antmaze-Large. Baselines include traditional state-of-the-art algorithms like CQL \citep{cql}, BCQ \citep{bcq}, and IQL \citep{iql}, which adopt Gaussian-like policies. We also include recent advances in offline RL that adopt diffusion-based policies. Diffusers \cite{diffuser} considers using an energy guidance method that ensembles the MSE-based method as described in \eqref{Eq:MSE_qt_objective}. Decision Diffuser (DD, \citet{dd}), on the other hand, explores using the classifier-free guidance \citep{ho2021classifier}. Diffusion-QL (D-QL, \citet{dql}) tracks the gradients of the actions sampled from the behavior diffusion policy to guide generated actions to high Q-value area. SfBC \citep{sfbc} simply resamples actions from multiple behavior action candidates using the predicted Q-value as sampling weights. Note that such a resampling trick is also shared by Diffusion-QL. We also study a variant of Diffusion-QL (D-QL@1) where the resampling trick is removed to better reflect the quality of decisions generated by the diffusion policy.

\section{Image Synthesis Examples with CEP}
\label{sec:image_exp}

\begin{figure}[t]
\vspace{-.05in}
	\begin{minipage}{0.18\linewidth}
		\centering
			\includegraphics[width=\linewidth]{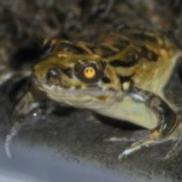}\\
			\small{$s=0.0$} \\
	\end{minipage}
	\begin{minipage}{0.18\linewidth}
		\centering
			\includegraphics[width=\linewidth]{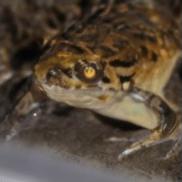}\\
			\includegraphics[width=\linewidth]{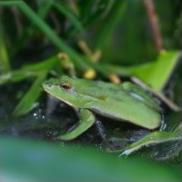}\\
			\includegraphics[width=\linewidth]{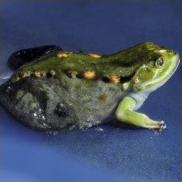}\\
		    \small{$s=1.0$} \\
	\end{minipage}
	\begin{minipage}{0.18\linewidth}
		\centering
			\includegraphics[width=\linewidth]{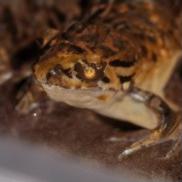}\\
			\includegraphics[width=\linewidth]{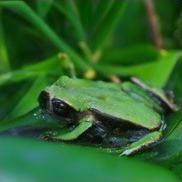}\\
			\includegraphics[width=\linewidth]{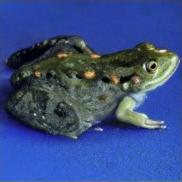}\\
		    \small{$s=2.0$} \\
	\end{minipage}
	\begin{minipage}{0.18\linewidth}
		\centering
			\includegraphics[width=\linewidth]{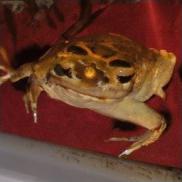}\\
			\includegraphics[width=\linewidth]{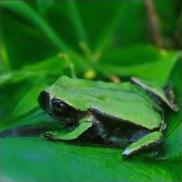}\\
			\includegraphics[width=\linewidth]{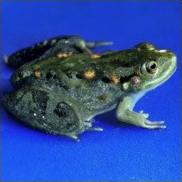}\\
		    \small{$s=3.0$} \\
	\end{minipage}
	\begin{minipage}{0.18\linewidth}
		\centering
			\includegraphics[width=\linewidth]{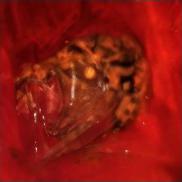}\\
			\includegraphics[width=\linewidth]{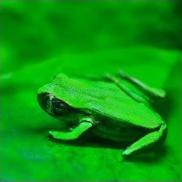}\\
			\includegraphics[width=\linewidth]{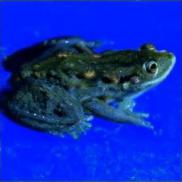}\\
		    \small{$s=10.0$} \\
	\end{minipage}
	\begin{minipage}{0.04\linewidth}
		\centering
			\small{\rotatebox{90}{red}}\\
			\vspace{0.4in}
			\small{\rotatebox{90}{green}}\\
			\vspace{0.3in}
			\small{\rotatebox{90}{blue}}\\
	\end{minipage}
	\caption{\label{fig:main_color}\small{Samples by color guidance with red, green, and blue, varying the guidance scale $s$ (under a fixed random seed).}
}
\vspace{-.2in}
\end{figure}

\subsection{Results in Class-Conditional Image Synthesis}
\label{classifier_results}

We quantitatively evaluate our proposed method (\eqref{Eq:contarstive_condition_loss}) in image synthesis tasks on ImageNet as is shown in \cref{tab:guide2}. Our method achieves results that are roughly on par with classic classifier guidance \citep{diffusion_beat_gan}. We also qualitatively compare sampled images guided by our proposed CEP guidance and classifier guidance with fixed random seeds in \cref{fig:image_guidance_qualitative}, which shows that they generate samples that are almost visually identical. Besides the training objective, our method uses exactly the same network architecture, training pipeline, and evaluation methods as \citet{diffusion_beat_gan}, without any kind of hyperparameter tuning. These results empirically indicate that our method is an almost equally well-performing alternative to classifier guidance in ImageNet image synthesis tasks.

\subsection{Energy-Guided Image Synthesis}
This section showcases how image synthesis can be controlled through a continuous energy function as described in \eqref{Eq:target_distribution_intro} instead of a discrete class condition as in \cref{classifier_results}. We define the energy function at $t=0$ data space to indicate the overall color appearance of an image:
\begin{equation}
    \label{Eq:color_q0}
    \gE(\x) := - \overline{\| h(\x) - h_\textbf{tar} \|_1},
\end{equation}
where $h(\x)$ represents the hue value for each pixel in an image $\x$, and can be calculated via Hue-Saturation-Intensity (HSI) decomposition \citep{cvbook}.  $h(\x)$ is defined in an angular space of range $[0, 2\pi]$, where red is at angle $0$, green at $2\pi/3$, blue at $4\pi/3$, and red again at $2\pi$. As a result, by setting the target hue $h_\textbf{tar}$ to corresponding angular values, we can evaluate how an image is visually close to a ``pure color'' image.

With such a definition of the energy function, we train three energy guidance models to control the overall color appearance of sampled images. An illustration is given in \cref{fig:main_color}. By switching among different guidance models and tuning guidance scales to control the guidance effect similar to \citet{diffusion_beat_gan, ho2021classifier}, we can effectively control the color appearance of an image. If the diffusion prior is a conditional model, generated images might have different backgrounds (e.g., desert, forest, and sky) to meet different preferences of color appearance while ensuring fidelity. See more examples in \cref{Sec:Scale_samples}.

\section{Conclusion}
In this work, we formally consider the problem of sampling by diffusion models pretrained on a data distribution but the target sampling distribution is edited by an energy function. We show that this can be achieved by adding additional energy guidance to the original sampling procedure. We further propose a novel training objective named contrastive energy prediction for training an energy model to estimate such guidance. Our proposed CEP guidance is exact compared with previous energy guidance methods in the sense that it can guarantee convergence to the desired distribution. We apply our proposed method to several downstream tasks in order to demonstrate its effectiveness and scalability. Experimental results show that our method outperforms existing guidance methods in offline RL and is roughly on par with the classic classifier guidance in conditional image synthesis.

\section*{Acknowledgements}
This work was supported by the National Key Research and Development Program of China
(2020AAA0106302); 
NSF of China Projects (Nos. 62061136001, 61620106010, 62076145, U19B2034, U1811461, U19A2081, 6197222, 62106120, 62076145); a grant from Tsinghua Institute for Guo Qiang; the High Performance Computing Center, Tsinghua University. J.Z was also supported by the New Cornerstone Science Foundation through the XPLORER PRIZE. C. Li was also sponsored by the Beijing Nova Program.

\nocite{langley00}

\bibliography{ref}
\bibliographystyle{icml2023}

\newpage
\appendix
\onecolumn

\section{Limitations and broader impacts}
Similar to other deep generative modeling methods, energy-guided diffusion sampling can be potentially used to generate harmful contents such as ``deepfakes'', and might reflect and amplify unwanted social bias existed in the training dataset.

\section{Related Work}
\label{appendix:related}

\paragraph{Diffusion Models and Applications.}
Diffusion models (also as known as score-based generative models)~\citep{sohl2015deep,ho2020denoising,song2020score,karras2022elucidating} are emerging powerful generative models and have achieved impressive success on various tasks, such as voice synthesis~\citep{liu2022diffsinger,chen2020wavegrad,chen2021wavegrad}, high-resolution image synthesis~\citep{diffusion_beat_gan,ho2022cascaded}, image editing~\citep{meng2021sdedit,saharia2022palette,zhao2022egsde}, text-to-image generation~\citep{saharia2022photorealistic,nichol2021glide,ramesh2022hierarchical,rombach2022high,gu2022vector}, molecule generation~\citep{xu2022geodiff,hoogeboom2022equivariant,wu2022diffusion}, 3-D shape generation~\citep{zeng2022lion,poole2022dreamfusion,wang2022rodin}, video generation~\citep{ho2022video,ho2022imagen,yang2022diffusion,zhou2022magicvideo} and data compression~\citep{theis2022lossy, kingma2021variational,lu2022maximum}. The sampling methods for diffusion models include training-free fast samplers~\citep{song2020denoising,bao2022analytic,lu2022dpm,lu2022dpm++,zhang2022fast,zhang2022gddim} and distillation-based samplers~\citep{salimans2022progressive,meng2022distillation}.

\paragraph{Diffusion Models as Priors.}
There exist many existing works that use a pretrained diffusion model as the prior distribution $q(\x)$ and aim to sample from \eqref{Eq:target_distribution_intro}. \citet{graikos2022diffusion} propose a training-free sampling method which directly use the constraint ($\gE(\cdot)$ in \eqref{Eq:target_distribution_intro}) and can be used for approximately solving traveling salesman problems; \citet{poole2022dreamfusion} use a pretrained 2-D diffusion model and optimizing the 3-D parameters for 3-D shape generation; \citet{kawar2022denoising,chung2022diffusion} use pretrained diffusion models to solve linear and some special non-linear inverse problems, such as image inpainting, deblurring and denoising; \citet{zhao2022egsde,bao2022equivariant} use human-designed intermediate energy guidance for image-to-image translation and inverse molecular design. However, all these existing methods cannot guarantee that the final samples follow the desired $p(\x)$ in \eqref{Eq:target_distribution_intro}.

\paragraph{Controllable Generation.}
To embed human preference and controllability into the sampling procedure of deep generative models, many recent work~\citep{nie2021controllable,li2022diffusion,pandey2022diffusevae} manipulate the sampling procedure of the latent space by some learned conditional models, and use the obtained latent code to generate desired samples. The generator include variational auto-encoder~\citep{kingma2013auto} and generative adversarial networks~\citep{goodfellow2020generative}. However, existing methods for realizing controllable generation in diffusion models mainly focus on conditional guidance. The problem we considered in \eqref{Eq:target_distribution_intro} can be alternatively understood as a general formulation for embed human controllability into the sampling procedure of diffusion models.

\paragraph{Offline Reinforcement Learning.}
Offline RL typically requires reconciling two conflicting aims: Staying close to the behavior policy while maximizing the expected Q-values. In order to stick with a potentially diverse behavior policy, recent studies \cite{diffuser, arq, dql, sfbc, dd, csbc} have found diffusion models to be a powerful generative tool, which tends to outperform previous generative methods such as Gaussians \citep{awr, crr, awac} or VAEs \cite{bcq, bear} in terms of behavior modeling. In terms of how to generate actions that maximize the learned Q-functions, different methods take different approaches. Diffusers \cite{diffuser} intends to mimic the classifier-guidance method \cite{diffusion_beat_gan} and propose to use a guidance method as described in \eqref{Eq:MSE_qt_objective}, but without a detailed discussion on the convergence property of the proposed method. Decision Diffuser (DD, \citet{dd}), on the other hand, explores using classifier-free guidance. Diffusion-QL tracks the gradients of the actions sampled from the behavior diffusion policy to guide generated actions to high Q-value area. SfBC \citep{sfbc} and Diffusion-QL share the same trick by simply resampling actions from multiple behavior action candidates using the predicted Q-value as sampling weights. Other work \cite{arq, csbc} only uses diffusion models for pure behavior cloning, so no Q-value maximizing is required. 
In contrast with prior work, our work aims to study how an energy guidance model could be exactly trained and used to guide the sampling process in diffusion-based decision-making.

\section{Motivation for \eqref{Eq:target_distribution_intro}}
The formulation in \eqref{Eq:target_distribution_intro} is general and common across various settings, such as the product of experts by Gibbs distribution~\citep{hinton2002training}, posterior-regularized Bayesian inference~\citep{zhu2014bayesian}, exponential tilting of generative models~\citep{xiao2020exponential}, training deep generative models on limited data with regularization by pre-trained models~\citep{zhong2022deep}, constrained policy optimization in reinforcement learning~\citep{awr,ziegler2019fine}, and inverse problems~\citep{chung2022diffusion}. In this section, we give a motivating example for such an objective.

Let $q(\x)$ be an unknown data distribution, and $\gE(\x)$ be a loss function that we want to minimize. We want to optimize a generative model $p(\x)$ such that the samples from $p(\x)$ have as small loss $\gE(\x)$ as possible; meanwhile, we also use the data distribution $q(\x)$ to regularize the model $p(\x)$ to increase the diversity and avoid collapse solutions. The objective can be formulated by:
\begin{equation}
\min_{p} \E_{p(\x)} [\gE(\x)]
     + \frac{1}{\beta} \kl{p(\x)}{q(\x)},
\end{equation}
where $\kl{p(\x)}{q(\x)}$ is a regularization term and $\beta$ is a hyperparameter. By simply computing the derivation for $p$ and letting it be zero, we can obtain the optimal $p^*$ satisfies
\begin{equation}
    p^*(\x) \propto q(\x)e^{-\gE(\x)},
\end{equation}
which has the exactly same form as \eqref{Eq:target_distribution_intro}.

\section{Proofs and Additional Theory}
\label{appendix:proof}

\subsection{CEP with Condition Variables}
Assume the energy function is $\gE(\x_0,c)$ with an additional conditioning variable $c$, which follows a distribution $q(c)$. We aim to learn the intermediate energy guidance by a neural network $f_\phi(\cdot,c,t):\R^d\rightarrow \R$ parameterized by $\phi$.

\paragraph{CEP with unconditional prior.}
Assume the prior distribution $q_0(\x)$ is unconditional. We aim to sample from 
\begin{equation}
    p_0(\x_0|c) \propto q_0(\x_0)e^{-\beta\gE(\x_0,c)},
\end{equation}
for each given $c$.
By taking the expectation of $q(c)$, the objective in \eqref{Eq:infoNCE_energy} is
\begin{equation}
\min_\phi \E_{q(c)}\E_{p(t)}\E_{q_0(\x_0^{(1:K)})}\E_{p(\epsilonv^{(1:K)})}\Bigg[
    - \sum_{i=1}^K
     e^{-\beta\gE(\x_0^{(i)}, c)}
    \log \frac{e^{-f_\phi(\x_t^{(i)},c,t)}}{\sum_{j=1}^K e^{-f_\phi(\x_t^{(j)},c,t)}}
\Bigg].
\end{equation}

\paragraph{CEP with conditional prior.}
Assume the prior distribution $q_0(\x|c)$ is conditional on $c$. We aim to sample from 
\begin{equation}
    p_0(\x_0|c) \propto q_0(\x_0|c)e^{-\beta\gE(\x_0,c)},
\end{equation}
for each given $c$.
By taking the expectation of $q(c)$, the objective in \eqref{Eq:infoNCE_energy} is
\begin{equation}
\min_\phi \E_{q(c)}\E_{p(t)}\E_{q_0(\x_0^{(1:K)}|c)}\E_{p(\epsilonv^{(1:K)})}\Bigg[
    - \sum_{i=1}^K
     e^{-\beta\gE(\x_0^{(i)}, c)}
    \log \frac{e^{-f_\phi(\x_t^{(i)},c,t)}}{\sum_{j=1}^K e^{-f_\phi(\x_t^{(j)},c,t)}}
\Bigg].
\end{equation}
Moreover, we can draw $K$ sample pairs $(\x_0^{(i)},c^{(i)})$ for $i=1,\dots,K$, and the above objective becomes
\begin{equation}
\min_\phi \E_{p(t)}\E_{\prod_{i=1}^K q_0(\x_0^{(i)},c^{(i)})}\E_{p(\epsilonv^{(1:K)})}\Bigg[
    - \sum_{i=1}^K
     e^{-\beta\gE(\x_0^{(i)}, c^{(i)})}
    \log \frac{e^{-f_\phi(\x_t^{(i)},c^{(i)},t)}}{\sum_{j=1}^K e^{-f_\phi(\x_t^{(j)},c^{(i)},t)}}
\Bigg].
\end{equation}
Note that the terms in the numerator is $(\x_t^{(i)}, c^{(i)})\sim q_t(\x_t^{(i)}, c^{(i)})$ are samples from the joint distribution; while the terms in the denominator is $(\x_t^{(j)}, c^{(i)}, t) \sim q_t(\x_t^{(j)})q(c^{(i)})$ are independent samples. Such formulation is highly similar to the contrastive learning objective~\citep{infoNCE}, and we discuss the connections in Appendix~\ref{appendix:relationship_infonce}. In summary, CEP with conditional prior can be considered as a generalized version of contrastive learning with soft energy labels.

\subsection{CEP in Multiple Time Steps}
Training the guidance model by \eqref{Eq:infoNCE_energy} needs $K$ samples of $\x_t$ for each time $t$. If we use $M$ samples of $t\in(0,T]$, the number of total samples of $\x_t$ is $KM$. However, for high-dimensional data such as images, the memory budget is limited and we want to use as many samples at each time $t$ as possible. In this section, we propose an alternative objective that can leverage $K$ samples $x_t$ from different time $t$. Thus, we only need $K$ samples of $t$ and $K$ samples of $\x_t$ to reduce the memory cost. We formally propose the objective below and provide the proof in Appendix~\ref{appendix:proof_mixed}.

\begin{theorem}[CEP in Multiple Time Steps]
\label{thrm:infoNCE_energy_mixed}
Let $t^{(1)},\dots,t^{(K)}$ be K i.i.d. samples from $p(t)$. For each $i=1,\dots,K$, let $\x_t^{(i)}\coloneqq \alpha_{t^{(i)}}\x_0^{(i)} + \sigma_{t^{(i)}}\epsilonv^{(i)}$, where $\alpha_t,\sigma_t$ are defined in \eqref{Eq:forward_p}. Define an objective:
\begin{equation}
\label{Eq:infoNCE_energy_mixed}
\begin{split}
    \min_\phi \E_{p(t^{(1:K)})}\E_{q_0(\x_0^{(1:K)})}\E_{p(\epsilonv^{(1:K)})}\Bigg[
    - \sum_{i=1}^K
    \underbrace{%
        \vphantom{\frac{\exp(-f_\phi(\x_t^{(i)},t))}{\sum_{j=1}^K \exp(-f_\phi(\x_t^{(j)},t))}} 
        \left( e^{-\beta\gE(\x_0^{(i)})} \right)
    }_{\text{energy label}}
    \log \underbrace{\frac{e^{-f_\phi(\x_t^{(i)},t^{(i)})}}{\sum_{j=1}^K e^{-f_\phi(\x_t^{(j)},t^{(j)})}}}_{\text{predicted label}}
\Bigg].
\end{split}
\end{equation}
Given unlimited model capacity and data samples, For all $K>1$ and $t\in[0,T]$, the optimal $f_{\phi^*}$ in \eqref{Eq:infoNCE_energy_mixed} satisfies
\begin{equation}
    \nabla_{\x_t}f_{\phi^*}(\x_t,t) = \nabla_{\x_t}\gE_t(\x_t).
\end{equation}
\end{theorem}

Below we also give the corresponding objectives for energy functions with conditioning variables.

\paragraph{CEP with unconditional prior in multiple time steps.}
The objective is
\begin{equation}
\min_\phi \E_{q(c)}\E_{p(t^{(1:K)})}\E_{q_0(\x_0^{(1:K)})}\E_{p(\epsilonv^{(1:K)})}\Bigg[
    - \sum_{i=1}^K
     e^{-\beta\gE(\x_0^{(i)}, c)}
    \log \frac{e^{-f_\phi(\x_t^{(i)},c,t^{(i)})}}{\sum_{j=1}^K e^{-f_\phi(\x_t^{(j)},c,t^{(j)})}}
\Bigg].
\end{equation}

\paragraph{CEP with conditional prior in multiple time steps.}
\begin{equation}
\min_\phi \E_{p(t^{(1:K)})}\E_{\prod_{i=1}^K q_0(\x_0^{(i)},c^{(i)})}\E_{p(\epsilonv^{(1:K)})}\Bigg[
    - \sum_{i=1}^K
     e^{-\beta\gE(\x_0^{(i)}, c^{(i)})}
    \log \frac{e^{-f_\phi(\x_t^{(i)},c^{(i)},t^{(i)})}}{\sum_{j=1}^K e^{-f_\phi(\x_t^{(j)},c^{(i)},t^{(j)})}}
\Bigg].
\end{equation}
Also note that the terms in the numerator is $(\x_t^{(i)}, c^{(i)}, t^{(i)})\sim q(\x_t^{(i)}, c^{(i)}, t^{(i)})$ are samples from the joint distribution; while the terms in the denominator is $(\x_t^{(j)}, c^{(i)}, t^{(j)}) \sim q(\x_t^{(j)}, t^{(j)})q(c^{(i)})$ are independent samples.

\subsection{Proof of \cref{thrm:intermediate_energy_guidance}}
\begin{proof}
Assume the normalizing constant for $p_0(\x_0)$ is
\begin{equation*}
    Z = \int q_0(\x_0) e^{-\beta\gE(\x_0)}\dm\x_0 = \E_{q_0(\x_0)}\left[e^{-\beta\gE(\x_0)}\right],
\end{equation*}
then we have
\begin{equation*}
    p_0(\x_0) = \frac{q_0(\x_0)e^{-\beta\gE(\x_0)}}{Z}
\end{equation*}
According to the definition, we have
\begin{equation*}
\begin{split}
p_t(\x_t) &= \int p_{t0}(\x_t|\x_0)p_0(\x_0)\dm\x_0 \\
    &= \int p_{t0}(\x_t|\x_0)q_0(\x_0)\frac{e^{-\beta\gE(\x_0)}}{Z}\dm\x_0 \\
    &= \int q_{t0}(\x_t|\x_0)q_0(\x_0)\frac{e^{-\beta\gE(\x_0)}}{Z}\dm\x_0 \\
    &= q_t(\x_t)\int q_{0}(\x_0|\x_t)\frac{e^{-\beta\gE(\x_0)}}{Z}\dm\x_0 \\
    &= \frac{q_t(\x_t)\E_{q_{0}(\x_0|\x_t)}\left[e^{-\beta\gE(\x_0)}\right]}{Z} \\
    &= \frac{q_t(\x_t)e^{-\gE_t(\x_t)}}{Z}
\end{split}
\end{equation*}
and then
\begin{equation*}
    \nabla_{\x_t}\log p_t(\x_t) = \nabla_{\x_t}\log q_t(\x_t) - \nabla_{\x_t}\gE_t(\x_t) 
\end{equation*}

\end{proof}

\subsection{Proof of \cref{thrm:infoNCE_energy}}
\begin{proof}
As $q_{t0}(\x_t|\x_0)=\Nc(\x_t|\alpha_t\x_0,\sigma_t^2\Iv)$, we can rewrite \eqref{Eq:infoNCE_energy} by
\begin{equation*}
    \min_\phi \E_{p(t)}\E_{q_0(\x_0^{(1:K)})}\E_{q_{t0}(\x_t^{(1:K)}|\x_0^{(1:K)})}\Bigg[
    - \sum_{i=1}^K
     e^{-\beta\gE(\x_0^{(i)})}
    \log \frac{e^{-f_\phi(\x_t^{(i)},t)}}{\sum_{j=1}^K e^{-f_\phi(\x_t^{(j)},t)}}
\Bigg].
\end{equation*}
Rewriting $q_{0t}(\x_0,\x_t) = q_t(\x_t)q_{0t}(\x_0|\x_t)$ and moving the conditional expectation $q_{0t}(\x_0|\x_t)$ into the inner part, we have
\begin{equation*}
    \min_\phi \E_{p(t)}\E_{q_t(\x_t^{(1:K)})}\Bigg[
    - \sum_{i=1}^K
    \E_{q_{0t}(\x_0^{(i)}|\x_t^{(i)})}\left[ e^{-\beta\gE(\x_0^{(i)})} \right]
    \log \frac{e^{-f_\phi(\x_t^{(i)},t)}}{\sum_{j=1}^K e^{-f_\phi(\x_t^{(j)},t)}}
\Bigg].    
\end{equation*}
By \eqref{Eq:intermediate_energy}, we have $\E_{q_{0t}(\x_0|\x_t)}\left[ e^{-\beta\gE(\x_0)} \right] = e^{-\gE_t(\x_t)}$, thus the above objective is equivalent to
\begin{equation}
\label{Eq:appendix_proof_objective}
    \min_\phi \E_{p(t)}\E_{q_t(\x_t^{(1:K)})}\Bigg[
    - \sum_{i=1}^K
    e^{-\gE_t(\x_t^{(i)})}
    \log \frac{e^{-f_\phi(\x_t^{(i)},t)}}{\sum_{j=1}^K e^{-f_\phi(\x_t^{(j)},t)}}
\Bigg].
\end{equation}
For each $t$ and $\x_t^{(1:K)}$, for $i=1,\dots,K$, define
\begin{equation*}
\begin{split}
    a_i(\x_t^{(1:K)}, t) &\coloneqq  \frac{e^{-\gE_t(\x_t^{(i)})}}{\sum_{j=1}^K e^{-\gE_t(\x_t^{(j)})}}, \\
    b_i(\x_t^{(1:K)}, t) &\coloneqq \frac{e^{-f_\phi(\x_t^{(i)},t)}}{\sum_{j=1}^K e^{-f_\phi(\x_t^{(j)},t)}}, \\
    c(\x_t^{(1:K)}, t) &\coloneqq \sum_{j=1}^K e^{-\gE_t(\x_t^{(j)})} > 0,
\end{split}
\end{equation*}
Then \eqref{Eq:appendix_proof_objective} is equivalent to
\begin{equation*}
    \min_\phi \E_{p(t)}\E_{q_t(\x_t^{(1:K)})}\Bigg[
    -c(\x_t^{(1:K)}, t) \sum_{i=1}^K
    a_i(\x_t^{(1:K)},t)
    \log \left(b_i(\x_t^{(1:K)}, t)\right)
\Bigg].
\end{equation*}
For each fixed $t$ and $\x_t^{(!:K)}$, as $\sum_{i=1}^K a_i(\x_t^{(1:K)},t) = \sum_{i=1}^K b_i(\x_t^{(1:K)},t) = 1$, according to Gibbs' inequality, we have
\begin{equation*}
    -\sum_{i=1}^K a_i(\x_t^{(1:K)},t) \log \left(b_i(\x_t^{(1:K)}, t)\right)
    \geq -\sum_{i=1}^K a_i(\x_t^{(1:K)},t) \log \left(a_i(\x_t^{(1:K)}, t)\right),
\end{equation*}
so we have
\begin{equation*}
\begin{split}
&\phantom{{}={}}\E_{p(t)}\E_{q_t(\x_t^{(1:K)})}\Bigg[
    -c(\x_t^{(1:K)}, t) \sum_{i=1}^K
    a_i(\x_t^{(1:K)},t)
    \log \left(b_i(\x_t^{(1:K)}, t)\right)
\Bigg] \\
&\geq 
\E_{p(t)}\E_{q_t(\x_t^{(1:K)})}\Bigg[
    -c(\x_t^{(1:K)}, t) \sum_{i=1}^K
    a_i(\x_t^{(1:K)},t)
    \log \left(a_i(\x_t^{(1:K)}, t)\right)
\Bigg],
\end{split}
\end{equation*}
and the equality holds when for each $t\in(0,T]$ and each $\x_t^{(1:K)}$ in the supported space of $q_t(\x_t^{(1:K)})$,
\begin{equation*}
    b_i(\x_t^{(1:K)}, t) = a_i(\x_t^{(1:K)}, t), \quad i=1,\dots,K,
\end{equation*}
so given unlimited data and model capacity, the optimal $\phi^*$ satisfies
\begin{equation*}
    \frac{e^{-f_{\phi^*}(\x_t^{(i)},t)}}{\sum_{j=1}^K e^{-f_{\phi^*}(\x_t^{(j)},t)}} = \frac{e^{-\gE_t(\x_t^{(i)})}}{\sum_{j=1}^K e^{-\gE_t(\x_t^{(j)})}},
\end{equation*}
so for each $i=1,\dots,K$,
\begin{equation*}
    \frac{e^{-f_{\phi^*}(\x_t^{(i)},t)}}{e^{-\gE_t(\x_t^{(i)})}} = \frac{\sum_{j=1}^K e^{-f_{\phi^*}(\x_t^{(j)},t)}}{\sum_{j=1}^K e^{-\gE_t(\x_t^{(j)})}},
\end{equation*}
due to the arbitrariness of $\x_t^{(1:K)}$ and $t$, for any $\x_t^{(i)}$,$\x_t^{(j)}$, we have
\begin{equation*}
    \frac{e^{-f_{\phi^*}(\x_t^{(i)},t)}}{e^{-\gE_t(\x_t^{(i)})}} = \frac{e^{-f_{\phi^*}(\x_t^{(j)},t)}}{e^{-\gE_t(\x_t^{(j)})}}.
\end{equation*}
Therefore, there exists a constant $C_t$ independent of $\x_t$, such that
\begin{equation*}
    e^{-f_{\phi^*}(\x_t,t)} = C_t \cdot e^{-\gE_t(\x_t)} \propto e^{-\gE_t(\x_t)},
\end{equation*}
and then $\nabla_{\x_t}f_{\phi^*}(\x_t,t) = \nabla_{\x_t}\gE_t(\x_t)$.
\end{proof}

\subsection{Proof of \cref{thrm:infoNCE_energy_mixed}}
\label{appendix:proof_mixed}
Intuitively, \cref{thrm:infoNCE_energy_mixed} can be similarly proved as \cref{thrm:infoNCE_energy} by considering $(\x_t,t)$ as a whole random variable. We formally give the proof below.

\begin{proof}[Proof of \cref{thrm:infoNCE_energy_mixed}]
As $q_{t0}(\x_t|\x_0)=\Nc(\x_t|\alpha_t\x_0,\sigma_t^2\Iv)$, we can rewrite \eqref{Eq:infoNCE_energy_mixed} by
\begin{equation*}
    \min_\phi \E_{q(t^{(1:K)})}\E_{q(\x_0^{(1:K)},\x_t^{(1:K)})}\Bigg[
    - \sum_{i=1}^K
     e^{-\beta\gE(\x_0^{(i)})}
    \log \frac{e^{-f_\phi(\x_t^{(i)},t^{(i)})}}{\sum_{j=1}^K e^{-f_\phi(\x_t^{(j)},t^{(j)})}}
\Bigg].
\end{equation*}
Rewriting $q(\x_0,\x_t) = q_t(\x_t)q_{0t}(\x_0|\x_t)$ and moving the conditional expectation $q_{0t}(\x_0|\x_t)$ into the inner part, we have
\begin{equation*}
    \min_\phi \E_{p(t^{(1:K)})}\E_{q_t(\x_t^{(1:K)})}\Bigg[
    - \sum_{i=1}^K
    \E_{q_{0t^{(i)}}(\x_0^{(i)}|\x_t^{(i)})}\left[ e^{-\beta\gE(\x_0^{(i)})} \right]
    \log \frac{e^{-f_\phi(\x_t^{(i)},t^{(i)})}}{\sum_{j=1}^K e^{-f_\phi(\x_t^{(j)},t^{(j)})}}
\Bigg].    
\end{equation*}
By \eqref{Eq:intermediate_energy}, we have $\E_{q_{0t}(\x_0|\x_t)}\left[ e^{-\beta\gE(\x_0)} \right] = e^{-\gE_t(\x_t)}$, thus the above objective is equivalent to
\begin{equation}
\label{Eq:appendix_proof_objective_mixed}
    \min_\phi \E_{p(t^{(1:K)})}\E_{q_t(\x_t^{(1:K)})}\Bigg[
    - \sum_{i=1}^K
    e^{-\gE_{t^{(i)}}(\x_t^{(i)})}
    \log \frac{e^{-f_\phi(\x_t^{(i)},t^{(i)})}}{\sum_{j=1}^K e^{-f_\phi(\x_t^{(j)},t^{(j)})}}
\Bigg].
\end{equation}
For each $t^{(1:K)}$ and $\x_t^{(1:K)}$, for $i=1,\dots,K$, define
\begin{equation*}
\begin{split}
    a_i(\x_t^{(1:K)}, t^{(1:K)}) &\coloneqq  \frac{e^{-\gE_{t^{(i)}}(\x_t^{(i)})}}{\sum_{j=1}^K e^{-\gE_{t^{(j)}}(\x_t^{(j)})}}, \\
    b_i(\x_t^{(1:K)}, t^{(1:K)}) &\coloneqq \frac{e^{-f_\phi(\x_t^{(i)},t^{(i)})}}{\sum_{j=1}^K e^{-f_\phi(\x_t^{(j)},t^{(j)})}}, \\
    c(\x_t^{(1:K)}, t^{(1:K)}) &\coloneqq \sum_{j=1}^K e^{-\gE_{t^{(j)}}(\x_t^{(j)})} > 0,
\end{split}
\end{equation*}
Then \eqref{Eq:appendix_proof_objective_mixed} is equivalent to
\begin{equation*}
    \min_\phi \E_{p(t^{(1:K)})}\E_{q_t(\x_t^{(1:K)})}\Bigg[
    -c(\x_t^{(1:K)}, t^{(1:K)}) \sum_{i=1}^K
    a_i(\x_t^{(1:K)},t^{(1:K)})
    \log \left(b_i(\x_t^{(1:K)}, t^{(1:K)})\right)
\Bigg].
\end{equation*}
For each fixed $t^{(1:K)}$ and $\x_t^{(!:K)}$, as $\sum_{i=1}^K a_i(\x_t^{(1:K)},t^{(1:K)}) = \sum_{i=1}^K b_i(\x_t^{(1:K)},t^{(1:K)}) = 1$, according to Gibbs' inequality, we have
\begin{equation*}
    -\sum_{i=1}^K a_i(\x_t^{(1:K)},t^{(1:K)}) \log \left(b_i(\x_t^{(1:K)}, t^{(1:K)})\right)
    \geq -\sum_{i=1}^K a_i(\x_t^{(1:K)},t^{(1:K)}) \log \left(a_i(\x_t^{(1:K)}, t^{(1:K)})\right),
\end{equation*}
so we have
\begin{equation*}
\begin{split}
&\phantom{{}={}}\E_{p(t^{(1:K)})}\E_{q_t(\x_t^{(1:K)})}\Bigg[
    -c(\x_t^{(1:K)}, t^{(1:K)}) \sum_{i=1}^K
    a_i(\x_t^{(1:K)},t^{(1:K)})
    \log \left(b_i(\x_t^{(1:K)}, t^{(1:K)})\right)
\Bigg] \\
&\geq 
\E_{p(t^{(1:K)})}\E_{q_t(\x_t^{(1:K)})}\Bigg[
    -c(\x_t^{(1:K)}, t^{(1:K)}) \sum_{i=1}^K
    a_i(\x_t^{(1:K)},t^{(1:K)})
    \log \left(a_i(\x_t^{(1:K)}, t^{(1:K)})\right)
\Bigg],
\end{split}
\end{equation*}
and the equality holds when for each $t^{(1:K)}$ and each $\x_t^{(1:K)}$ in the supported space of $q_t(\x_t^{(1:K)})$,
\begin{equation*}
    b_i(\x_t^{(1:K)}, t^{(1:K)}) = a_i(\x_t^{(1:K)}, t^{(1:K)}), \quad i=1,\dots,K,
\end{equation*}
so given unlimited data and model capacity, the optimal $\phi^*$ satisfies
\begin{equation*}
    \frac{e^{-f_{\phi^*}(\x_t^{(i)},t^{(i)})}}{\sum_{j=1}^K e^{-f_{\phi^*}(\x_t^{(j)},t^{(j)})}} = \frac{e^{-\gE_{t^{(i)}}(\x_t^{(i)})}}{\sum_{j=1}^K e^{-\gE_{t^{(j)}}(\x_t^{(j)})}},
\end{equation*}
so for each $i=1,\dots,K$,
\begin{equation*}
    \frac{e^{-f_{\phi^*}(\x_t^{(i)},t^{(i)})}}{e^{-\gE_t(\x_t^{(i)})}} = \frac{\sum_{j=1}^K e^{-f_{\phi^*}(\x_t^{(j)},t^{(j)})}}{\sum_{j=1}^K e^{-\gE_{t^{(j)}}(\x_t^{(j)})}},
\end{equation*}
due to the arbitrariness of $\x_t^{(1:K)}$ and $t^{(1:K)}$, for any $\x_t^{(i)}$,$\x_t^{(j)}$, we have
\begin{equation*}
    \frac{e^{-f_{\phi^*}(\x_t^{(i)},t^{(i)})}}{e^{-\gE_{t^{(i)}}(\x_t^{(i)})}} = \frac{e^{-f_{\phi^*}(\x_t^{(j)},t^{(j)})}}{e^{-\gE_{t^{(j)}}(\x_t^{(j)})}}.
\end{equation*}
Therefore, there exists a constant $C_t$ independent of $\x_t$, such that
\begin{equation*}
    e^{-f_{\phi^*}(\x_t,t)} = C_t \cdot e^{-\gE_t(\x_t)} \propto e^{-\gE_t(\x_t)},
\end{equation*}
and then $\nabla_{\x_t}f_{\phi^*}(\x_t,t) = \nabla_{\x_t}\gE_t(\x_t)$.
\end{proof}

\section{Comparison with Existing Energy-Guided Sampling Algorithms}
\label{appendix:relationship}

\begin{table}[t]
    \vspace{-.1in}
    \centering
    \caption{\small{Comparison between energy-guided sampling algorithms.}}
    \vskip 0.1in
    \begin{small}
    \resizebox{0.99\textwidth}{!}{%
    \begin{tabular}{lllc}
    \toprule
    Method & Optimal Solution of Energy & Optimal Solution of Guidance & 
    Exact Guidance \\
    \midrule
    \vspace{0.03in}
    CEP (ours) & $-\log\E_{q_{0t}(\x_0|\x_t)}\left[e^{-\gE_0(\x_0)}\right]$ & $\E_{q_{0t}(\x_0|\x_t)}\left[-e^{\gE_t(\x_t)-\gE_0(\x_0)}\nabla_{\x_t}\log q_{0t}(\x_0|\x_t)\right]$ & \color{green}{\ding{51}} \\
    \vspace{0.03in}
    MSE & $\E_{q_{0t}(\x_0|\x_t)}[\gE_0(\x_0)]$ & $\E_{q_{0t}(\x_0|\x_t)}\Big[\gE_0(\x_0)\nabla_{\x_t}\log q_{0t}(\x_0|\x_t)\Big]$ & \color{red}{\ding{55}} \\
    DPS & $\gE_0\left(\E_{q_{0t}(\x_0|\x_t)}[\x_0]\right)$ & $\E_{q_{0t}(\x_0|\x_t)}\Big[\left(\left(\nabla\gE_0\left(\E_{q_{0t}(\x_0|\x_t)}[\x_0]\right)\right)^\top \x_0\right) \nabla_{\x_t}\log q_{0t}(\x_0|\x_t)\Big]$ & \color{red}{\ding{55}} \\
    \bottomrule
    \end{tabular}%
    }
    \end{small}
    \label{tab:comparison_appendix}
\end{table}

Firstly, we can easily compute the gradients for the energy guidance used in MSE~\citep{diffuser,bao2022equivariant} and DPS~\citep{ho2022video,chung2022diffusion}, and we summarize the results in \cref{tab:comparison_appendix}. Below we propose a deeper connection between these methods.

\paragraph{Comparison for Energy.}
The exact energy function is
\begin{equation*}
    \gE_t(\x_t) = -\log\E_{q_{0t}(\x_0|\x_t)}\left[e^{-\gE_0(\x_0)}\right].
\end{equation*}
By exchanging the order of the $\log$ function and the expectation, we can derive the energy used in MSE:
\begin{equation*}
    \gE_t^{\text{MSE}}(\x_t) = -\E_{q_{0t}(\x_0|\x_t)}\left[\log\left(e^{-\gE_0(\x_0)}\right)\right] = \E_{q_{0t}(\x_0|\x_t)}\left[\gE_0(\x_0)\right].
\end{equation*}
By further exchanging the order of $\gE_0$ function and the expectation, we can derive the energy used in DPS:
\begin{equation*}
    \gE_t^{\text{DPS}}(\x_t) = \gE_0\left(\E_{q_{0t}(\x_0|\x_t)}\left[\x_0\right]\right).
\end{equation*}
Intuitively, exchanging the order between a nonlinear function and an expectation will introduce additional approximation errors, and the errors depend on the complexity of the nonlinear function. As $\log(\cdot)$ is a simple concave function but $\gE_0(\cdot)$ may be rather complex, the approximation error of DPS may be quite large, which may explain why the sample results of DPS are quite worse than CEP in Fig.~\ref{fig:toy_main} and Fig.~\ref{fig:toy_appendix}.

\paragraph{Comparison for Guidance.}
The exact guidance is
\begin{equation}
\label{Eqn:appendix_exact_guidance}
    \nabla_{\x_t}\gE_t(\x_t) = \E_{q_{0t}(\x_0|\x_t)}\left[-e^{\gE_t(\x_t)-\gE_0(\x_0)}\nabla_{\x_t}\log q_{0t}(\x_0|\x_t)\right].
\end{equation}
And the guidance by MSE is
\begin{equation*}
    \nabla_{\x_t}\gE_t^{\text{MSE}}(\x_t) = \E_{q_{0t}(\x_0|\x_t)}\left[
        \gE_0(\x_0)\nabla_{\x_t}\log q_{0t}(\x_0|\x_t)
    \right].
\end{equation*}
And the guidance by DPS is
\begin{equation*}
    \nabla_{\x_t}\gE_t^{\text{DPS}}(\x_t) = \E_{q_{0t}(\x_0|\x_t)}\Big[\left(\nabla\gE_0\left(\left(\E_{q_{0t}(\x_0|\x_t)}[\x_0]\right)\right)^\top \x_0\right) \nabla_{\x_t}\log q_{0t}(\x_0|\x_t)\Big].
\end{equation*}
Below we discuss the relationship between these three functions.

By taking Taylor expansion for $e^{\gE_t(\x_t)-\gE_0(\x_0)}\approx 1 + \gE_t(\x_t)-\gE_0(\x_0)$, we have
\begin{equation*}
\begin{aligned}
    \nabla_{\x_t}\gE_t(\x_t) &\approx \E_{q_{0t}(\x_0|\x_t)}\left[
        (-1 - \gE_t(\x_t) + \gE_0(\x_0))\nabla_{\x_t}\log q_{0t}(\x_0|\x_t)
    \right] \\
    &= \E_{q_{0t}(\x_0|\x_t)}\left[
        \gE_0(\x_0)\nabla_{\x_t}\log q_{0t}(\x_0|\x_t)
    \right]\\
    &= \nabla_{\x_t}\gE_t^{\text{MSE}}(\x_t),
\end{aligned}
\end{equation*}
where the penultimate equation follows the fact that $\E_{q_{0t}(\x_0|\x_t)}[\nabla_{\x_t}\log q_{0t}(\x_0|\x_t)] = 0$. Therefore, the guidance by $\nabla_{\x_t}\gE_t^\text{MSE}(\x_t)$ is a first-order approximation of the true energy guidance by assuming $\gE_t(\x_t) \approx \gE_0(\x_0)$, which only makes sense for $t$ near to $0$. However, as shown in Fig.~\ref{fig:intermediate_energy}, for $t$ near to $T$, $\gE_t$ is quite different from $\gE_0$, so guided sampling by $\nabla_{\x_t}\gE_t^{\text{MSE}}(\x_t)$ may have large guidance errors near to $T$, which can explain why MSE is worse than CEP, especially for large $\beta$.

By further taking Taylor expansion for $\gE_0(\x_0)$ at $\E_{q_{0t}(\x_0|\x_t)}[\x_0]$, we have
\begin{equation*}
\gE_0(\x_0)
\approx
\gE_0\left(\E_{q_{0t}(\x_0|\x_t)}\left[\x_0\right]\right)
+ \left(\nabla\gE_0\left(\E_{q_{0t}(\x_0|\x_t)}[\x_0]\right)\right)^\top \left(\x_0 - \E_{q_{0t}(\x_0|\x_t)}[\x_0]\right)
\end{equation*}
Then we can further approximate $\nabla_{\x_t}\gE_t^{\text{MSE}}(\x_t)$ by
\begin{equation*}
\begin{aligned}
    \nabla_{\x_t}\gE_t^{\text{MSE}}(\x_t) &\approx \E_{q_{0t}(\x_0|\x_t)}\Big[\left(\nabla\gE_0\left(\left(\E_{q_{0t}(\x_0|\x_t)}[\x_0]\right)\right)^\top \x_0\right) \nabla_{\x_t}\log q_{0t}(\x_0|\x_t)\Big] \\
    &\quad + \left(\left(\gE_0\E_{q_{0t}(\x_0|\x_t)}\left[\x_0\right]\right) - \nabla\gE_0\left(\left(\E_{q_{0t}(\x_0|\x_t)}[\x_0]\right)\right)^\top \left(\E_{q_{0t}(\x_0|\x_t)}[\x_0]\right)\right)
    \E_{q_{0t}(\x_0|\x_t)}[\nabla_{\x_t}\log q_{0t}(\x_0|\x_t)] \\
    &= \E_{q_{0t}(\x_0|\x_t)}\Big[\left(\nabla\gE_0\left(\left(\E_{q_{0t}(\x_0|\x_t)}[\x_0]\right)\right)^\top \x_0\right) \nabla_{\x_t}\log q_{0t}(\x_0|\x_t)\Big] \\
    &= \nabla_{\x_t}\gE_t^{\text{DPS}}(\x_t),
\end{aligned}
\end{equation*}
where the penultimate equation follows the fact that $\E_{q_{0t}(\x_0|\x_t)}[\nabla_{\x_t}\log q_{0t}(\x_0|\x_t)] = 0$. Therefore, the guidance by $\nabla_{\x_t}\gE_t^\text{DPS}(\x_t)$ is a further first-order approximation of $\nabla_{\x_t}\gE_t^{\text{MSE}}(\x_t)$ by assuming $\x_0 \approx \E_{q_{0t}(\x_0|\x_t)}[\x_0]$, which also only makes sense for $t$ near to $0$. However, as $\nabla_{\x_t}\gE_t^{\text{MSE}}(\x_t)$ is also an approximation for the exact guidance $\nabla_{\x_t}\gE_t(\x_t)$, the difference between $\nabla_{\x_t}\gE_t^{\text{DPS}}(\x_t)$ and $\nabla_{\x_t}\gE_t(\x_t)$ may be rather large.

\paragraph{Additional Experiment Results.}
We further compare CEP, MSE, and DPS in both 2-D experiments and offline RL experiments. All the results show that empirically, the performance of CEP is significantly better than MSE, and MSE is significantly better than DPS. We refer to Appendix~\ref{appendix:toy_more} and Appendix~\ref{appendix:RL_baselines} for details.

\section{Relationship with Contrastive Learning}
\label{appendix:relationship_infonce}

Given a condition variable $c$, assume $(\x_0,c) \sim q_0(\x_0,c)$, and we learn the intermediate energy guidance by a neural network $f_\phi(\cdot,c,t):\R^d\rightarrow \R$ parameterized by $\phi$. In this section, we prove that for a special energy function ($\beta=1$ and $\gE(\x_0)=-\log q_0(c|\x_0)$), the objective of CEP in \eqref{Eq:infoNCE_energy} is equivalent to the traditional contrastive InfoNCE objective (for a fixed $t$).

\begin{theorem}
If $\beta=1$ and $\gE(\x_0)=-\log q_0(c|\x_0)$, the objective in \eqref{Eq:infoNCE_energy} with the sum over all possible $c$ is equivalent to
\begin{equation}
    \E_{t,\epsilonv^{(1:K)}}\E_{\prod_{i=1}^K q_0(\x_0^{(i)},c^{(i)})}\Bigg[
    -\sum_{i=1}^K \log \frac{e^{-f_\phi(\x_t^{(i)},c^{(i)},t)}}{\sum_{j=1}^K e^{-f_\phi(\x_t^{(j)},c^{(i)},t)}}
\Bigg],
\end{equation} 
\end{theorem}

\begin{proof}
Firstly, for a fixed $c$, \eqref{Eq:infoNCE_energy} becomes
\begin{equation*}
    \min_\phi \E_{p(t)}\E_{q_0(\x_0^{(1:K)})}\E_{p(\epsilonv^{(1:K)})}\Bigg[ \\
    - \sum_{i=1}^K
    q_0(c|\x_0^{(i)})
    \log \frac{e^{-f_\phi(\x_t^{(i)},c,t)}}{\sum_{j=1}^K e^{-f_\phi(\x_t^{(j)},c,t)}}
\Bigg].
\end{equation*}
By taking the integral over $c$, it becomes
\begin{equation*}
\begin{split}
&\phantom{{}\Leftrightarrow{}}\min_\phi \E_{p(t)}\E_{q_0(\x_0^{(1:K)})}\E_{p(\epsilonv^{(1:K)})}\Bigg[
    - \sum_{c}\sum_{i=1}^K
    q_0(c|\x_0^{(i)})
    \log \frac{e^{-f_\phi(\x_t^{(i)},c,t)}}{\sum_{j=1}^K e^{-f_\phi(\x_t^{(j)},c,t)}}
\Bigg] \\
&\Leftrightarrow \min_\phi \E_{p(t)}\E_{q_0(\x_0^{(1:K)})}\E_{p(\epsilonv^{(1:K)})}\Bigg[
    - \sum_{i=1}^K\sum_{c}
    q_0(c|\x_0^{(i)})
    \log \frac{e^{-f_\phi(\x_t^{(i)},c,t)}}{\sum_{j=1}^K e^{-f_\phi(\x_t^{(j)},c,t)}}
\Bigg] \\
&\Leftrightarrow \min_\phi \E_{p(t)}\E_{q_0(\x_0^{(1:K)})}\E_{p(\epsilonv^{(1:K)})}\Bigg[
    - \sum_{i=1}^K \E_{q(c|\x_0^{(i)})} \left[
    \log \frac{e^{-f_\phi(\x_t^{(i)},c,t)}}{\sum_{j=1}^K e^{-f_\phi(\x_t^{(j)},c,t)}}
    \right]
\Bigg] \\
&\Leftrightarrow \min_\phi \E_{p(t)}\E_{q_0(\x_0^{(1:K)})}\E_{p(\epsilonv^{(1:K)})}\Bigg[
    - \sum_{i=1}^K \E_{q(c^{(i)}|\x_0^{(i)})} \left[
    \log \frac{e^{-f_\phi(\x_t^{(i)},c^{(i)},t)}}{\sum_{j=1}^K e^{-f_\phi(\x_t^{(j)},c^{(i)},t)}}
    \right]
\Bigg] \\
&\Leftrightarrow \min_\phi \E_{p(t)}\E_{q_0(\x_0^{(1:K)}, c^{(1:K)})}\E_{p(\epsilonv^{(1:K)})}\Bigg[
    - \sum_{i=1}^K
    \log \frac{e^{-f_\phi(\x_t^{(i)},c^{(i)},t)}}{\sum_{j=1}^K e^{-f_\phi(\x_t^{(j)},c^{(i)},t)}}
\Bigg]
\end{split}
\end{equation*}

\end{proof}

Note that the above objective assumes that we can draw samples from $q_0(\x_0, c)$, which means that we can draw samples from $p_0(\x_0)=q_0(\x_0|c)$. However, for general energy functions, such an assumption is hard to ensure and we can not draw samples from $p_0(\x_0)$ but only $q_0(\x_0)$. Therefore, CEP can be understood as a generalized version of the traditional contrastive objective and is suitable for both conditional sampling and energy-guided sampling in diffusion models.

\section{Relationship between Inverse Temperature and Guidance Scale}
A widely-used trick in guided sampling is to introduce an additional hyperparameter $s$ called ``guidance scale''~\citep{diffusion_beat_gan}, and replace the score function for $p_t$ by $\tilde p_t$ during the guided sampling procedure as following:
\begin{equation}
\label{Eqn:appendix_guidance_scale}
    \nabla_{\x_t}\log \tilde p_t(\x_t) \coloneqq \nabla_{\x_t}\log q_t(\x_t) - s \cdot \nabla_{\x_t}\gE_t(\x_t)
\end{equation}
According to \eqref{Eqn:appendix_exact_guidance}, the above equation is equivalent to
\begin{equation*}
\begin{split}
    \nabla_{\x_t}\log \tilde p_t(\x_t) &= \nabla_{\x_t}\log q_t(\x_t)
    - s\cdot \E_{q_{0t}(\x_0|\x_t)}\left[-e^{\gE_t(\x_t)-\gE_0(\x_0)}\nabla_{\x_t}\log q_{0t}(\x_0|\x_t)\right] \\
    &= \nabla_{\x_t}\log q_t(\x_t)
    + s\cdot \frac{\E_{q_{0t}(\x_0|\x_t)}\left[e^{-\beta\gE(\x_0)}\nabla_{\x_t}\log q_{0t}(\x_0|\x_t)\right]}{\E_{q_{0t}(\x_0|\x_t))}\left[e^{-\beta\gE(\x_0)}\right]}.
\end{split}
\end{equation*}
Note that the influences of changing $s$ and changing $\beta$ are different: changing $s$ will linearly affect the guidance strength, but changing $\beta$ will affect the guidance w.r.t. the exponential term. Thus, $s$ and $\beta$ are two different hyperparameters and we can tune them together.

Empirically, we find that in simple 2-D experiments, only changing $\beta$ is enough to guarantee convergence to our desired distribution $p(\x)$. However, in complex tasks such as image synthesis and reinforcement learning, by only varying $\beta$ we cannot guarantee a good performance, so a larger $s$ is somewhat necessary. Our hypothesis for explaining this is that the neural network used is not expressive enough, such that when $\beta$ increases and the task becomes more complex, the model capacity approaches saturation so we must rely on a training-free method in order to amplify the guidance effect (\cref{sec:exp_details_rl}).

\section{E-MSE Energy Guidance}
\label{sec:emse}
In this section, we propose an alternative way of CEP to learn energy guidance. In order to ensure an exact converged point, we add an exponential activation in the original MSE-based training objective (\eqref{Eq:MSE_qt_objective}), which we name E-MSE:
\begin{equation}
\label{Eq:emse}
    \min_{\phi} \E_{t, \x_0, \x_t} \left[
    \|\exp(f_\phi(\x_t, t)) - \exp(\beta \gE(\x_0))\|_2^2
\right]
\end{equation}
Such a training objective could also guarantee convergence to the exact energy guidance for sampling from $p(\x)$. As visualized in \cref{fig:toy_appendix}, we find that both CEP and E-MSE guidance could generate more accurate data samples in 2-D settings compared to the MSE-based guidance method, especially when $\beta$ is large. However, one main disadvantage of E-MSE guidance is that \eqref{Eq:emse} is not numerically stable due to the isolated exponential term. In particular, in RL settings where the energy function $\gE$ is no longer normalized in the range $[0, 1]$, but is defined by a potentially noisy neural network $Q_\psi$, E-MSE guidance generally tends to underperform CEP and even MSE guidance (\cref{tab:Ablation_RL}).  
\newpage
\section{Experiment details for offline RL}
\label{Sec:rl_detail}

\subsection{Pseudocode of QGPO}
\label{Sec:pseudocode}
\begin{algorithm}
    \caption{Q-Guided Policy Optimization (QGPO) for Offline RL}
    \label{alg:offline_rl}
    \begin{algorithmic}
        \STATE Initialize the diffusion behavior model $\epsilonv_{\theta}$, the action evaluation model $Q_\psi$ and the intermediate energy model $f_\phi$
        \STATE \small{\color{gray} \textit{// Training the behavior model}}
        \FOR {each gradient step}
        \STATE Sample $B$ data points $\left(\vs, \va\right)$ from $\mathcal{D}^\mu$, $B$ Gaussian noises $\epsilon$ from $\mathcal{N}(0,\bm{I})$ and $B$ time $t$ from $\mathcal{U}(0, T)$
        \STATE Perturb $\va$ according to $\va_t := \alpha_t \va + \sigma_t \bm{\epsilon}$
        \STATE Update $\theta \leftarrow \theta - \lambda_\theta\nabla_{\theta} \sum[\| \epsilonv_\theta(\va_t|\vs,t) - \bm{\epsilon} \|_2^2]$
        \ENDFOR
        \STATE \small{\color{gray} \textit{// Generating the support action set}}
        \FOR {each state $\vs$ in $\mathcal{D}^\mu$}
        \STATE Sample $K$ support actions ${\hat \va}^{(1:K)}$ from $\mu_\theta(\cdot|\vs)$ and store them as $\mathcal{D}^{\mu_\theta}(\vs)$
        \ENDFOR
        \STATE \small{\color{gray} \textit{// Training the action evaluation model and the energy guidance model}}
        \FOR {each gradient step}
        \STATE Sample $B$ data points $\left({\vs, \va, r, \vs'}\right)$ from $\mathcal{D}^\mu$, $B$ Gaussian noises $\epsilon$ from $\mathcal{N}(0,\bm{I})$ and $B$ time $t$ from $\mathcal{U}(0, T)$
        \STATE Retrieve support action sets ${\hat \va}^{(1:K)}$ and ${\hat \va}^{\prime(1:K)}$ respectively from $\mathcal{D}^\mu(\vs)$ and $\mathcal{D}^\mu(\vs')$
        \STATE Calculate the target Q-value $\mathcal{T}^\pi Q_\psi(\vs, \va) = r(\vs, \va) + \gamma\sum_{\hat\va'} \bigg[\frac{e^{\beta_Q Q_\psi(\vs',\hat\va')}}{\sum_{\hat\va'} e^{\beta_Q Q_\psi(\vs',\hat\va')}}Q_\psi(\vs',\hat\va')\bigg]$ and detach gradient
        \STATE Update $\psi \leftarrow \psi - \lambda_\psi\nabla_{\psi} \sum[\| Q_\psi(\vs, \va) - \mathcal{T}^\pi Q_\psi(\vs, \va) \|_2^2]$
        \STATE Perturb $\hat\va$ according to $\hat\va_t := \alpha_t \hat\va + \sigma_t \bm{\epsilon}$
        \STATE Update $\phi \leftarrow \phi + \lambda_\phi\nabla_{\phi} \sum_i[ \frac{e^{\beta Q_\psi(\vs, \hat \va_i)}}{\Sigma_j e^{\beta Q_\psi(\vs, \hat \va_j)}}  \log\frac{e^{f_\phi(\hat \va_{i, t}| \vs, t)}}{\Sigma_j e^{f_\phi(\hat \va_{j, t}| \vs, t)}} ]$
        \ENDFOR
    \end{algorithmic}
\end{algorithm}

\subsection{Experiment Details of QGPO}
\label{sec:exp_details_rl}
For offline RL benchmarks, our methods require training three neural networks in total for each task, namely a diffusion-based behavior model $s_\theta$, an action evaluation model $Q_\psi$, and an energy guidance model $f_\phi$. We first provide experiment details in training every component described above and then discuss how to combine these components for policy evaluation.

\paragraph{Training behavior model.}
The architecture and training method of our behavior model completely follow \citet{sfbc}. The network architecture resembles U-Nets, but with
spatial convolutions changed to dense connections, such that it is compatible with a vectorized data representation. A similar network architecture was also adopted by \citet{diffuser} and \citet{csbc}. We train the behavior model for 600k gradient steps, using the Adam optimizer with a learning rate of 1e-4. The batchsize is 4096. As for the data perturbation method, we adopt the default VPSDE setting in \citep{song2020score} with a linear schedule. The $\alpha_t$ and $\sigma_t$ in \eqref{Eq:forward_p} are:
\begin{equation}
    \alpha_t = -\frac{\beta_1 - \beta_0}{4}t^2 - \frac{\beta_0}{2}t, \quad \sigma_t = \sqrt{1 - \alpha_t^2}, \quad \beta_0=0.1,\ \beta_1=20.
\end{equation}

\paragraph{Training action evaluation model.}
The action evaluation model is a 3-layer MLP with 256 hidden units and ReLU activations. We train the action evaluation model for 500k gradient steps, using the Adam optimizer with a learning rate of 3e-4. The batchsize is 256. We set $\beta_Q=1$ and $K=16$ for MuJoCo Locomotion tasks, $\beta_Q=20$ and $K=32$ for AntMaze tasks in \eqref{Eq:one_step_bellman_softmax}. 
Before training, we follow \citet{iql} and normalize task rewards. We also use standard tricks such as soft updates \citep{ddpg} and double networks \citep{td3} to stabilize Q-learning.

\paragraph{Training energy guidance model.}
The energy guidance model is a 4-layer MLP with 256 hidden units and SiLU activations. We train it for 1M gradient steps, using the Adam optimizer with a learning rate of 3e-4. The batchsize is 256. The size support action set is the same as the one used in training the action evaluation model ($K=16$ for MuJoCo Locomotion and $K=32$ for AntMaze), though we set $\beta=3$ in all tasks.

\paragraph{Evaluation.}
We run all experiments over 5 independent trials and report their averaged performance. For each trial, we additionally collect the evaluation score averaged on multiple test seeds (10 for MuJoCo Locomotion and 100 for AntMaze). In order to sample from the learned diffusion-based policy, we adopt a recent advance in diffusion sampling, namely DPM-Solver \citep{lu2022dpm}. We use the second-order sampler and report performance scores at a diffusion step of 15. We conduct an ablation study of diffusion steps (\cref{fig:ablation_diffusion_steps}) and find that a diffusion step of 10 could already yield equally good performance, while a diffusion step of 5 only slightly underperforms 15 diffusion steps. Note that we only ablated diffusion step numbers in evaluation. The support action set is still generated using a diffusion step of 15. We also adopted a widely used trick in guided diffusion sampling \citep{diffusion_beat_gan, ho2021classifier}, which tunes a hyperparameter $s$ to amplify the guidance effect during sampling, by multiplying energy guidance in \eqref{Eq:rl_energy_guidance} with the guidance scale $s$. To choose the optimal energy guidance scales $s$ for action sampling, we sweep over $[1.0, 2.0, 3.0, 5.0, 8.0, 10.0]$ for MuJoCo Locomotion and $[1.0, 1.5, 2.0, 2.5, 3.0, 4.0]$ for AntMaze tasks during evaluation (\cref{fig:ablation_gradient_scale}). Gradient scales used for reported scores are listed in \cref{tbl:gradient_scales_rl}.

\begin{figure}[h]
\centering
\includegraphics[width=\linewidth]{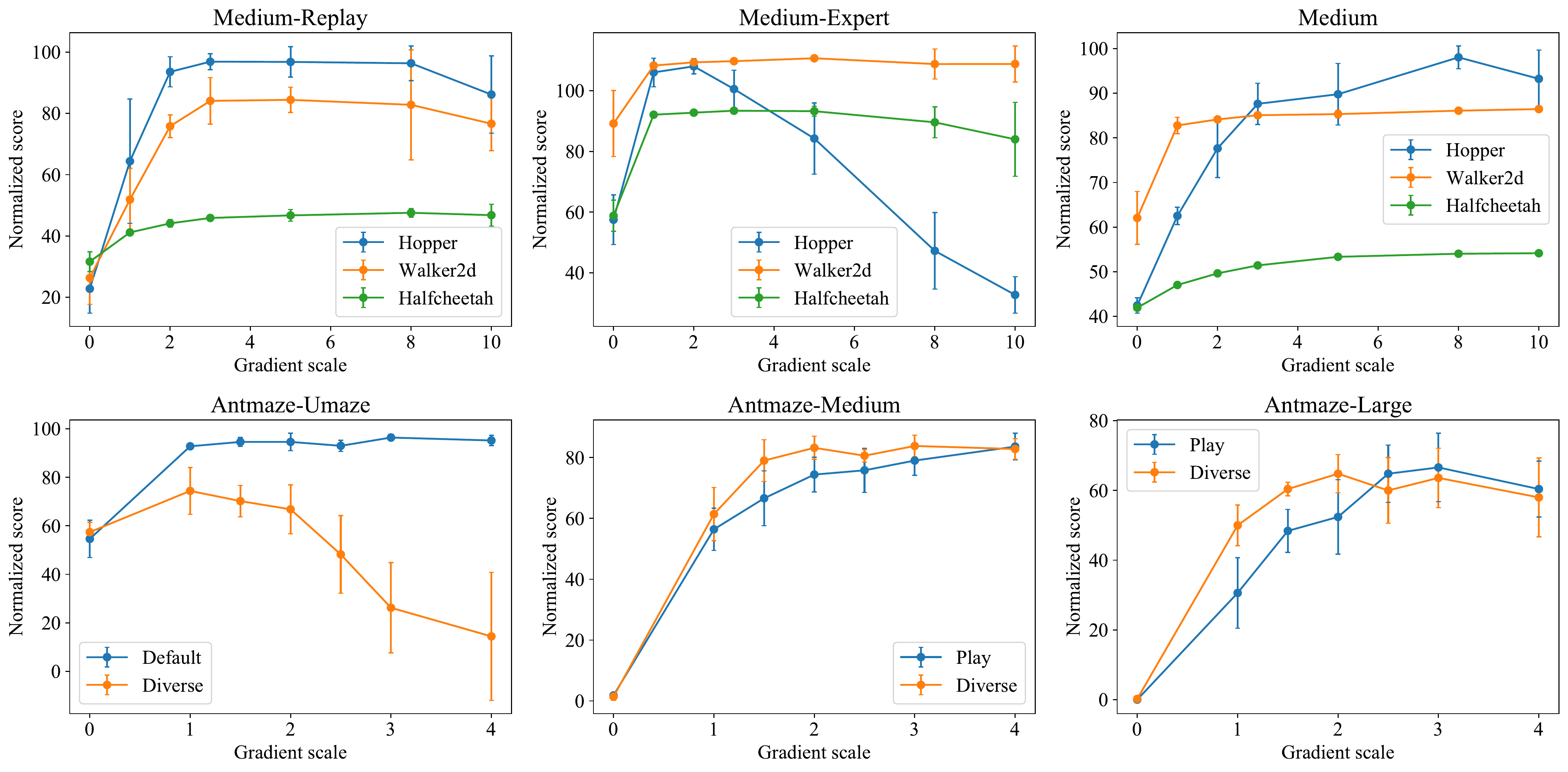}
\caption{
Ablation of gradient scales in D4RL benchmark.
}
\label{fig:ablation_gradient_scale}
\end{figure}

\begin{figure}[h]
\centering
\includegraphics[width=0.5\linewidth]{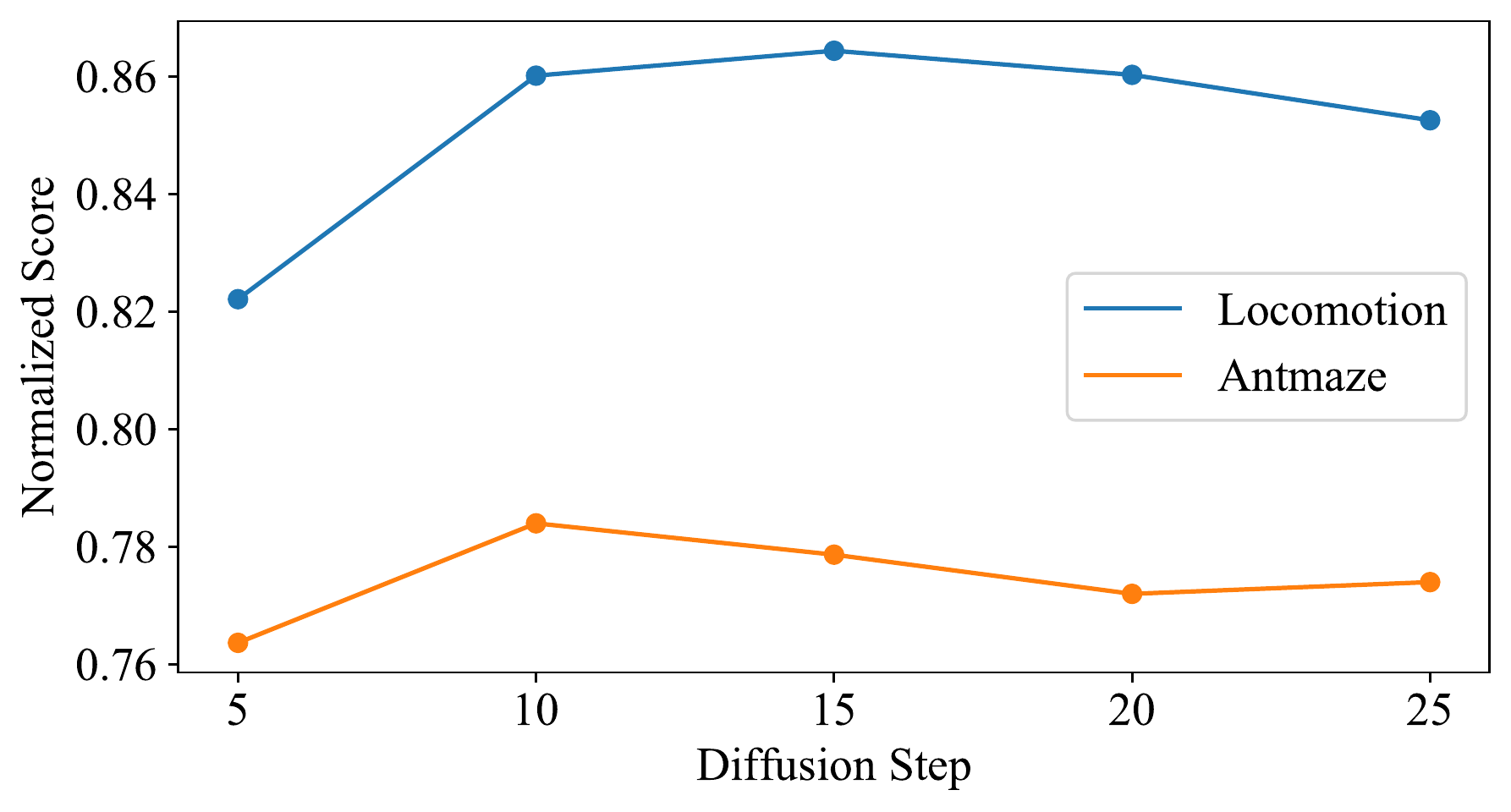}

\caption{
Ablation of diffusion steps in evaluation.
}
\label{fig:ablation_diffusion_steps}
\end{figure}

\begin{table}[h]
\centering
\begin{tabular}{c|c|c|c}
\hline
\multirow{2}*{Locomotion-Medium-Expert} & Walker2d & Halfcheetah & Hopper\\
~ & 5.0 & 3.0 & 2.0\\
\hline
\multirow{2}*{Locomotion-Medium} & Walker2d & Halfcheetah & Hopper\\
~ & 10.0 & 10.0 & 8.0\\
\hline
\multirow{2}*{Locomotion-Medium-Replay} & Walker2d & Halfcheetah & Hopper\\
~ & 5.0 & 8.0 & 3.0\\
\hline    
\multirow{2}*{AntMaze-Fixed} & Umaze & Medium & Large\\
~ & 3.0 & 4.0 & 3.0\\
\hline
\multirow{2}*{AntMaze-Diverse} & Umaze & Medium & Large\\
~ & 1.0 & 3.0 & 2.0\\
\hline
\end{tabular}
\caption{Guidance scale $s$ used across different tasks}
\label{tbl:gradient_scales_rl}
\end{table}

\subsection{Experiment Details for Other Baselines}
\label{appendix:RL_baselines}
\paragraph{Diffusion-QL.}
We use the official implementation of Diffusion-QL ( \url{https://github.com/Zhendong-Wang/Diffusion-Policies-for-Offline-RL}) and default hyperparameter settings to rerun all experiments for Diffusion-QL to ensure a consistent evaluation metric. We follow \citet{d4rl} and report averaged scores across 5 independent trails at the end of training for each task, instead of the max performance scores during training as in \citet{dql}. In addition, we notice that Diffusion-QL adopts a resampling technique for evaluation. Specifically, during evaluation, the learned policy first generates 50 different action candidates and then selects one action with the highest Q-value for execution. We empirically found that such a technique is important for good performance in MuJoCo Locomotion tasks. However, this technique makes it hard to reflect the true quality of sampled actions before resampling and is computationally expensive, we thus additionally conduct an ablation study in which we remove the resampling procedure in evaluation, and use a single action candidate (Diffusion-QL@1 in \cref{tbl:rl_results} and \cref{sec:training_curves}). 

\paragraph{Ablations of CEP guidance.}
We study three variations of our proposed guidance method. Specifically, an MSE-based guidance method as described in \eqref{Eq:MSE_qt_objective} (similarly used in \citet{diffuser}), an E-MSE guidance method as described in \eqref{Eq:emse}, and a resampling-based method following \citet{sfbc}. For MSE and E-MSE guidance, we only change the training objective of the energy guidance model while leaving other hyperparameters untouched. For the resampling-based method, the energy guidance model is not required. During evaluation, at every decision step we first sample 50 random action candidates from the behavior policy model $\mu_\theta(\va|\vs)$ and then select one action with the highest predicted Q-value via $Q_\psi$ for execution.

\begin{table}[]
    \centering
    \begin{tabular}{lcccc}
        \toprule
        \textbf{Environment} & \textbf{MSE}& \textbf{E-MSE} & \textbf{RS} & \textbf{CEP} \\
        \midrule
        \textbf{Locomotion} & $68.0$  & $58.1$ & $76.9$  & $86.6$\\ 
        \textbf{AntMaze} & $46.4$ & $24.5$ & $63.0$ & $78.3$ \\ 
        \bottomrule
    \end{tabular}
    \caption{Ablation studies of different energy guidance methods and the resampling technique.}
    \label{tab:Ablation_RL}
\end{table}

\section{Experiment Details for 2-D experiments}
To perform unconditional energy-guided sampling in low-dimensional data space. Our method requires training two neural networks independently, specifically, one generative diffusion model and one energy guidance model. In contrast with offline RL, the energy function at $t=0$ data space is pre-defined (as illustrated in \cref{fig:toy_appendix}) and does not require training. A total of 1M datapoints is generated and used as the training set. Each datapoint contains a two-dimensional data sample $\x$ and a float number $e$ representing its energy.

\paragraph{Training diffusion generative models.}
The generative diffusion model is a 5-layer MLP with hidden sizes of $[512, 512, 512, 512, 256]$ and SiLU activations. The network is trained for 750 epochs, using the Adam optimizer with a learning rate of 1e-4. The batchsize is 1 with $K=4096$. As for the data perturbation method, we adopt the default VPSDE setting in \citep{song2020score} with a linear schedule. The $\alpha_t$ and $\sigma_t$ in \eqref{Eq:forward_p} are:
\begin{equation}
    \alpha_t = -\frac{\beta_1 - \beta_0}{4}t^2 - \frac{\beta_0}{2}t, \quad \sigma_t = \sqrt{1 - \alpha_t^2}, \quad \beta_0=0.1,\ \beta_1=20.
\end{equation}

\paragraph{Training energy guidance models.}
The energy guidance model is a 4-layer MLP with 512 hidden units and SiLU activations. We train it for 750 epochs, using the Adam optimizer with a learning rate of 3e-4. The batchsize is also 4096.

\paragraph{Guided sampling.}
 In order to perform guided sampling, we adopt a recent advance in diffusion sampling, namely DPM-Solver \citep{lu2022dpm}. We use the second-order sampler and a diffusion step of 25. We fix the guidance scale $s$ to 1 in all experiments.

\section{Experiment Details for Image Synthesis}
We completely follow \citet{diffusion_beat_gan} to train and evaluate our energy-guided diffusion models for image synthesis, without any kind of hyperparameter tuning or network architecture changes. For the generative diffusion prior, we use the pretrained ImageNet models released at \url{https://github.com/openai/guided-diffusion}. For the energy guidance model, we adopt the same U-Net architecture as \citet{diffusion_beat_gan} but rewrite the training objective to \eqref{Eq:contarstive_condition_loss} for conditional image synthesis, and to \eqref{Eq:infoNCE_energy_normalized} for energy guided image synthesis. Our ImageNet 128 energy guidance model is trained for 300k gradient steps with a batch size of 256 (distributed on 8 GPUs). The ImageNet 256 energy guidance model is trained for 500k steps. During sampling, we use 250 diffusion steps by default except when we use a DDIM \citep{song2020denoising} sampler with 25 steps. 

For energy-guided image synthesis tasks, we set $\beta=50$. A penalty is added to the energy function at defined $t=0$ data space in \eqref{Eq:color_q0} when the image's average saturation is lower than $0.1$. This penalization mainly intends to avoid generating images with low saturation (overly bright), such that image samples guided by different color models are more distinguishable. We respectively let $h_\text{tar}$ be $0$ (red), $2\pi/3$ (green) and $4\pi/3$ (blue) for the three guidance models.

\newpage
\section{More 2-D Results}
\label{appendix:toy_more}

\begin{figure}[!hbt]
\centering
\begin{minipage}{0.03\linewidth}
\rotatebox{90}{\scriptsize{Groundtruth}}
\end{minipage}
\begin{minipage}{0.46\linewidth}
\includegraphics[width=\columnwidth]{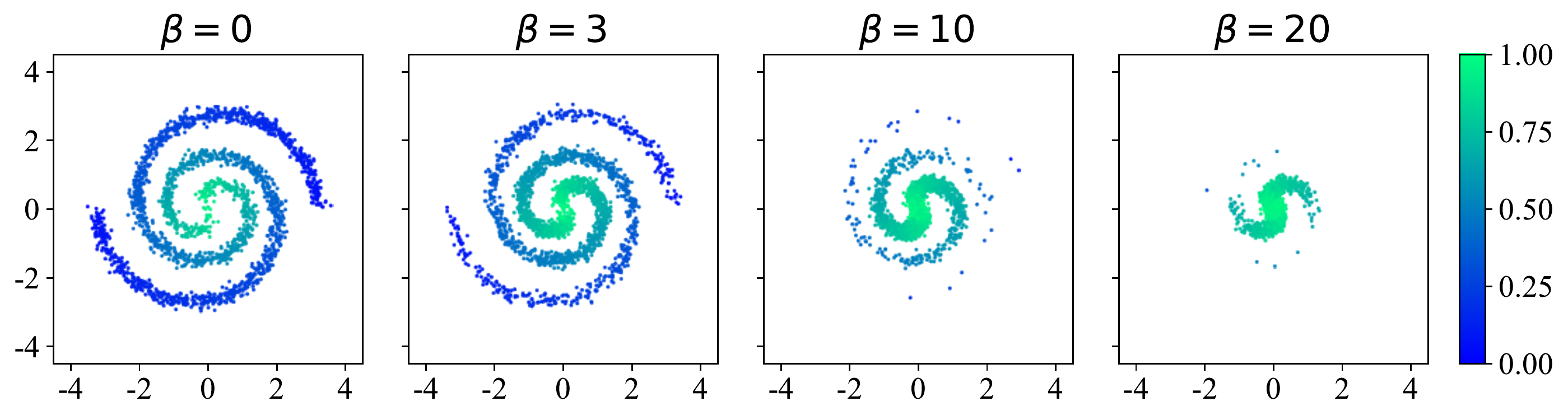}
\end{minipage}
\begin{minipage}{0.46\linewidth}
\includegraphics[width=\columnwidth]{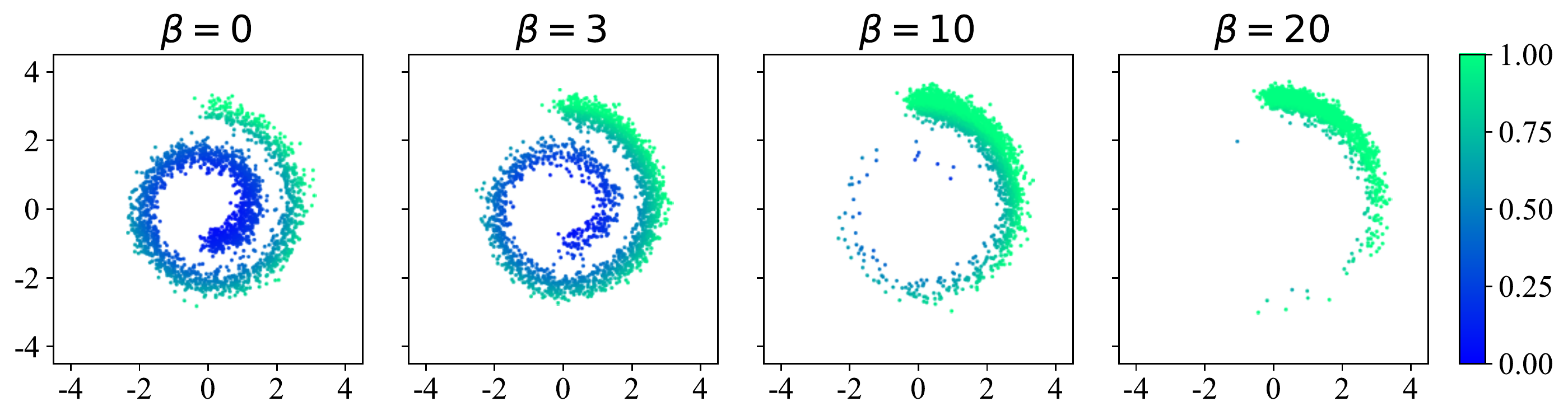}
\end{minipage}\\
\begin{minipage}{0.015\linewidth}
\rotatebox{90}{\scriptsize{DPS}}
\end{minipage}
\begin{minipage}{0.015\linewidth}
\rotatebox{90}{\scriptsize{(\citeauthor{chung2022diffusion}, etc.)}}
\end{minipage}
\begin{minipage}{0.46\linewidth}
\includegraphics[width=\columnwidth]{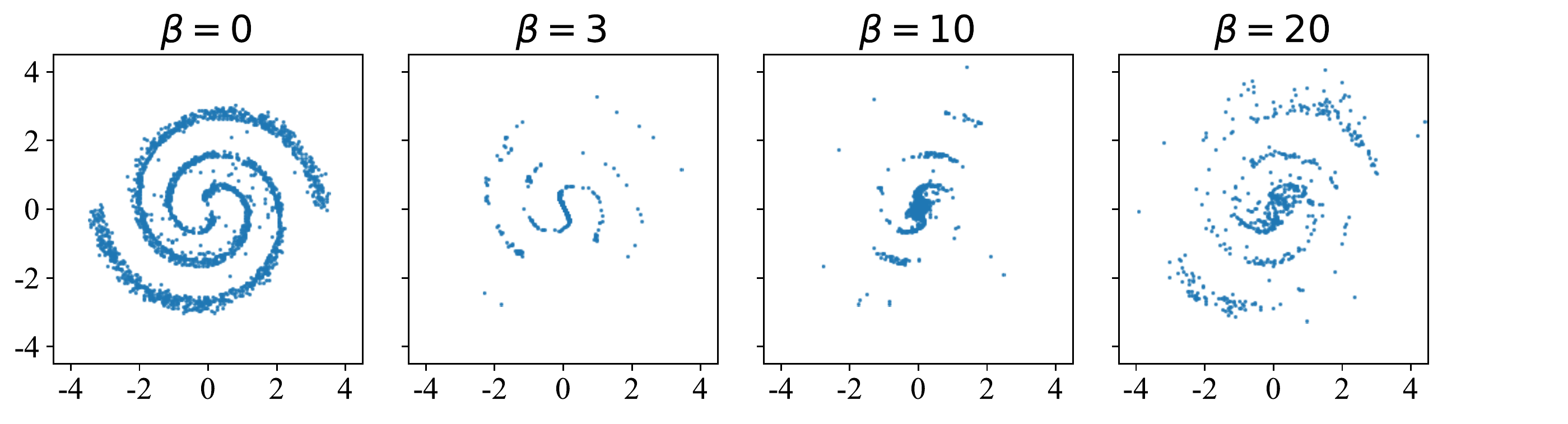}
\end{minipage}
\begin{minipage}{0.46\linewidth}
\includegraphics[width=\columnwidth]{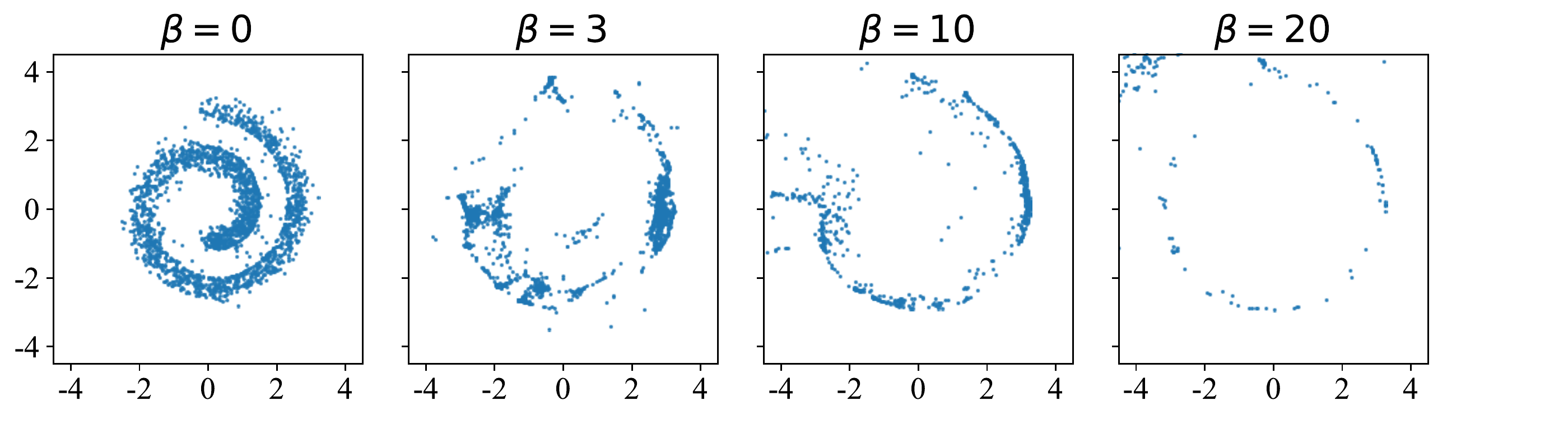}
\end{minipage}\\
\begin{minipage}{0.015\linewidth}
\rotatebox{90}{\scriptsize{MSE}}
\end{minipage}
\begin{minipage}{0.015\linewidth}
\rotatebox{90}{\scriptsize{(\citeauthor{diffuser}, etc.)}}
\end{minipage}
\begin{minipage}{0.46\linewidth}
\includegraphics[width=\columnwidth]{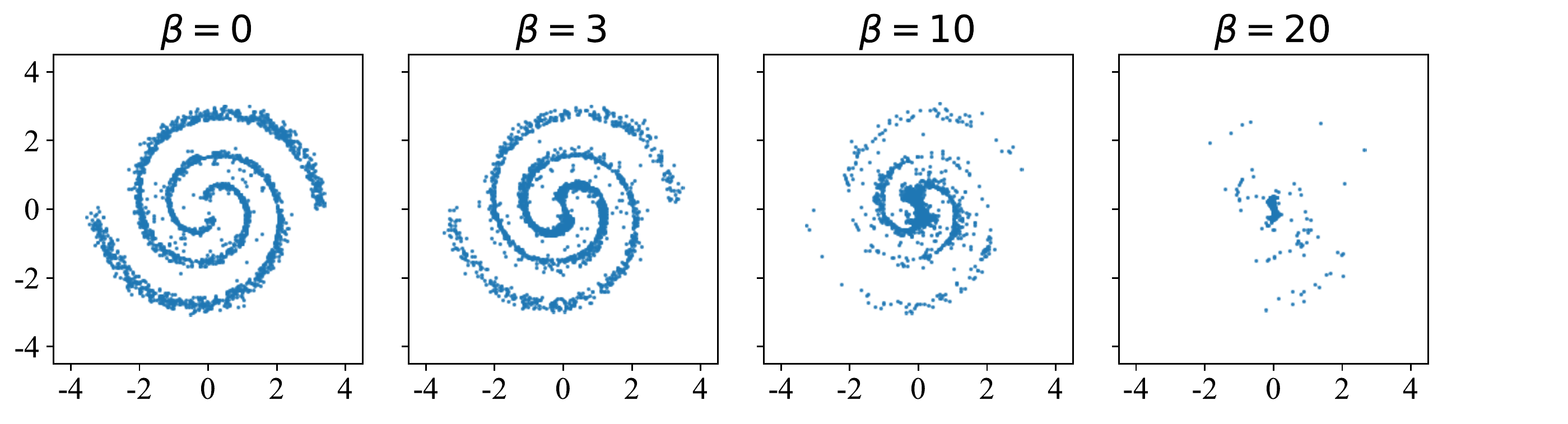}
\end{minipage}
\begin{minipage}{0.46\linewidth}
\includegraphics[width=\columnwidth]{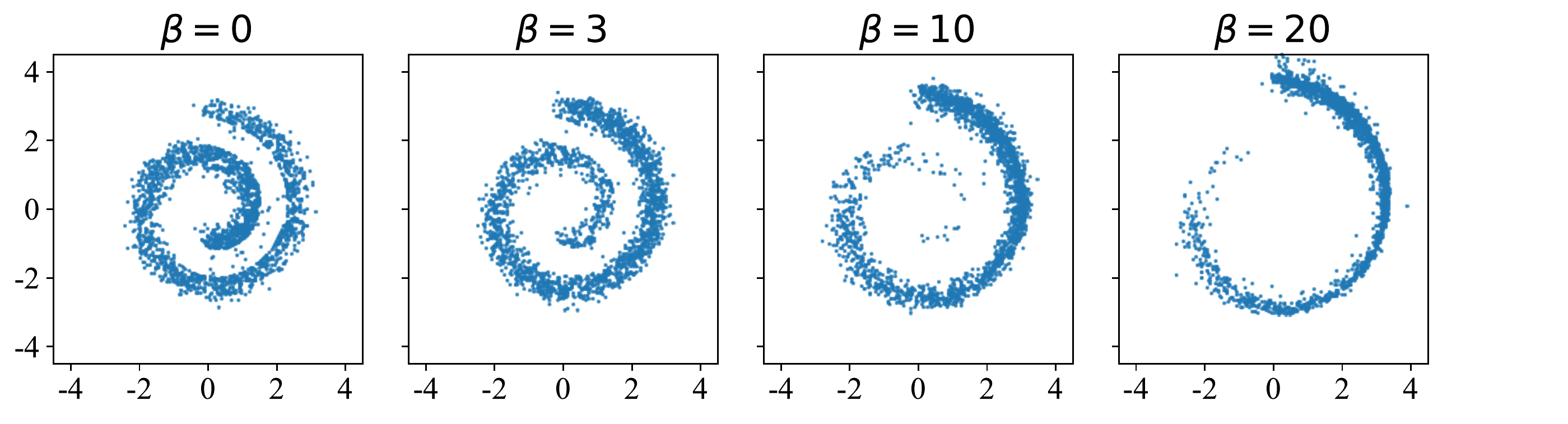}
\end{minipage}\\
\begin{minipage}{0.015\linewidth}
\rotatebox{90}{\scriptsize{E-MSE}}
\end{minipage}
\begin{minipage}{0.015\linewidth}
\rotatebox{90}{\scriptsize{(\textbf{ours})}}
\end{minipage}
\begin{minipage}{0.46\linewidth}
\includegraphics[width=\columnwidth]{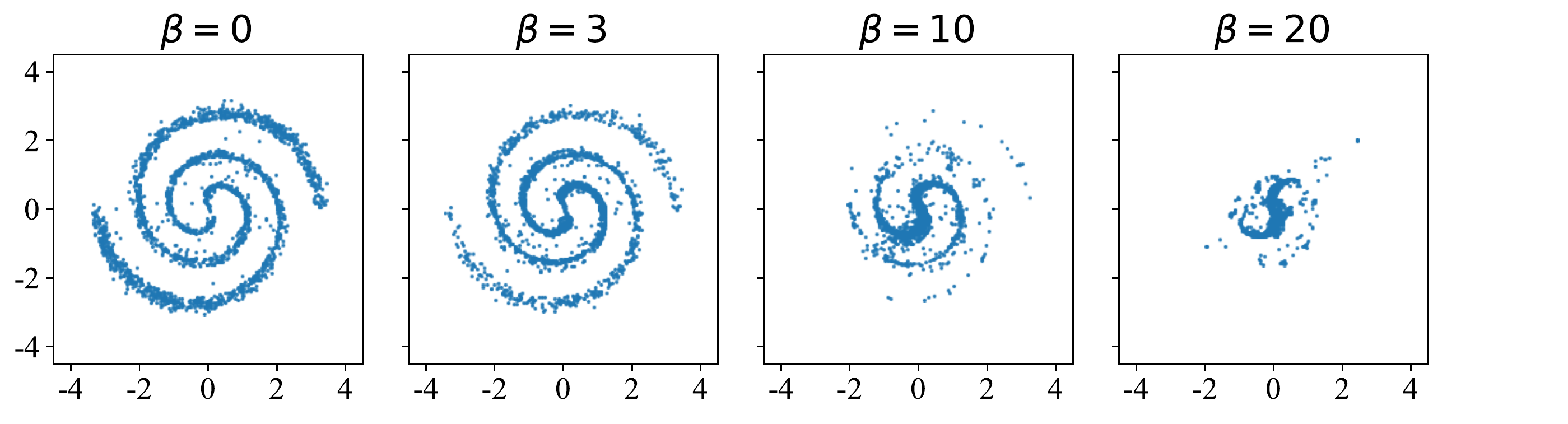}
\end{minipage}
\begin{minipage}{0.46\linewidth}
\includegraphics[width=\columnwidth]{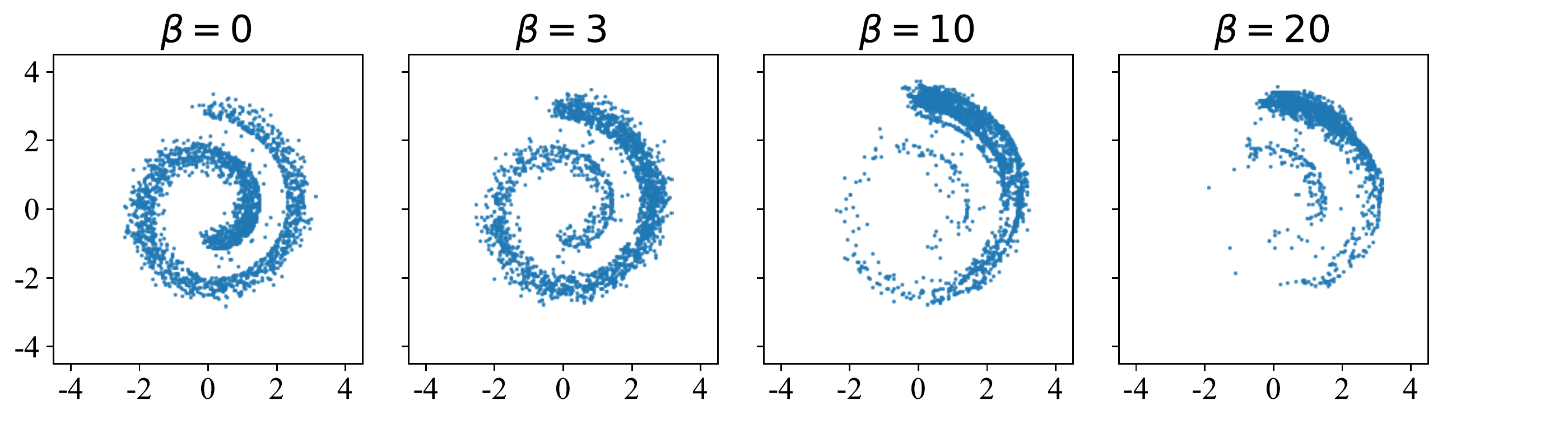}
\end{minipage}\\
\begin{minipage}{0.015\linewidth}
\rotatebox{90}{\scriptsize{CEP}}
\end{minipage}
\begin{minipage}{0.015\linewidth}
\rotatebox{90}{\scriptsize{(\textbf{ours})}}
\end{minipage}
\begin{minipage}{0.46\linewidth}
\includegraphics[width=\columnwidth]{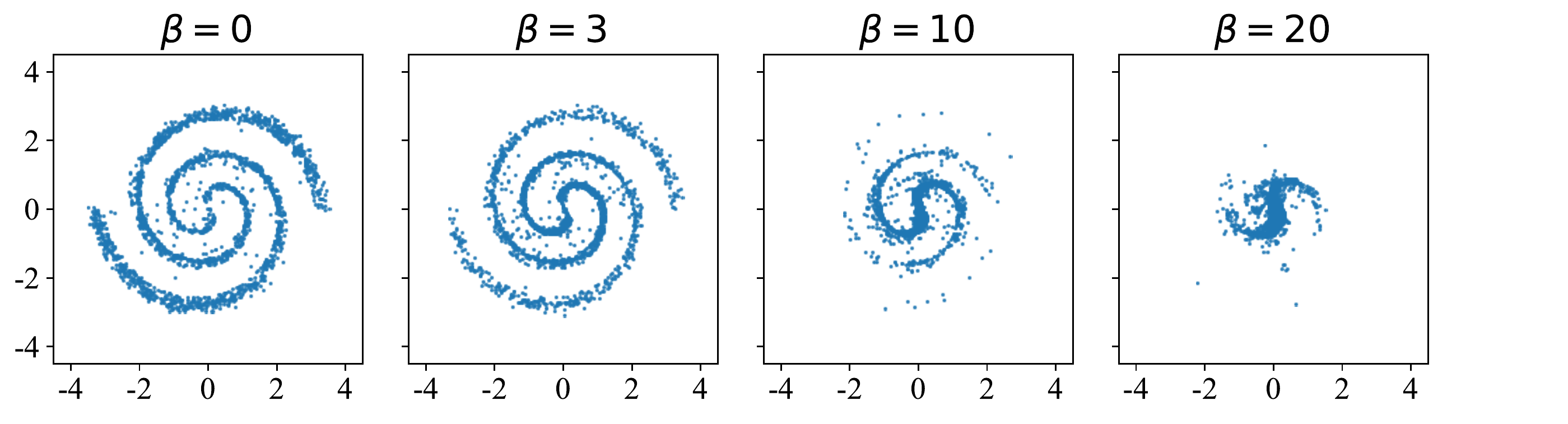}
\end{minipage}
\begin{minipage}{0.46\linewidth}
\includegraphics[width=\columnwidth]{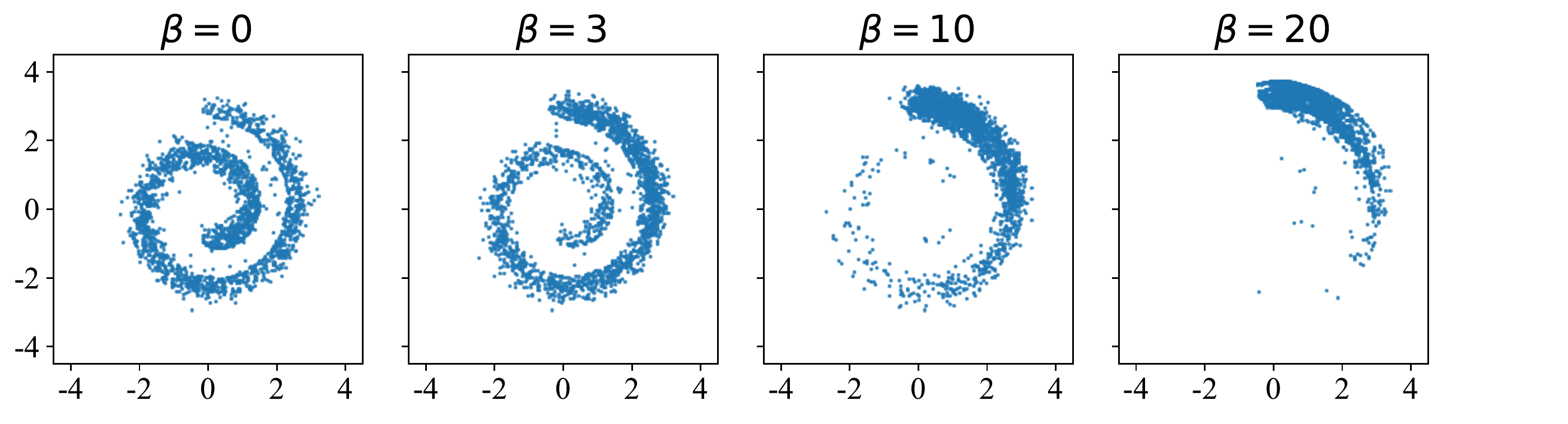}
\end{minipage}\\
\begin{minipage}{0.03\linewidth}
\rotatebox{90}{\scriptsize{Groundtruth}}
\end{minipage}
\begin{minipage}{0.46\linewidth}
\includegraphics[width=\columnwidth]{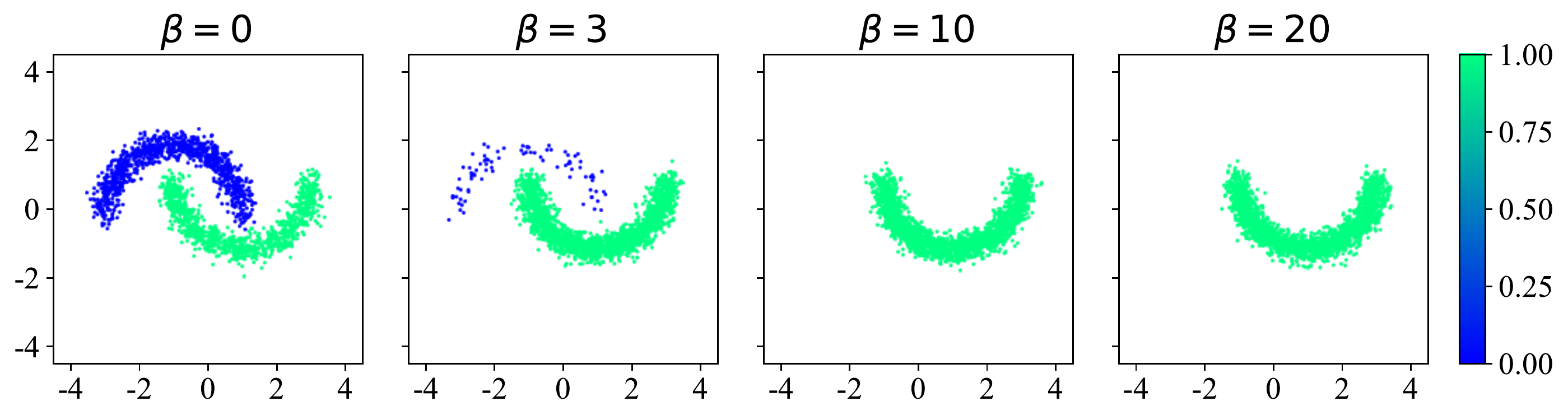}
\end{minipage}
\begin{minipage}{0.46\linewidth}
\includegraphics[width=\columnwidth]{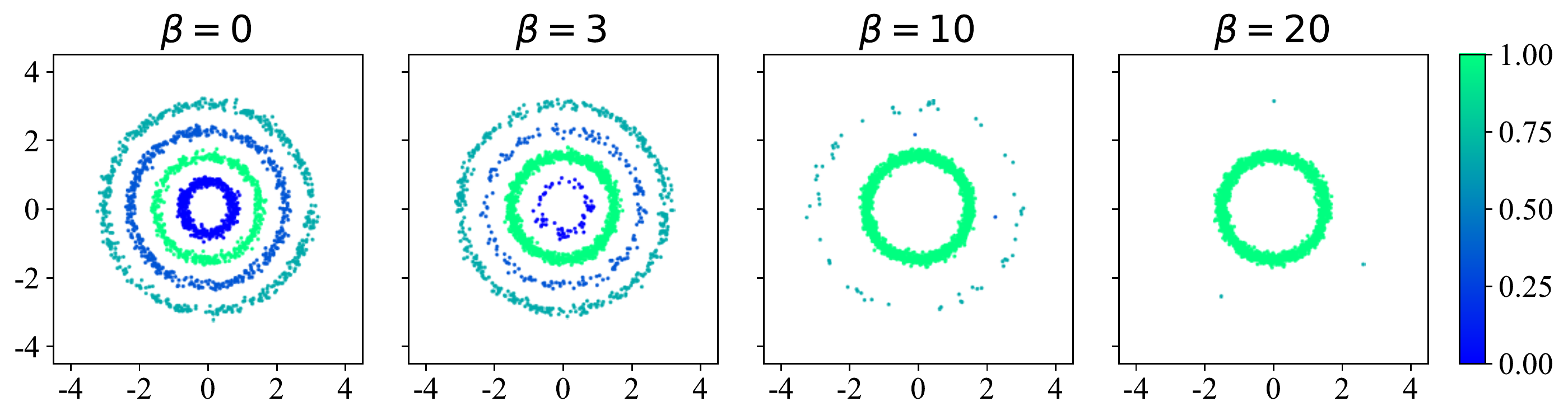}
\end{minipage}\\
\begin{minipage}{0.015\linewidth}
\rotatebox{90}{\scriptsize{DPS}}
\end{minipage}
\begin{minipage}{0.015\linewidth}
\rotatebox{90}{\scriptsize{(\citeauthor{chung2022diffusion}, etc.)}}
\end{minipage}
\begin{minipage}{0.46\linewidth}
\includegraphics[width=\columnwidth]{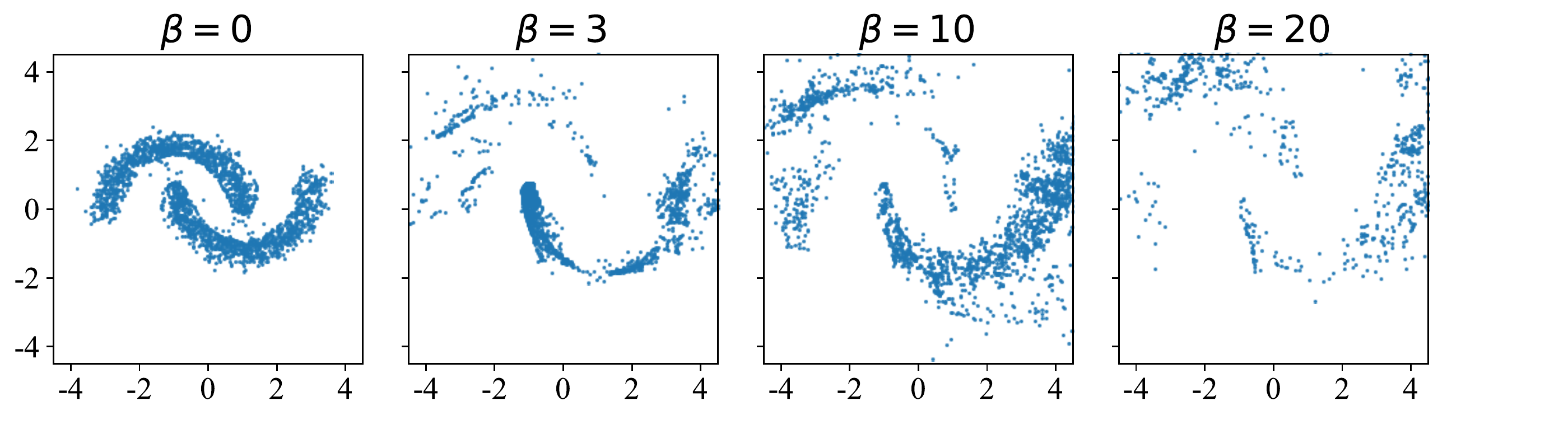}
\end{minipage}
\begin{minipage}{0.46\linewidth}
\includegraphics[width=\columnwidth]{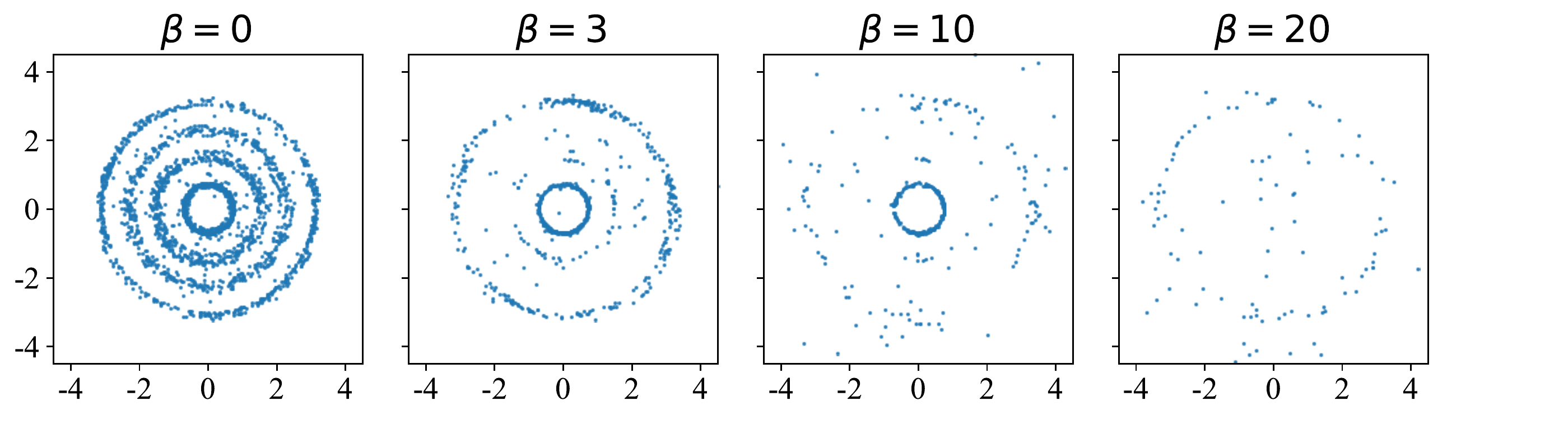}
\end{minipage}\\
\begin{minipage}{0.015\linewidth}
\rotatebox{90}{\scriptsize{MSE}}
\end{minipage}
\begin{minipage}{0.015\linewidth}
\rotatebox{90}{\scriptsize{(\citeauthor{diffuser}, etc.)}}
\end{minipage}
\begin{minipage}{0.46\linewidth}
\includegraphics[width=\columnwidth]{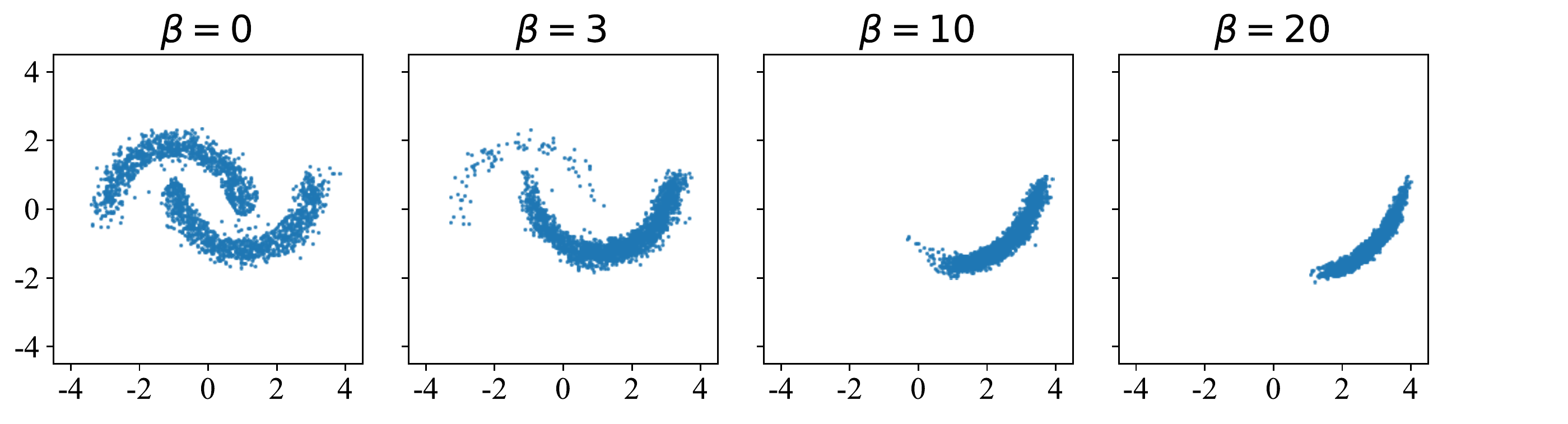}
\end{minipage}
\begin{minipage}{0.46\linewidth}
\includegraphics[width=\columnwidth]{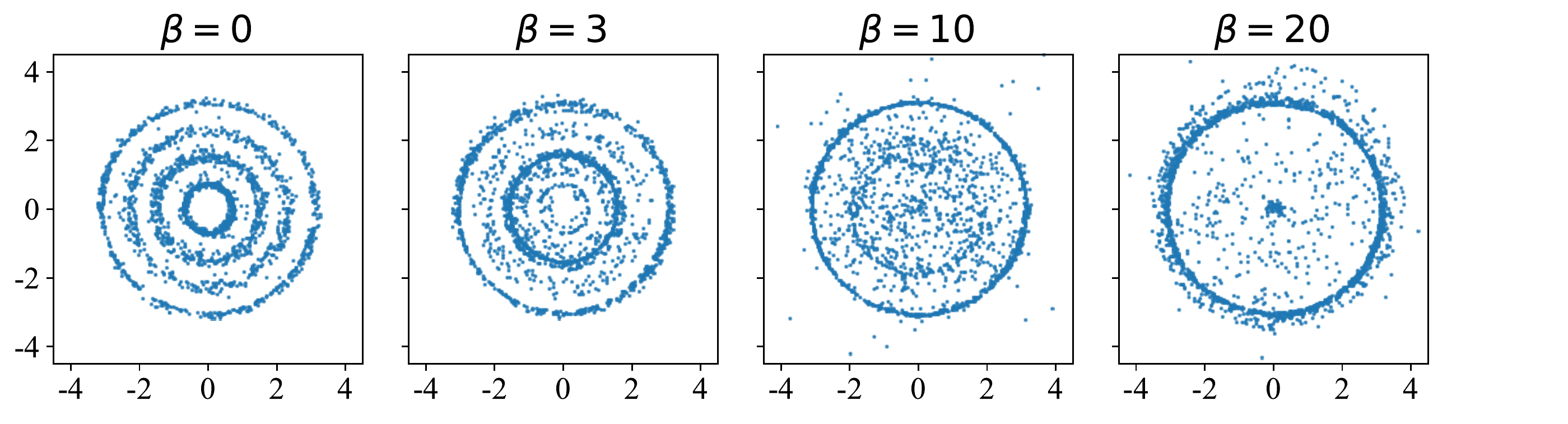}
\end{minipage}\\
\begin{minipage}{0.015\linewidth}
\rotatebox{90}{\scriptsize{E-MSE}}
\end{minipage}
\begin{minipage}{0.015\linewidth}
\rotatebox{90}{\scriptsize{(\textbf{ours})}}
\end{minipage}
\begin{minipage}{0.46\linewidth}
\includegraphics[width=\columnwidth]{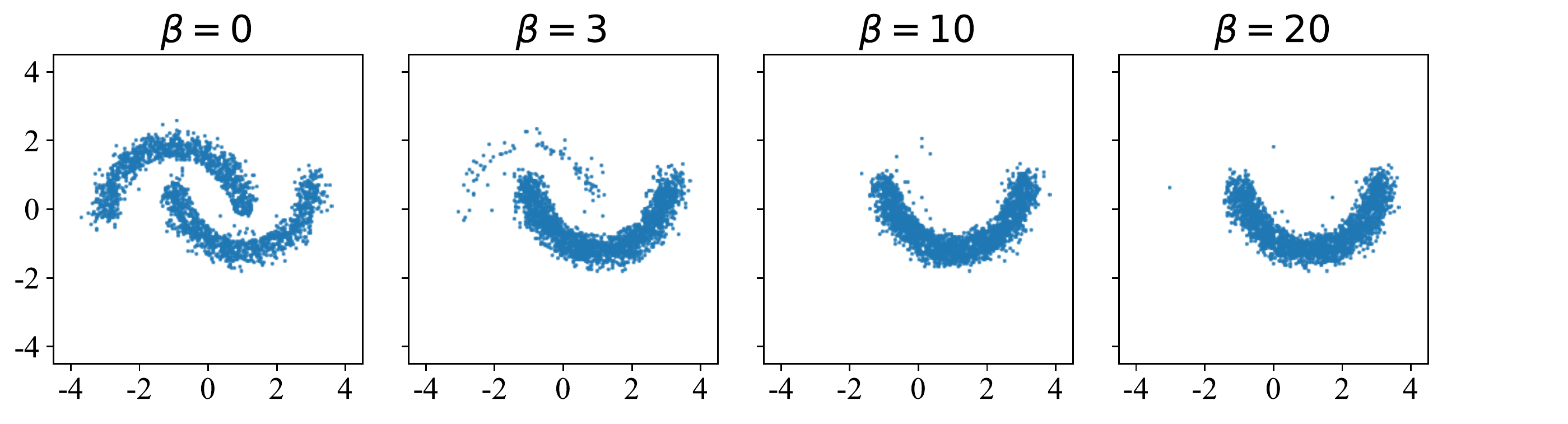}
\end{minipage}
\begin{minipage}{0.46\linewidth}
\includegraphics[width=\columnwidth]{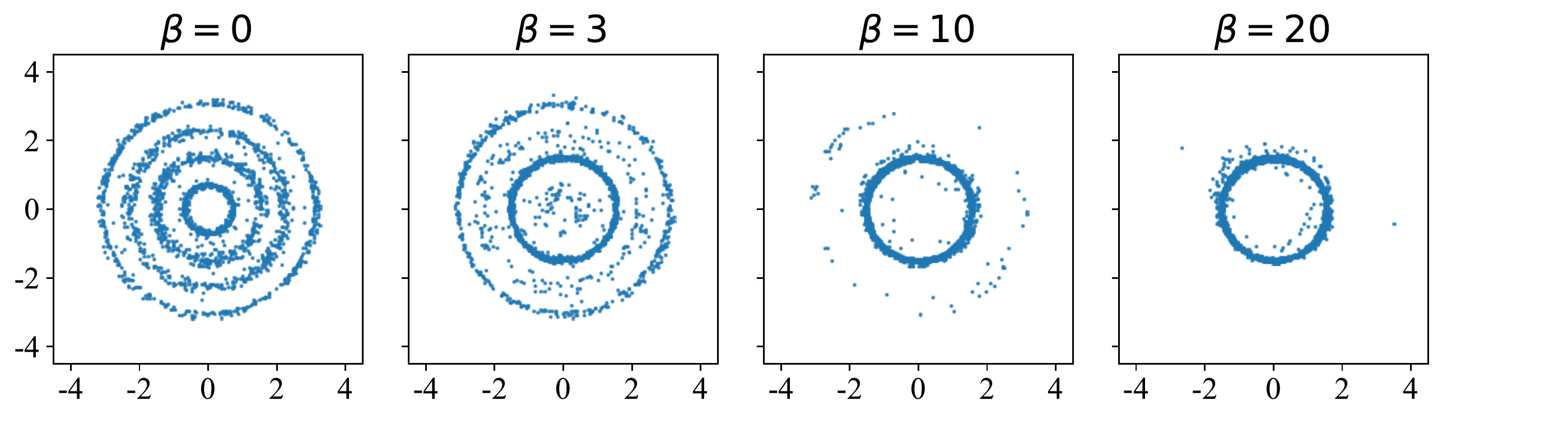}
\end{minipage}\\
\begin{minipage}{0.015\linewidth}
\rotatebox{90}{\scriptsize{CEP}}
\end{minipage}
\begin{minipage}{0.015\linewidth}
\rotatebox{90}{\scriptsize{(\textbf{ours})}}
\end{minipage}
\begin{minipage}{0.46\linewidth}
\includegraphics[width=\columnwidth]{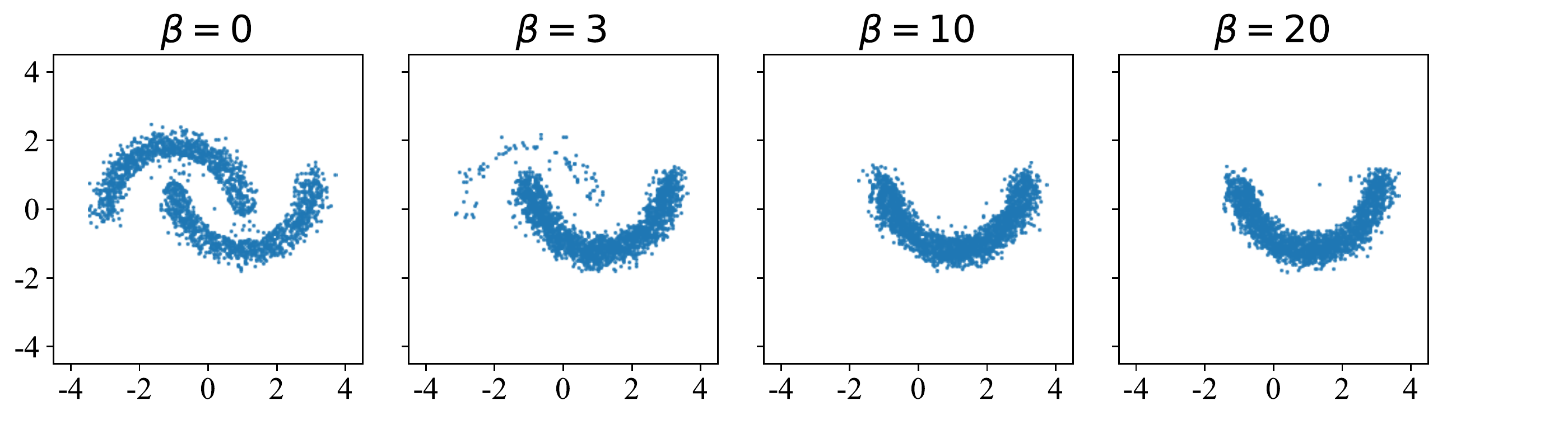}
\end{minipage}
\begin{minipage}{0.46\linewidth}
\includegraphics[width=\columnwidth]{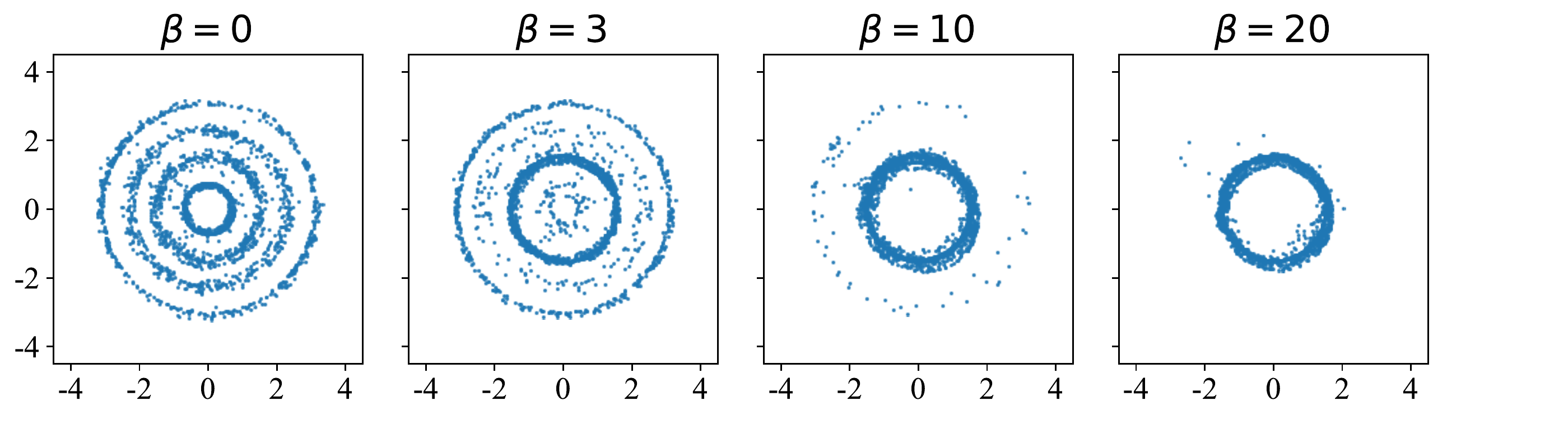}
\end{minipage}\\
\caption{
Scatter plots of different energy guidance methods in 2-D experiments. E-MSE is another method we propose as a variant of MSE guidance (\cref{sec:emse}).
}
\vspace{-8mm}
\label{fig:toy_appendix}
\end{figure}

\newpage
\begin{figure}[!hbt]
\centering
\begin{minipage}{0.02\linewidth}
\rotatebox{90}{\scriptsize{8gaussians}}
\end{minipage}
\begin{minipage}{0.97\linewidth}
\includegraphics[width=\columnwidth]{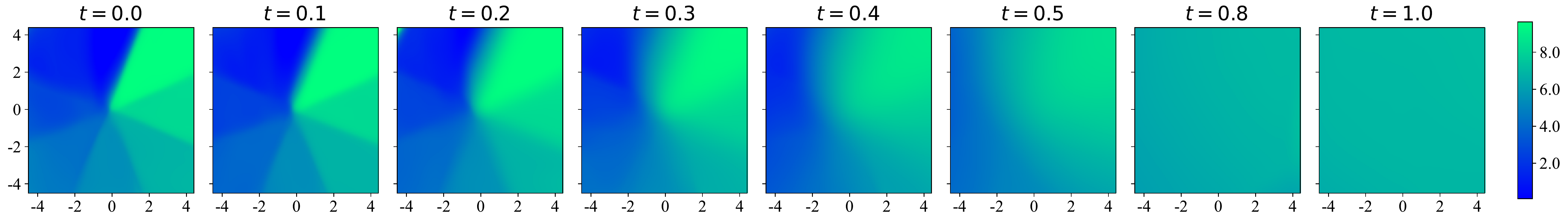}
\end{minipage}\\
\begin{minipage}{0.02\linewidth}
\rotatebox{90}{\scriptsize{swissroll}}
\end{minipage}
\begin{minipage}{0.97\linewidth}
\includegraphics[width=\columnwidth]{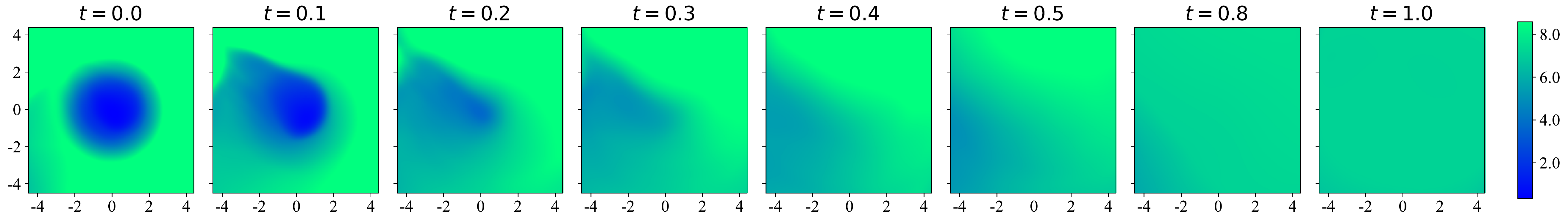}
\end{minipage}\\
\begin{minipage}{0.02\linewidth}
\rotatebox{90}{\scriptsize{2spirals}}
\end{minipage}
\begin{minipage}{0.97\linewidth}
\includegraphics[width=\columnwidth]{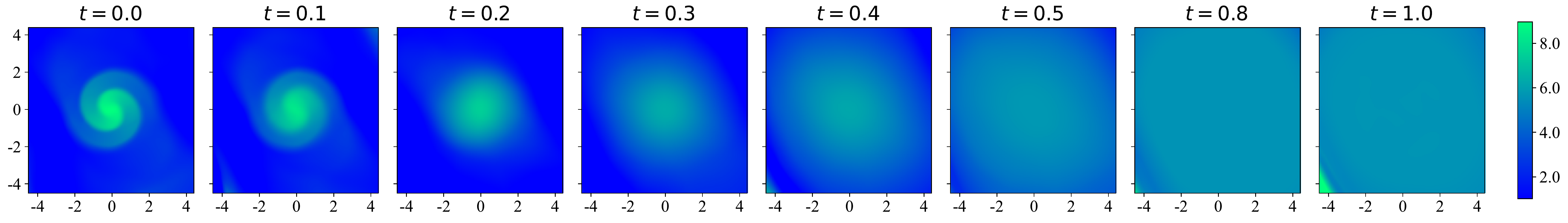}
\end{minipage}\\
\begin{minipage}{0.02\linewidth}
\rotatebox{90}{\scriptsize{moons}}
\end{minipage}
\begin{minipage}{0.97\linewidth}
\includegraphics[width=\columnwidth]{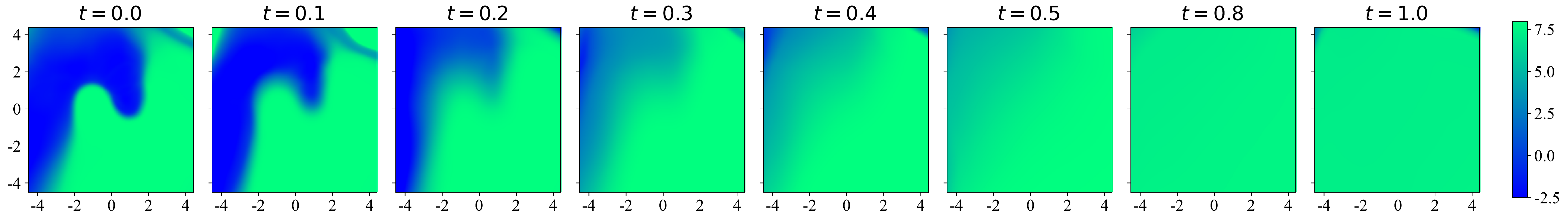}
\end{minipage}\\
\begin{minipage}{0.02\linewidth}
\rotatebox{90}{\scriptsize{rings}}
\end{minipage}
\begin{minipage}{0.97\linewidth}
\includegraphics[width=\columnwidth]{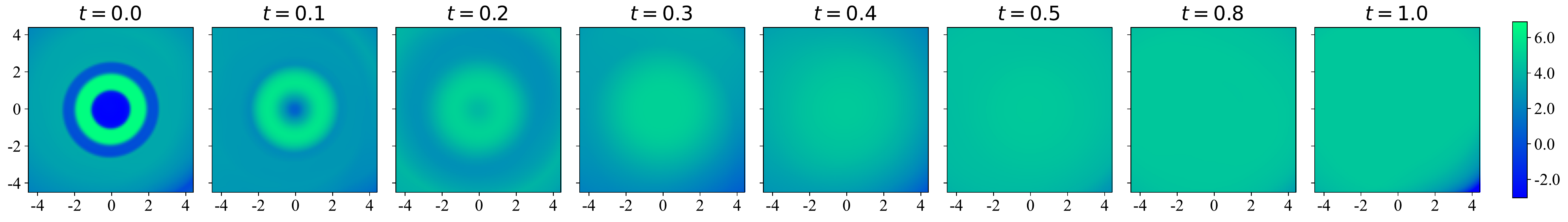}
\end{minipage}\\
\caption{Visualization of the contrastively learned intermediate energy model when $\beta=10$.}
\vspace{-8mm}
\label{fig:toy_appendix_qt}
\end{figure}

\newpage

\newpage
\section{Training Curves for Offline Reinforcement Learning}
\label{sec:training_curves}
\begin{figure}[!h]
\vspace{-2mm}
\centering
\includegraphics[width = .33\linewidth]{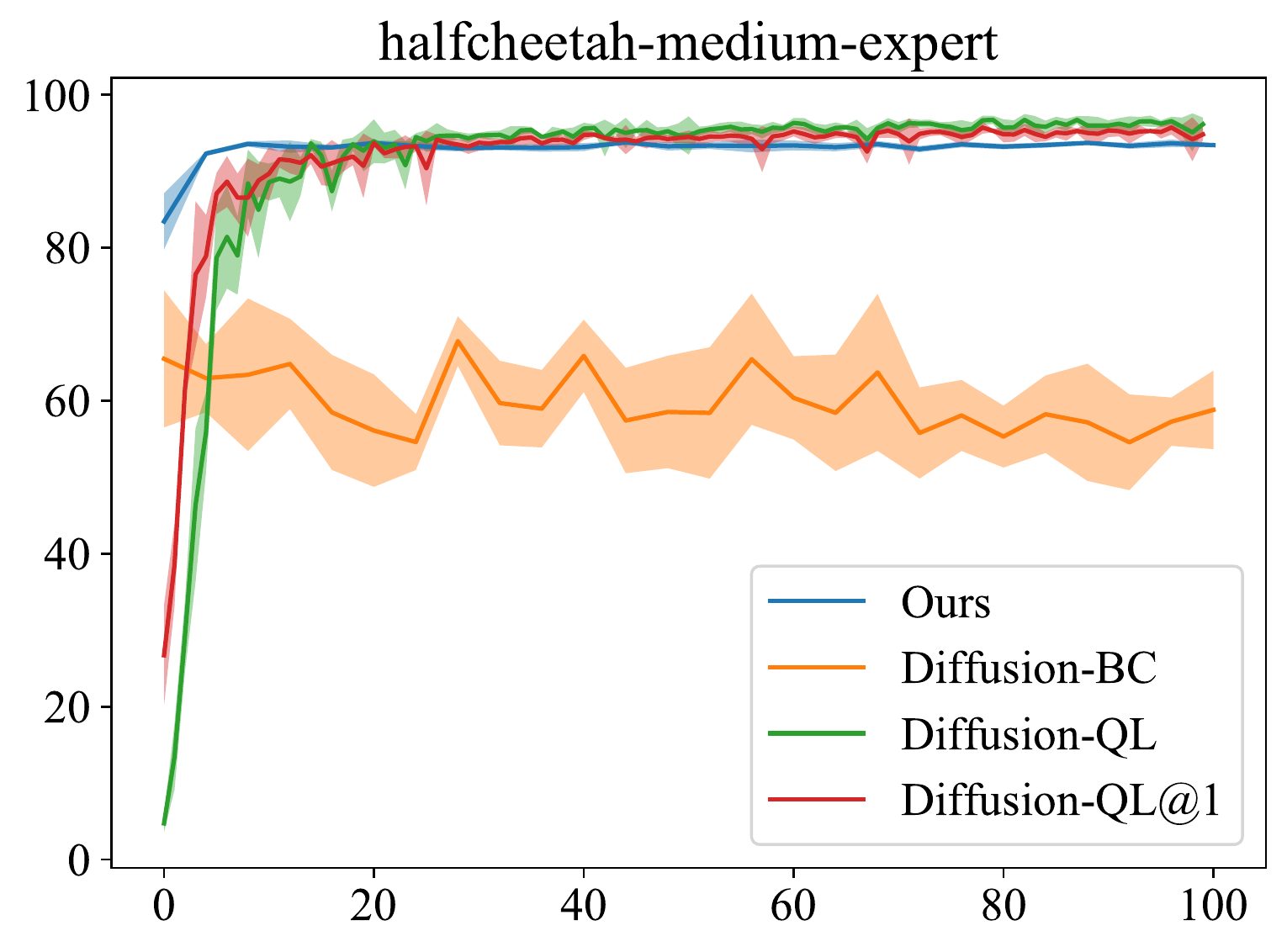}
\includegraphics[width = .33\linewidth]{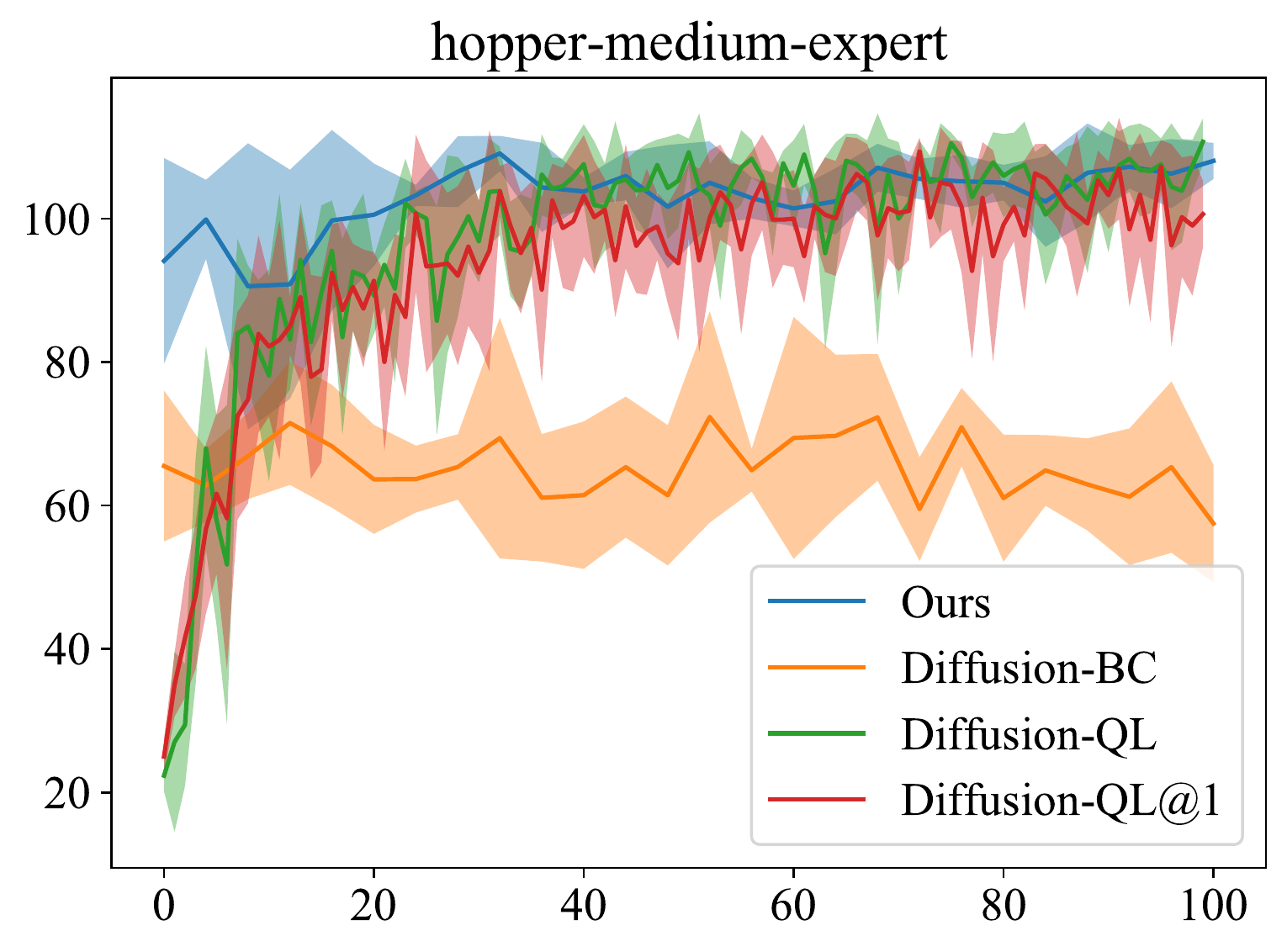}
\includegraphics[width = .33\linewidth]{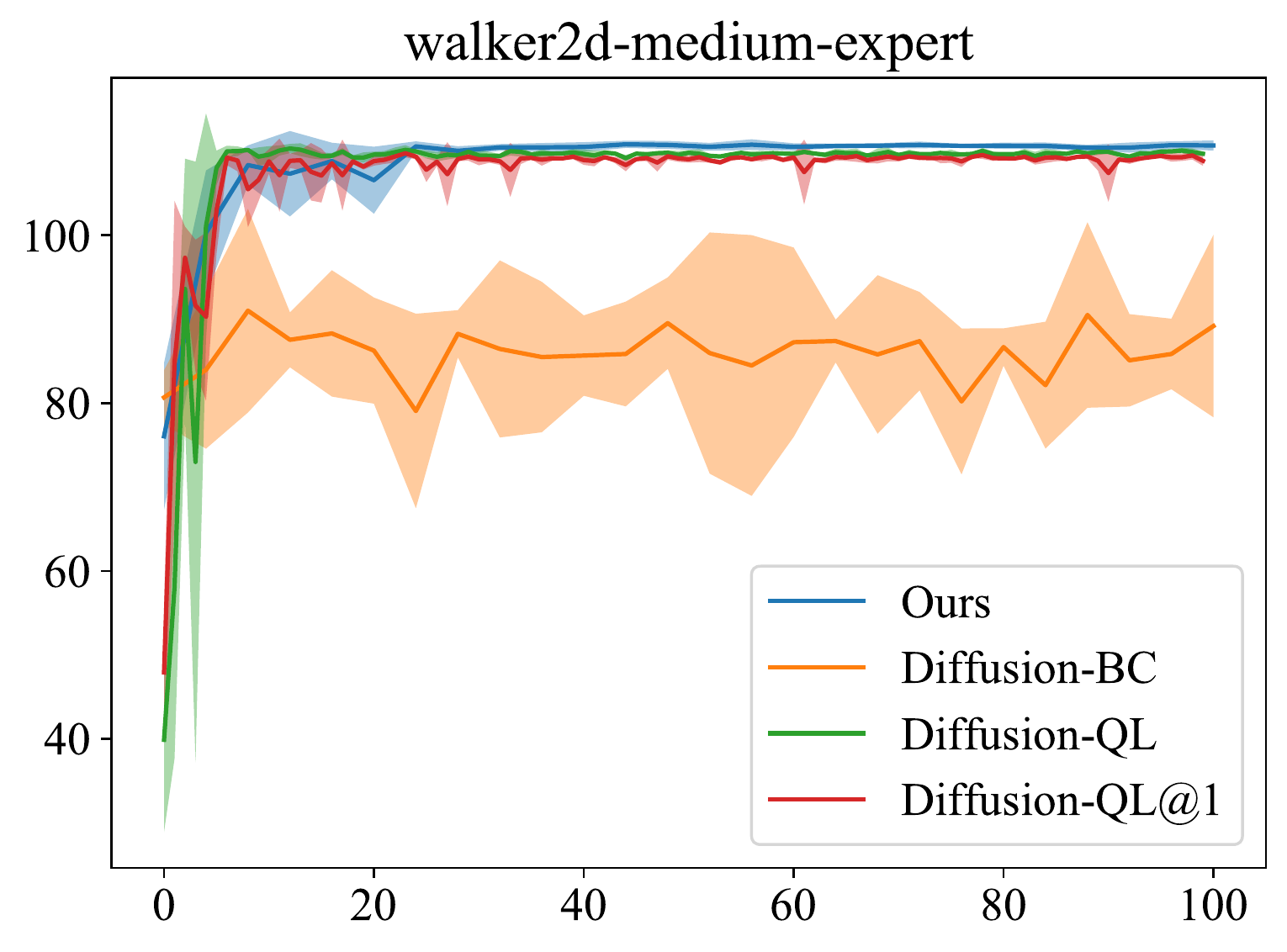}\\
\includegraphics[width = .33\linewidth]{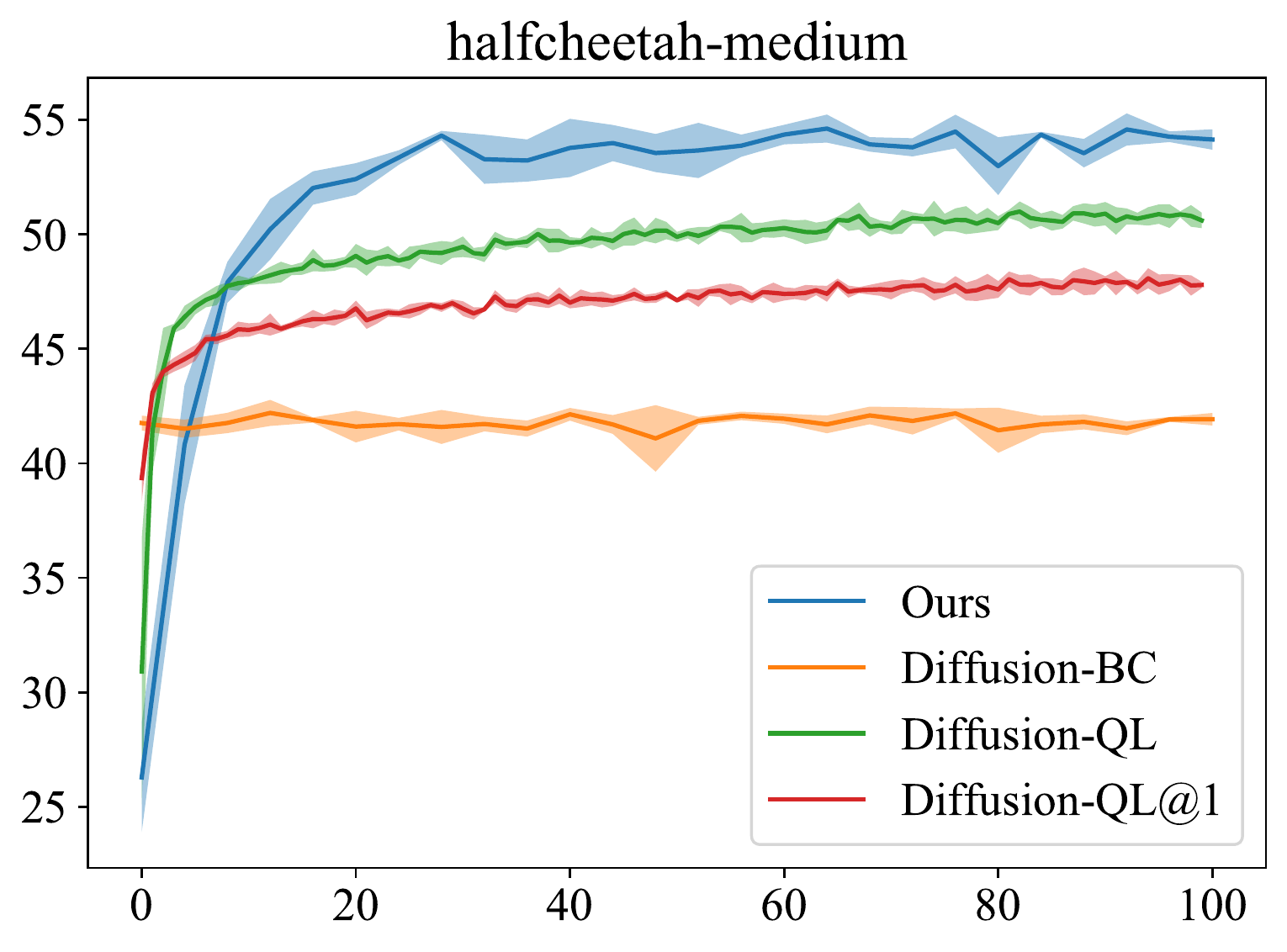}
\includegraphics[width = .33\linewidth]{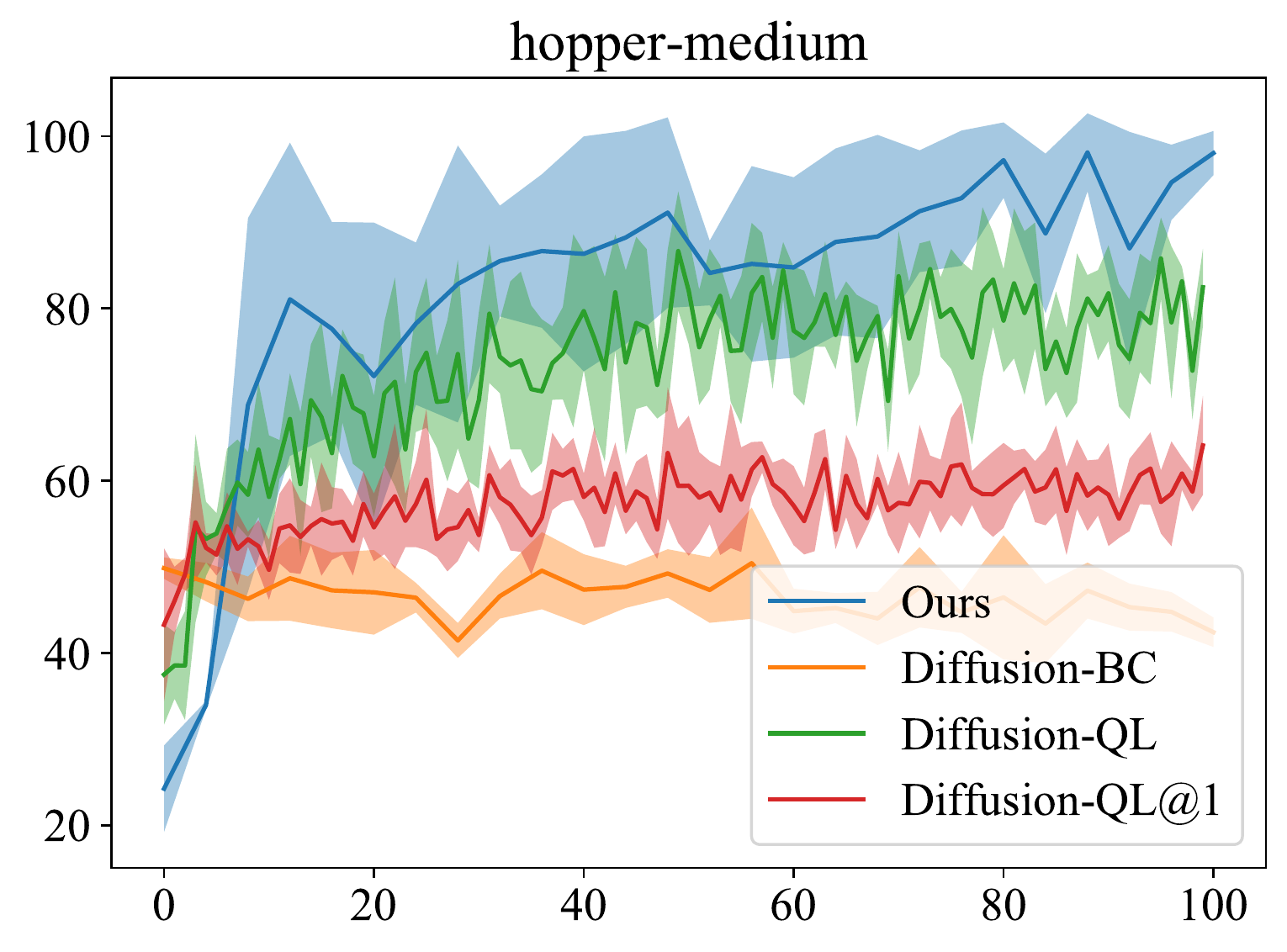}
\includegraphics[width = .33\linewidth]{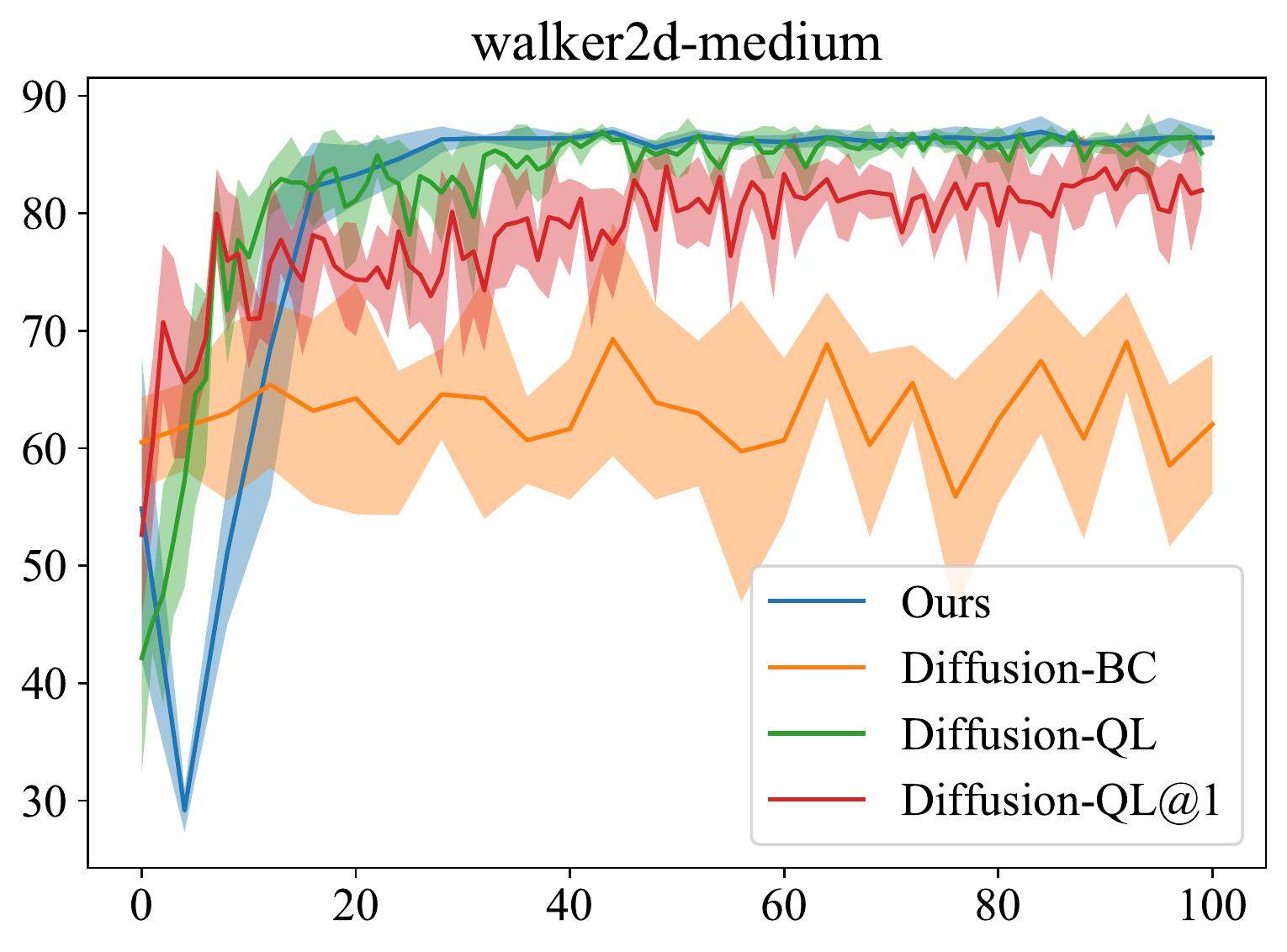}\\
\includegraphics[width = .33\linewidth]{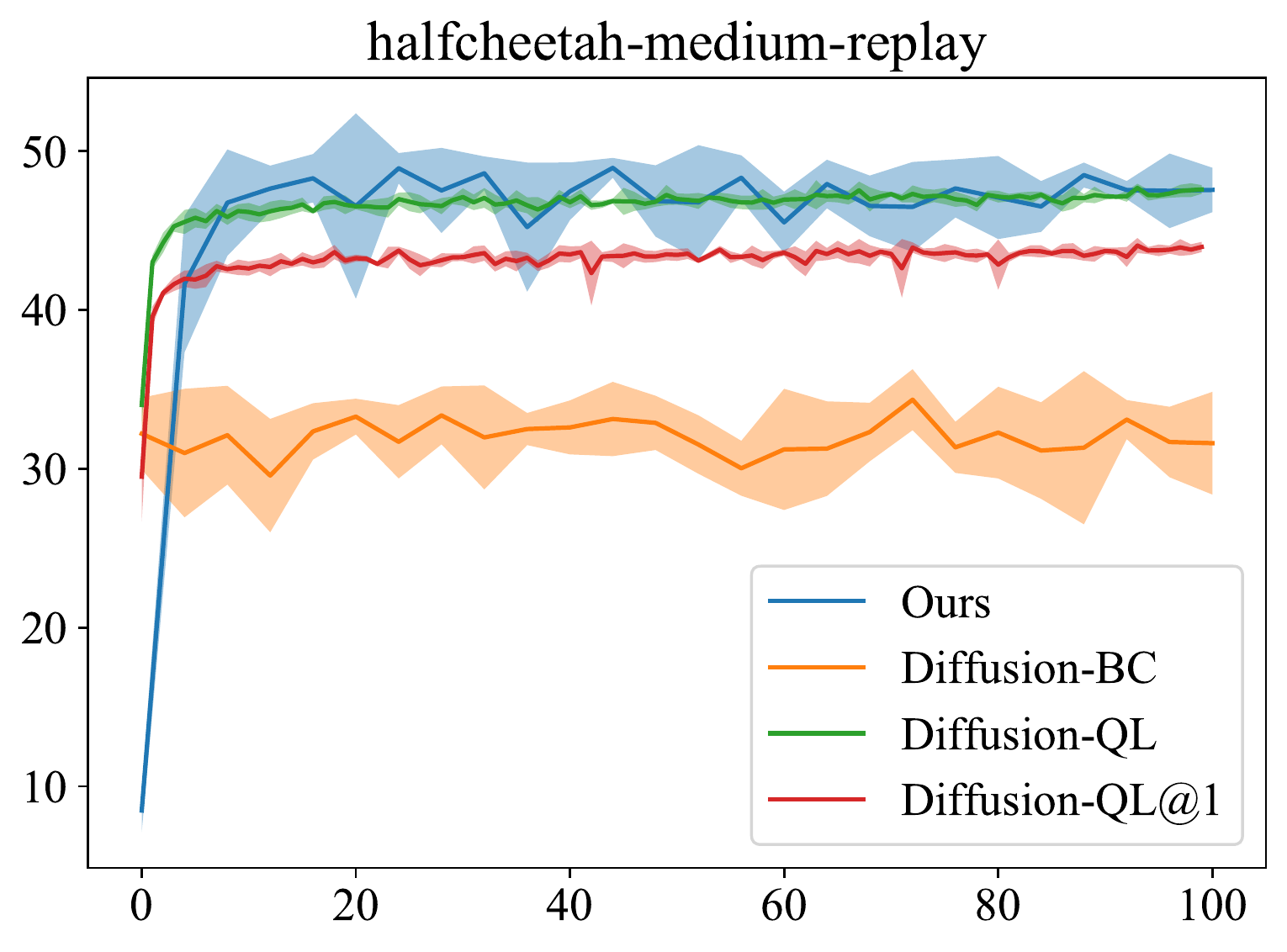}
\includegraphics[width = .33\linewidth]{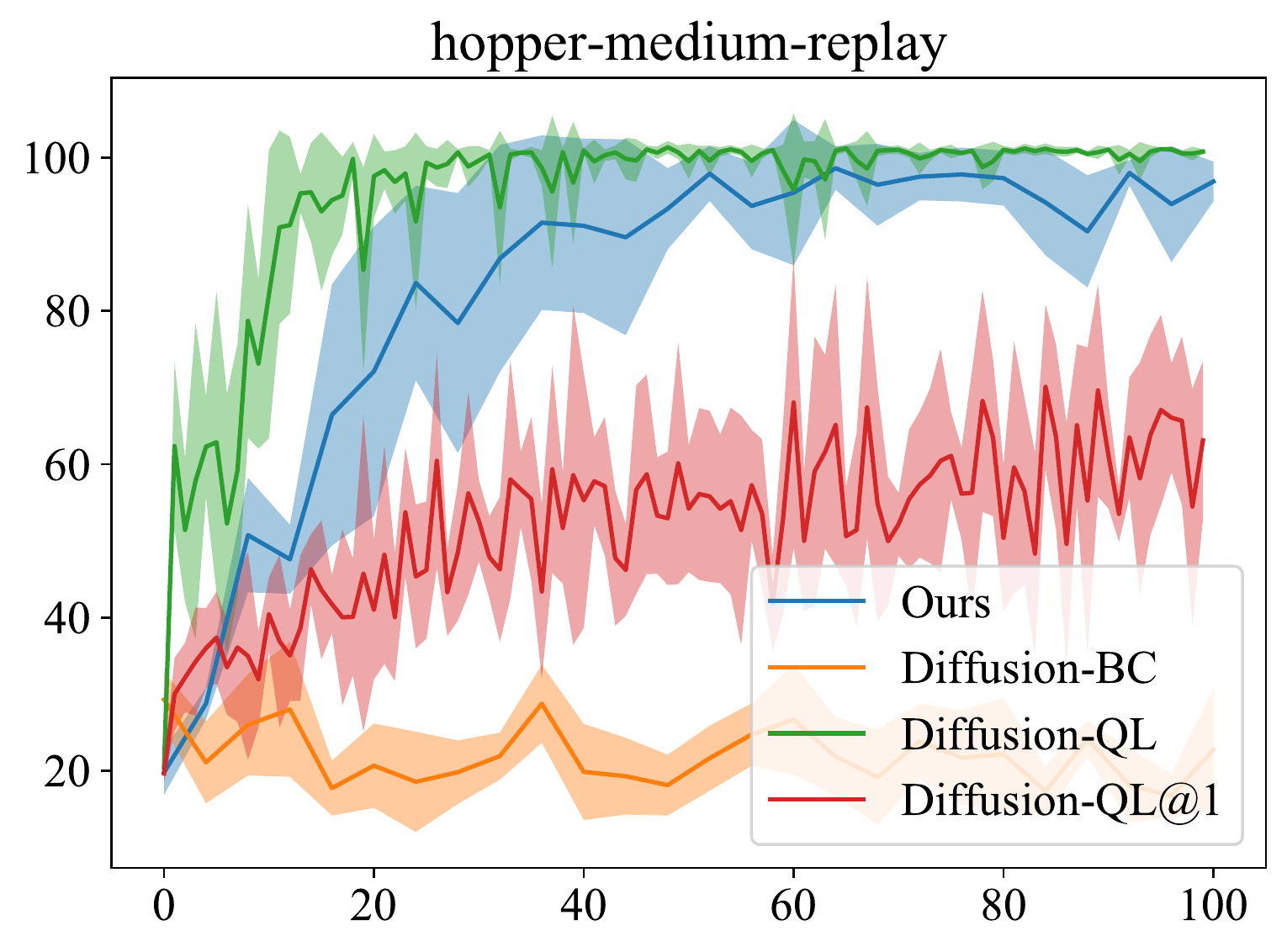}
\includegraphics[width = .33\linewidth]{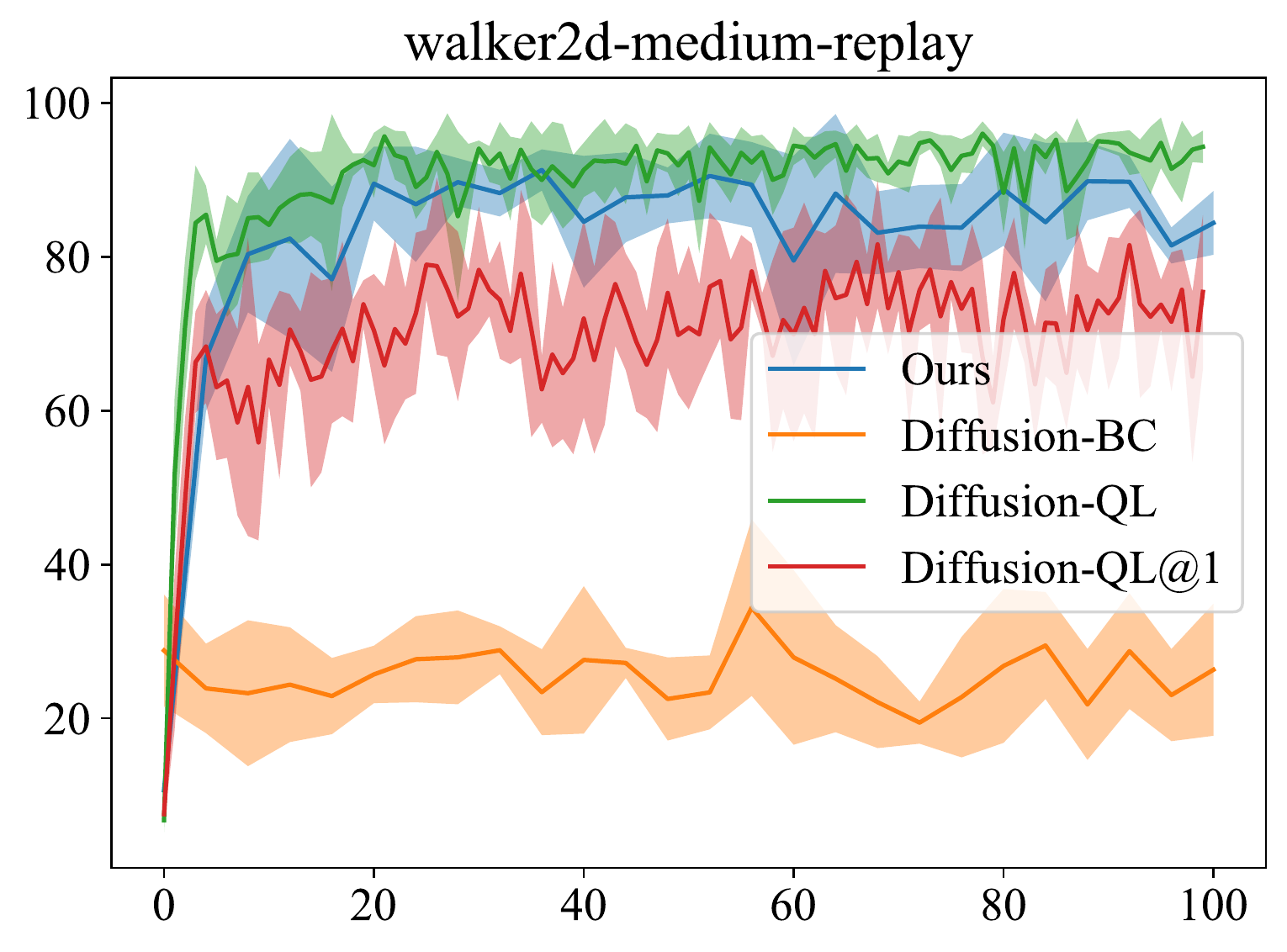}\\
\includegraphics[width = .33\linewidth]{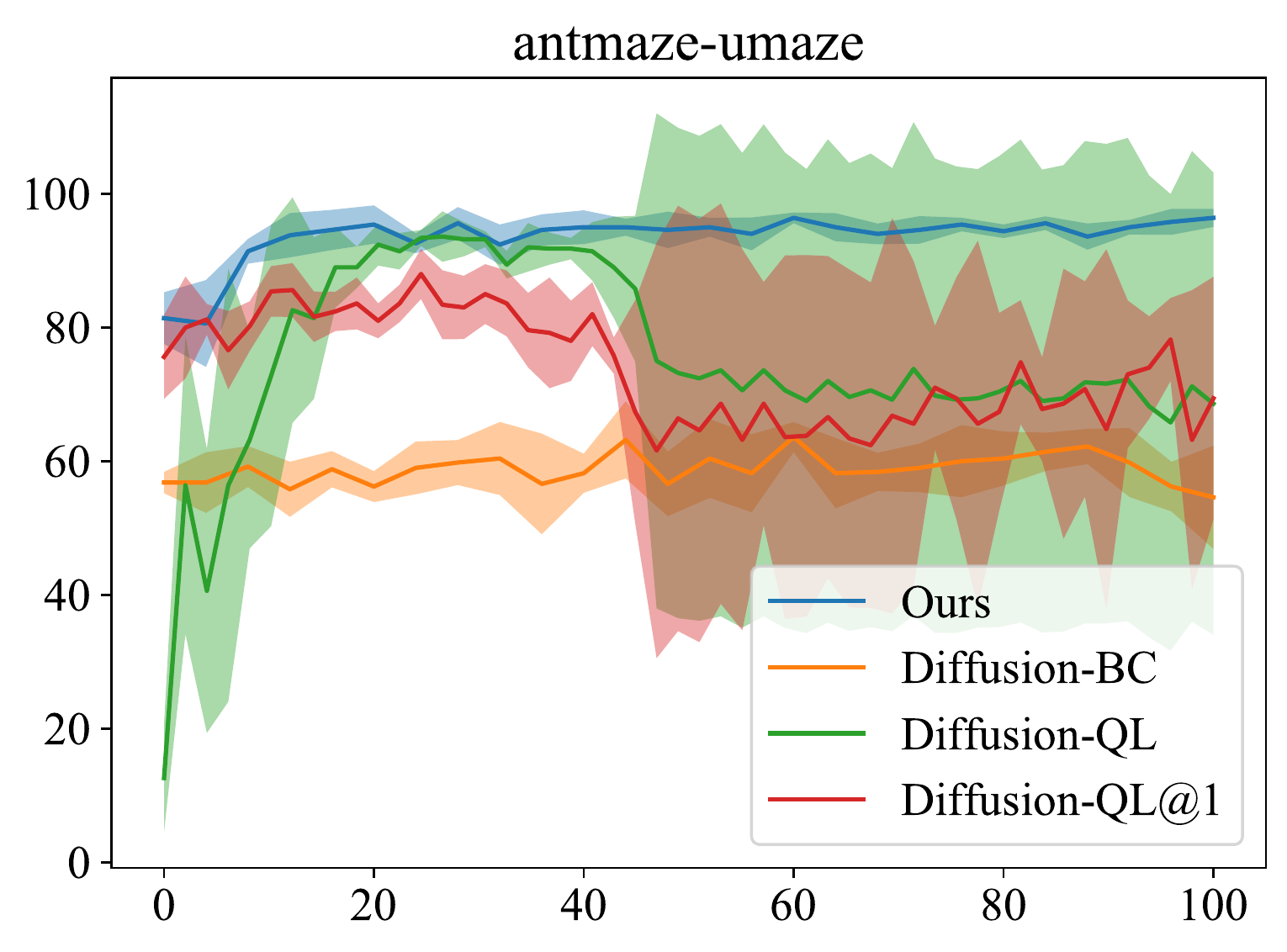}
\includegraphics[width = .33\linewidth]{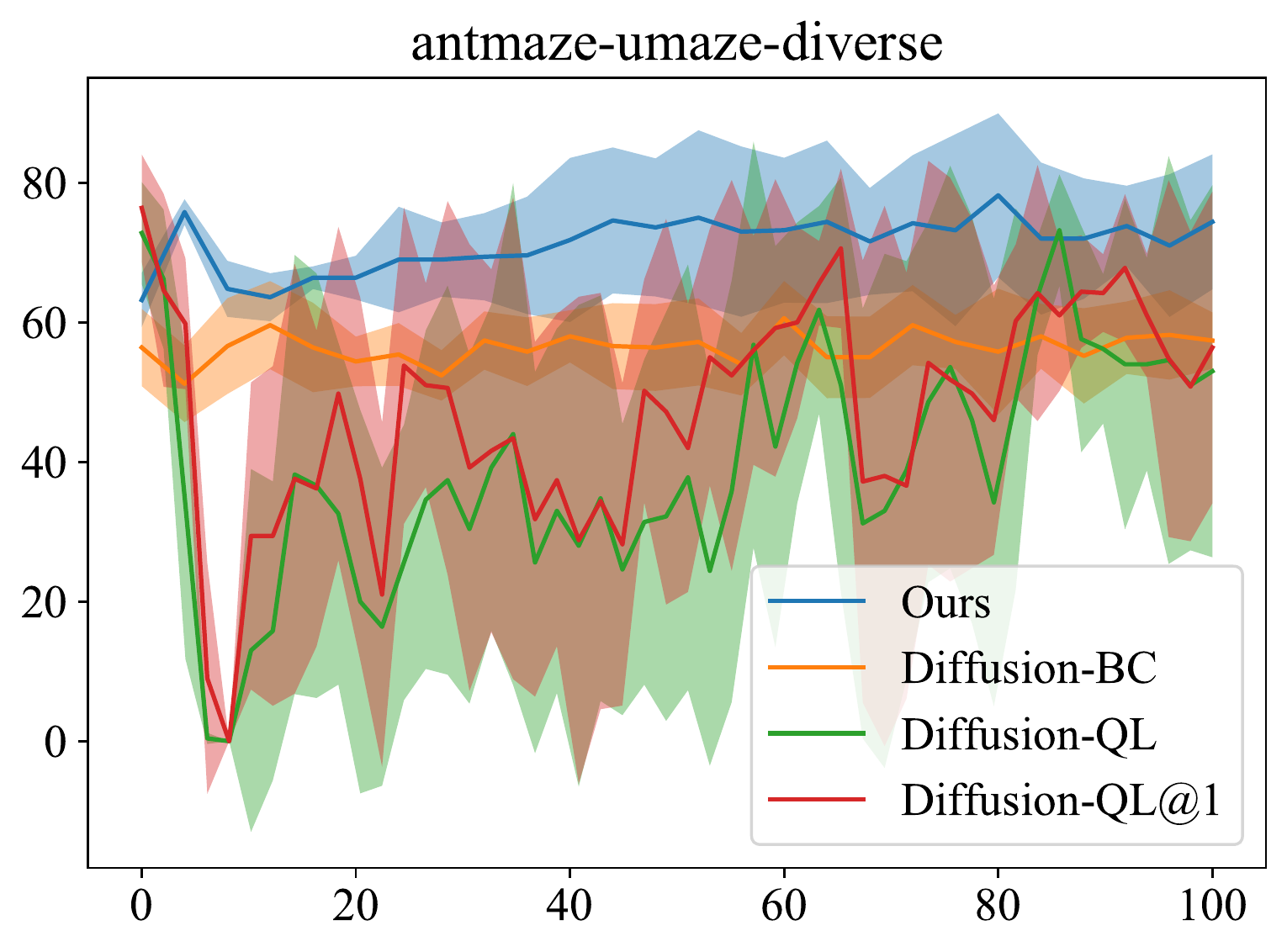}
\includegraphics[width = .33\linewidth]{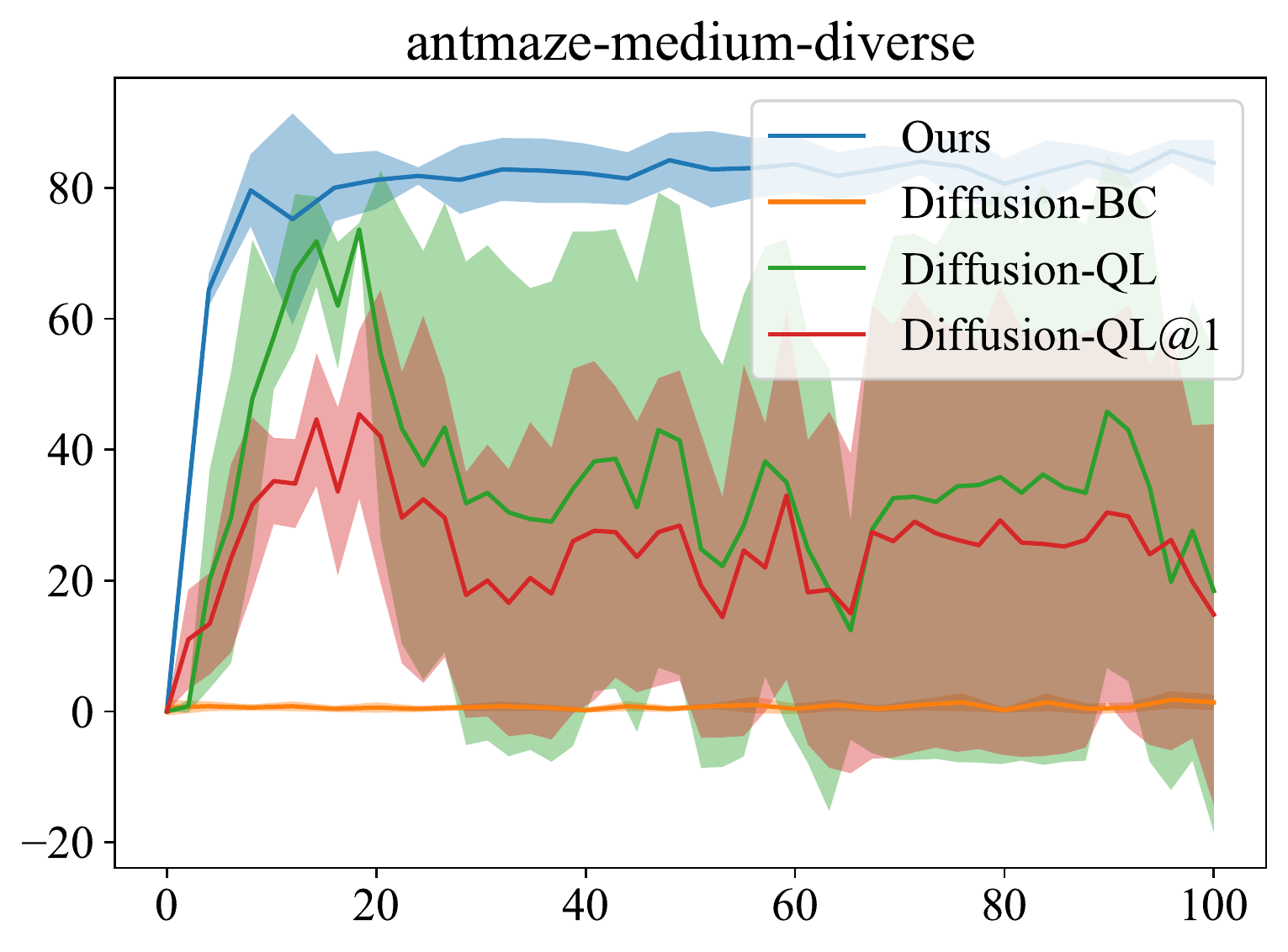}\\
\includegraphics[width = .33\linewidth]{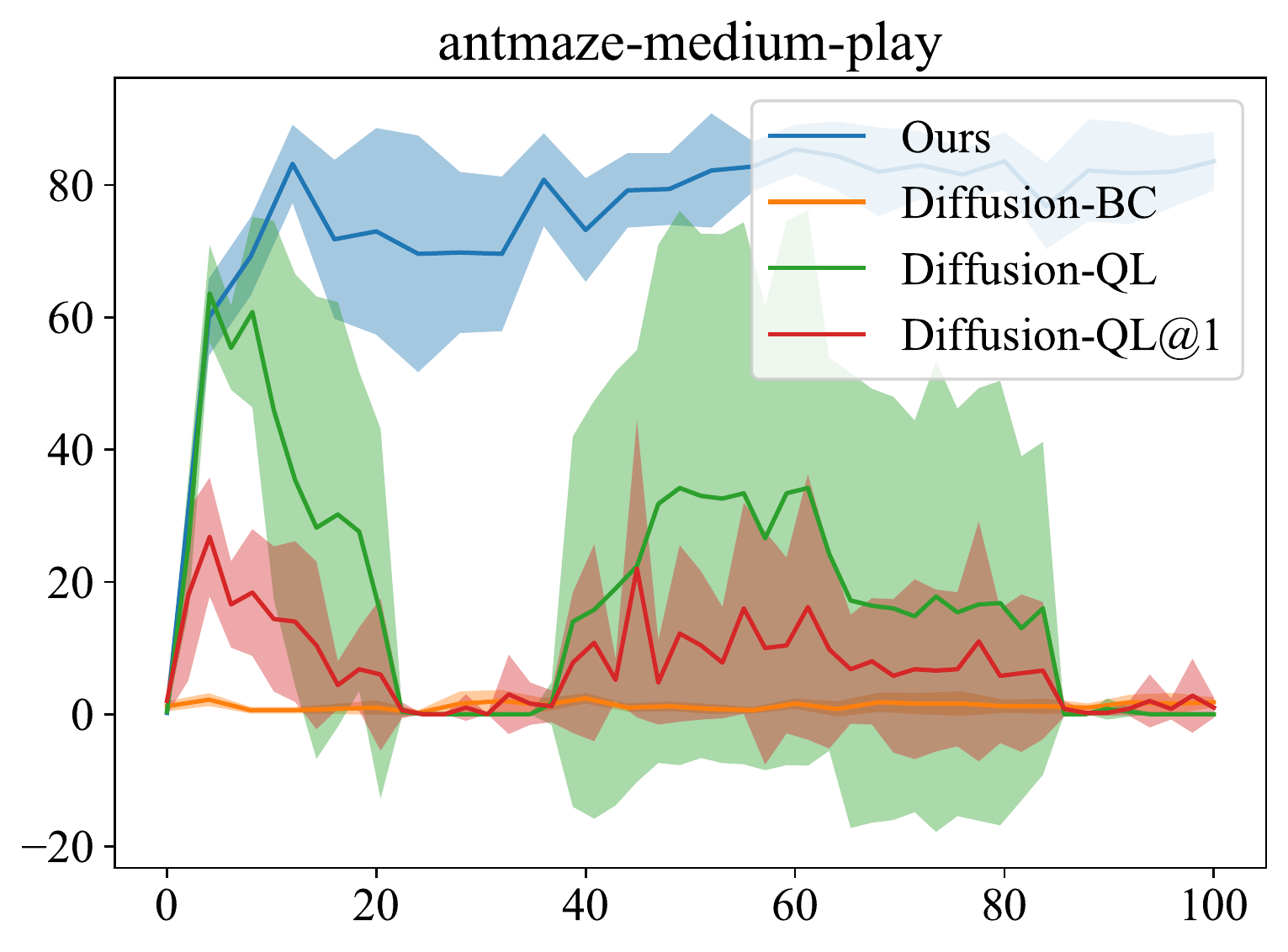}
\includegraphics[width = .33\linewidth]{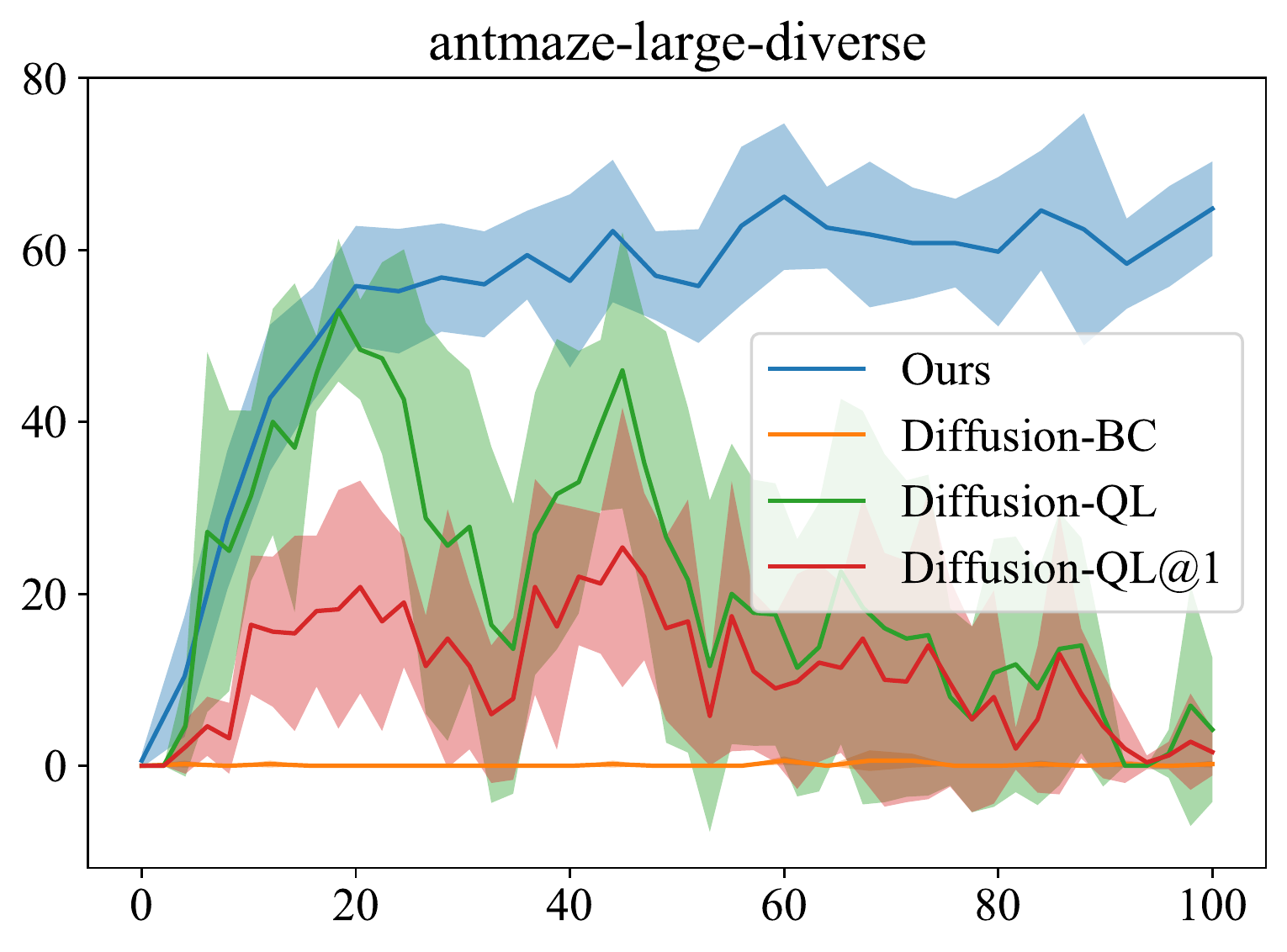}
\includegraphics[width = .33\linewidth]{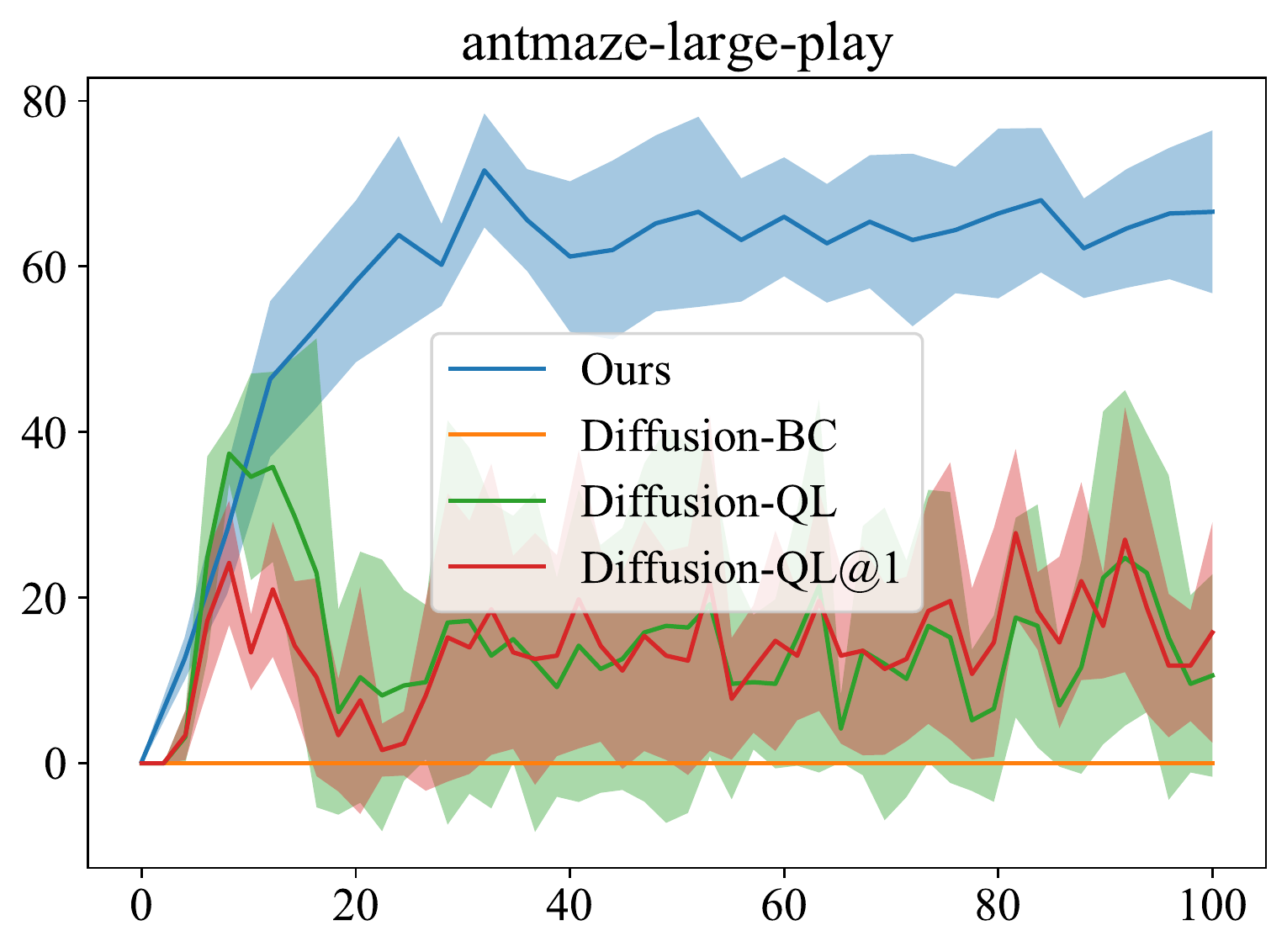}\\
\vspace{-2mm}
\caption{Training curves of QGPO (ours) and several baselines. We plot mean and standard deviation of results across five random seeds. Scores are normalized according to \cite{d4rl}. Diffusion BC indicates evaluation scores of the learned behavior policy without any guidance ($s=0$).}
\vspace{-4mm}
\label{fig:plots}
\end{figure}

\newpage
\section{More Results for Energy-Guided Image Synthesis}
\label{Sec:Scale_samples}
\begin{figure}[hbt!]
\centering
\includegraphics[width = 0.94\linewidth]{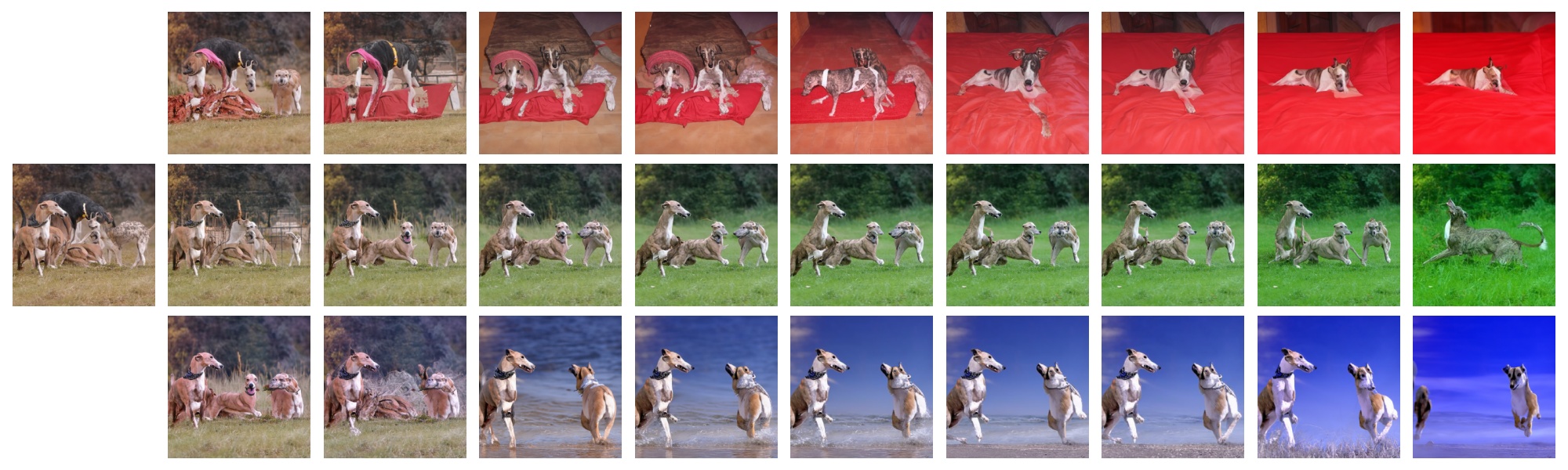}
\includegraphics[width = 0.94\linewidth]{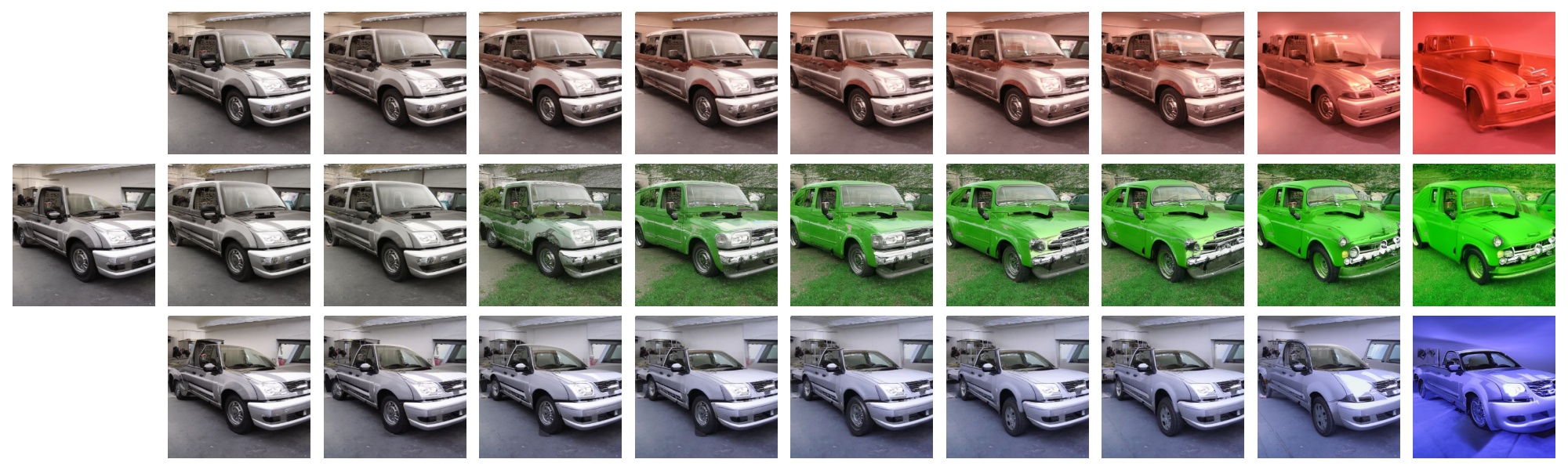}
\vspace{-3mm}
\caption{Ablation of color guidance with a conditional diffusion prior. From left to right are samples under an increasing guidance scale in [0.0, 0.25, 0.5, 1.0, 1.5, 2.0, 2.5, 3.0, 5.0, 10.0].}
\label{fig:colors_conditional}
\end{figure}
\begin{figure}[hbt!]
\centering
\includegraphics[width = 0.94\linewidth]{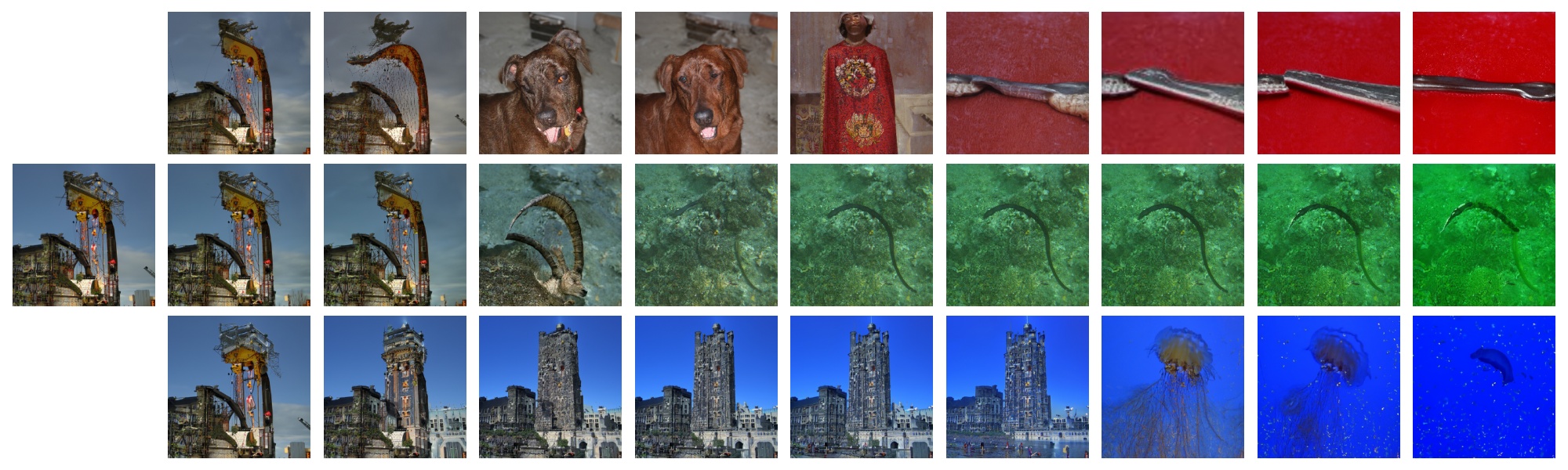}
\includegraphics[width = 0.94\linewidth]{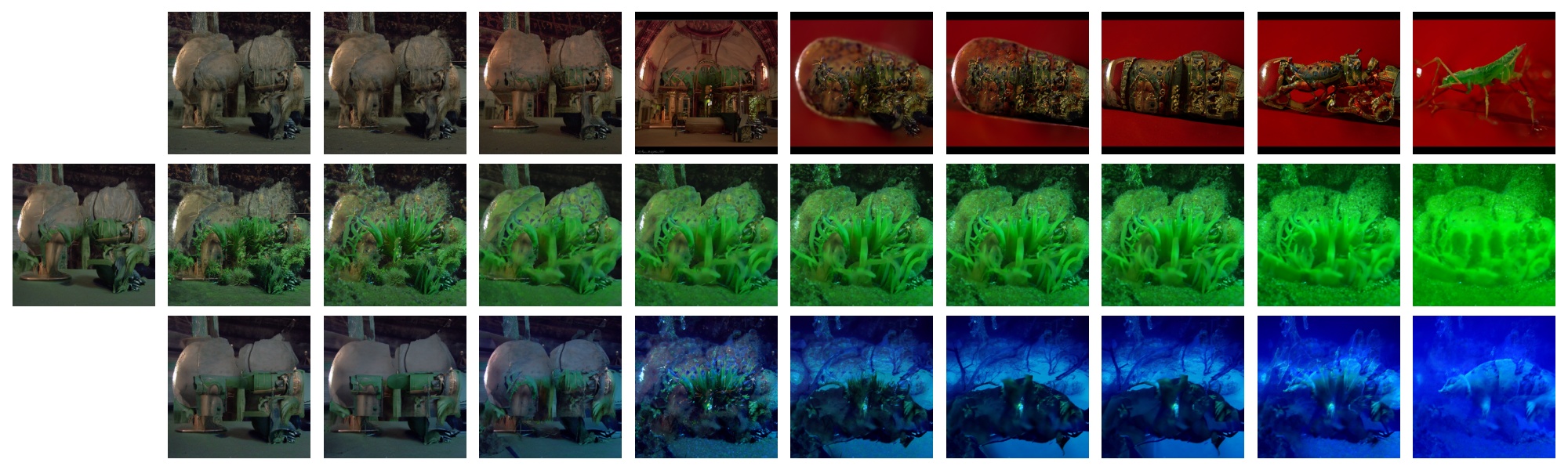}
\vspace{-3mm}
\caption{Ablation of color guidance with an unconditional diffusion prior. From left to right are samples under an increasing guidance scale in [0.0, 0.25, 0.5, 1.0, 1.5, 2.0, 2.5, 3.0, 5.0, 10.0].}
\label{fig:colors}
\end{figure}

\newpage
\section{CEP Guidance vs. Classifier Guidance}
\label{fig:image_guidance_qualitative}
\begin{figure}[!hbt]
\centering
\vspace{-2mm}
\begin{minipage}{0.48\linewidth}
\centering
\includegraphics[width =\linewidth]{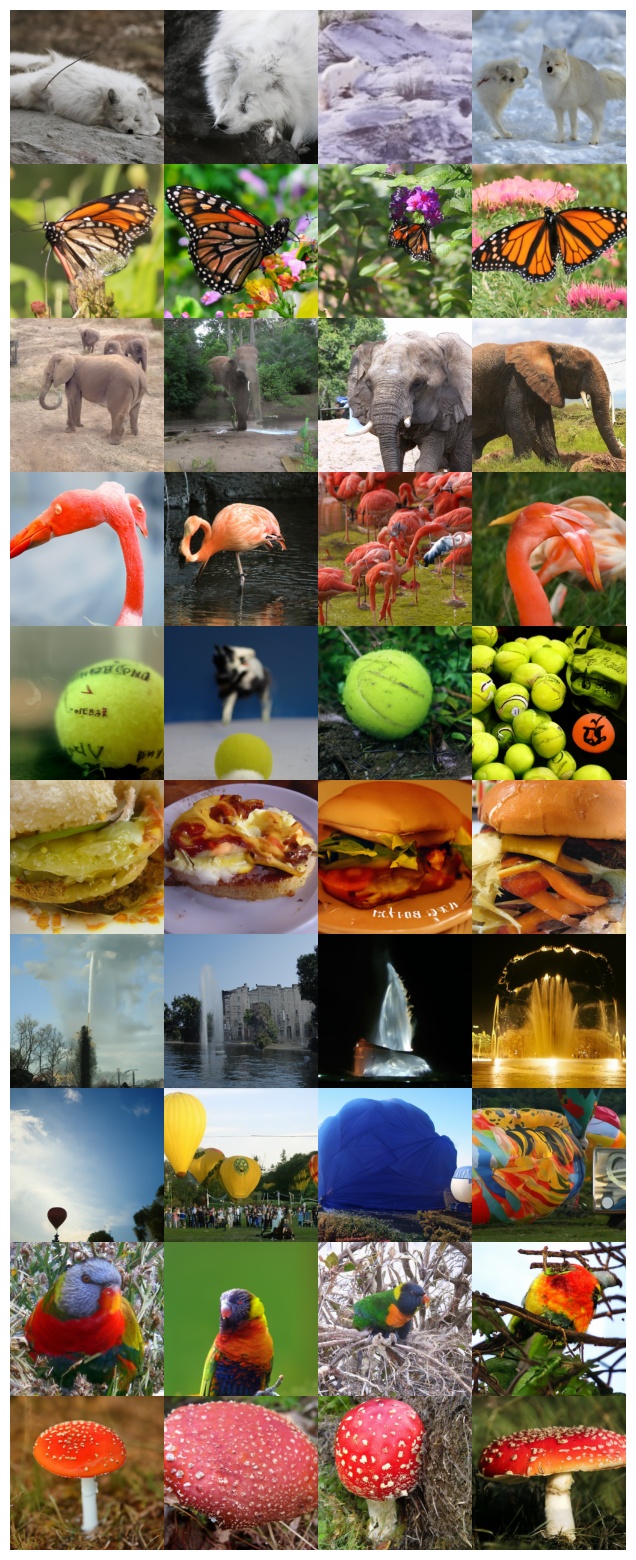}
Classifier guidance \citep{diffusion_beat_gan} \\
\end{minipage}
\begin{minipage}{0.48\linewidth}
\centering
\includegraphics[width =\linewidth]{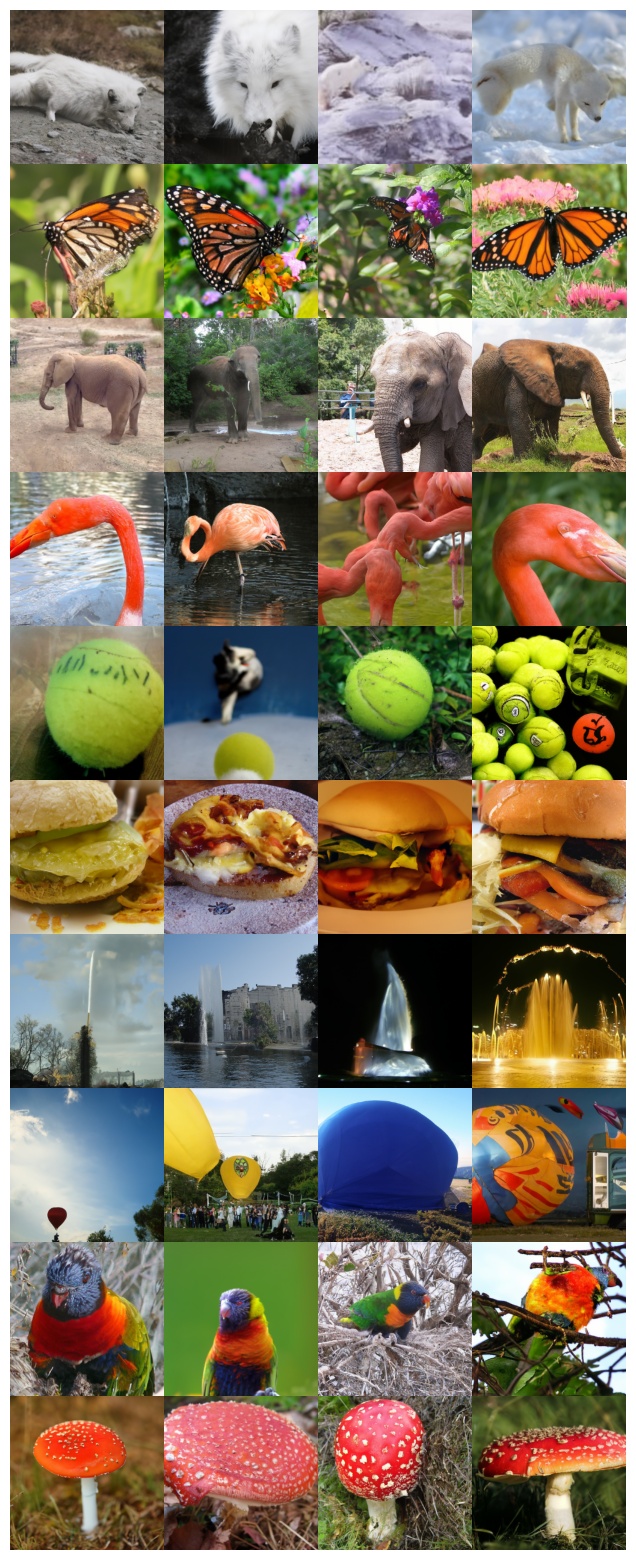}
Energy guidance (\textbf{ours}) \\
\end{minipage}
\caption{Samples from the conditional 256×256 ImageNet model with different guidance methods. Random seeds are fixed across experiments. Classes are 279: arctic fox, 323: monarch butterfly, 386: african elephant, 130: flamingo, 852: tennis ball, 933: cheeseburger, 562: fountain, 417: balloon, 90: lorikeet, 992: agaric.}
\vspace{-8mm}
\end{figure}

\end{document}